\documentclass{article}

\usepackage[preprint]{neurips_2026}

\usepackage[utf8]{inputenc}
\DeclareUnicodeCharacter{266B}{}
\usepackage[T1]{fontenc}      
\usepackage{amsmath,amssymb,amsthm,mathtools}
\usepackage{amsfonts}         
\usepackage{bm}
\usepackage{enumitem}
\usepackage{booktabs}
\usepackage{xcolor}
\usepackage{url}              
\usepackage{microtype}        
\usepackage{nicefrac}         
\usepackage{float}
\usepackage{adjustbox}
\usepackage{multirow}
\usepackage{wrapfig}
\usepackage{placeins}
\usepackage{graphicx}
\usepackage[font=small,skip=2pt]{subcaption}
\usepackage{hyperref}         
\usepackage{cleveref}         
\usepackage[most]{tcolorbox}  

\newtcolorbox{skillbox}[1]{
  enhanced, breakable,
  colback=gray!4, colframe=black!70,
  colbacktitle=black!75, fonttitle=\bfseries\color{white}, coltitle=white,
  title={#1},
  boxrule=0.4pt, arc=1pt,
  left=8pt, right=8pt, top=5pt, bottom=5pt,
  toptitle=3pt, bottomtitle=3pt, lefttitle=8pt,
  before skip=8pt, after skip=8pt,
  toprule at break=0.4pt,
  bottomrule at break=0.4pt,
  pad at break*=2mm,
}

\definecolor{darkgreen}{rgb}{0,0.5,0}

\newcommand{\std}[1]{{\footnotesize$_{\pm#1}$}}

\newtheorem{theorem}{Theorem}
\newtheorem{lemma}[theorem]{Lemma}
\newtheorem{proposition}[theorem]{Proposition}
\newtheorem{corollary}[theorem]{Corollary}
\newtheorem{definition}{Definition}
\newtheorem{remark}{Remark}

\DeclareMathOperator{\Var}{Var}

\newcommand{\E}{\mathbb{E}}
\newcommand{\R}{\mathbb{R}}
\newcommand{\Rtilde}{\tilde{R}}
\newcommand{\Lttb}{\mathcal{L}_{\text{TTB}}}
\newcommand{\Slib}{\mathcal{S}}
\newcommand{\Dset}{\mathcal{D}}
\newcommand{\Batch}{\mathcal{B}}

\title{\textbf{SkillFlow: Flow-Driven Recursive Skill Evolution\\for Agentic Orchestration}}
\author{%
  Mingda Zhang$^{1}$, \;
  Tiesunlong Shen$^{2}$, \;
  Haoran Luo$^{3}$, \;
  Wenjin Liu$^{3}$ \\
  Zikai Xiao$^{4}$, \;
  Erik Cambria$^{3}$, \;
  Xiaoying Tang$^{1}$ \\[6pt]
  $^{1}$The Chinese University of Hong Kong, Shenzhen \quad
  $^{2}$National University of Singapore \\
  $^{3}$Nanyang Technological University \quad
  $^{4}$Zhejiang University
}
\date{}

\newcommand{\afterfloatsep}[1][-1.5]{\vspace*{#1em}}
\newcommand{\beforesubsectionsep}{\vspace*{-0.6em}}
\newcommand{\zerosep}{\vspace*{0em}}
\begin{document}
\raggedbottom
\maketitle


\begin{abstract}
In recent years, a variety of powerful LLM-based agentic systems have been applied to automate complex tasks through task orchestration.
However, existing orchestration methods still face key challenges, including strategy collapse under reward maximization, high gradient variance with opaque credit assignment, and unguided skill evolution whose decisions are typically made by directly prompting an LLM to judge rather than derived from principled training signals.
To address these challenges, we propose SkillFlow, a flow-based framework that takes a trainable Supervisor as the agent and a structured environment with dynamic skill library and frozen executor, automating task orchestration through multi-turn interaction.
SkillFlow employs Tempered Trajectory Balance (TTB), a regression-based flow-matching loss that samples trajectories proportional to reward, preserving diverse orchestration strategies rather than collapsing to a single mode.
The same flow objective yields a jointly learned backward policy that provides transparent per-step credit assignment at zero additional inference cost.
Building on these flow diagnostics, a recursive skill evolution mechanism determines \emph{when} to evolve, \emph{what} skills to create or prune, and \emph{where} decision gaps lie---closing the loop from training signal to autonomous capability growth.
Experimental results on 14 datasets show that SkillFlow significantly outperforms baselines across question answering, mathematical reasoning, code generation, and real-world interactive decision making tasks. Our code is available at \url{https://github.com/beita6969/SkillFlow}.
\end{abstract}


\section{Introduction}\label{sec:intro}

\begin{wrapfigure}[14]{r}{0.5\textwidth}
\vspace{-1.5em}
\centering
\includegraphics[width=7.05cm,height=3.459cm]{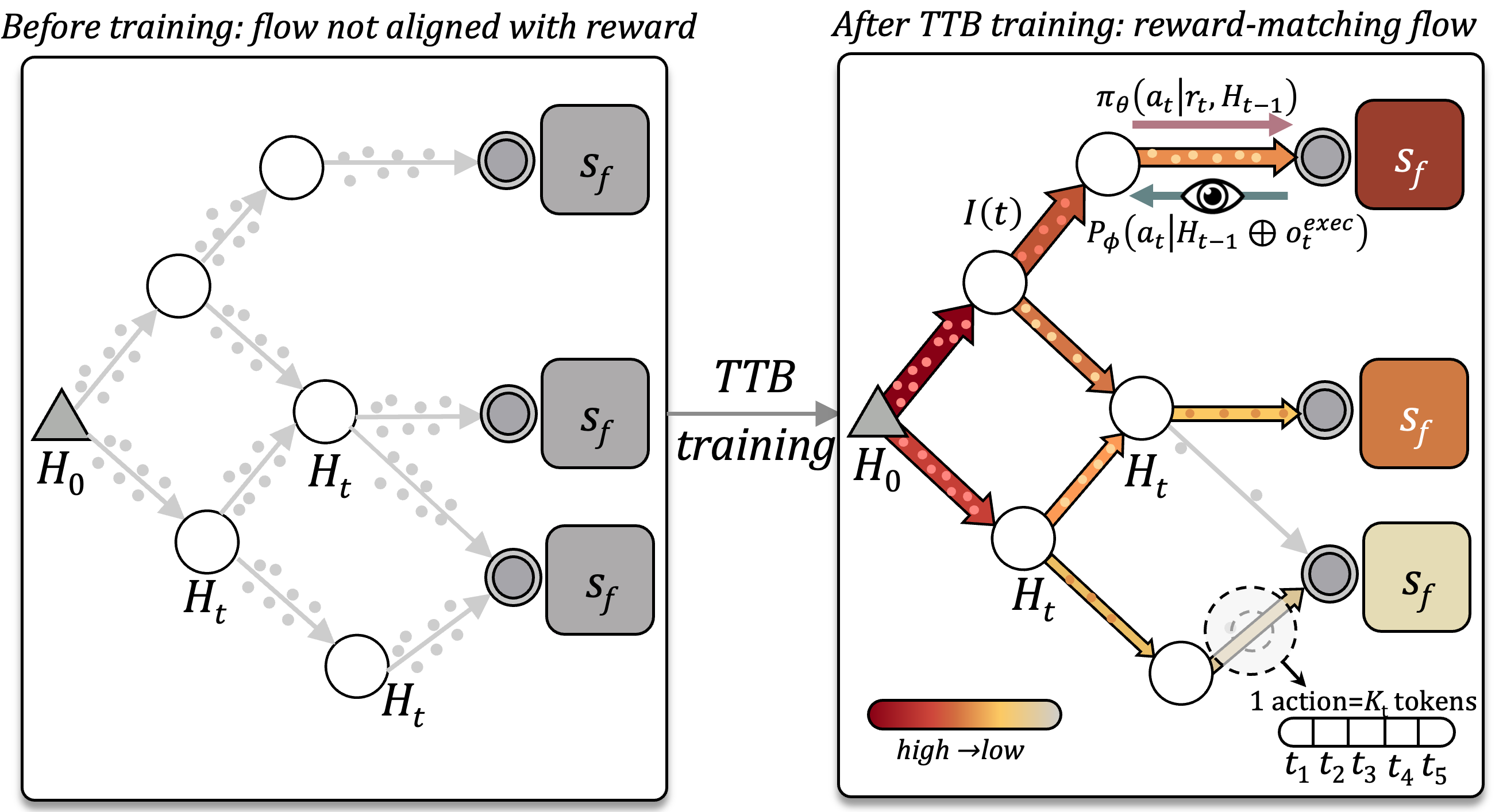}
\caption{SkillFlow at a glance. \textbf{Left}: before training, flow on the orchestration DAG is uniform. \textbf{Right}: after TTB, flow concentrates on reward-rich paths (colormap = flow magnitude). \textbf{Inset}: each edge is one action ($K_t$ tokens).}
\label{fig:teaser}
\end{wrapfigure}

In recent years, a variety of powerful LLM-based agentic systems have been applied to solve a wide range of complex tasks~\citep{yao2023react,hong2024metagpt,openhands2025,dang2025evolving_orchestration}, gradually moving beyond single-turn question answering toward executable end-to-end task completion.

\begin{figure}[!t]
\centering
\includegraphics[width=\textwidth]{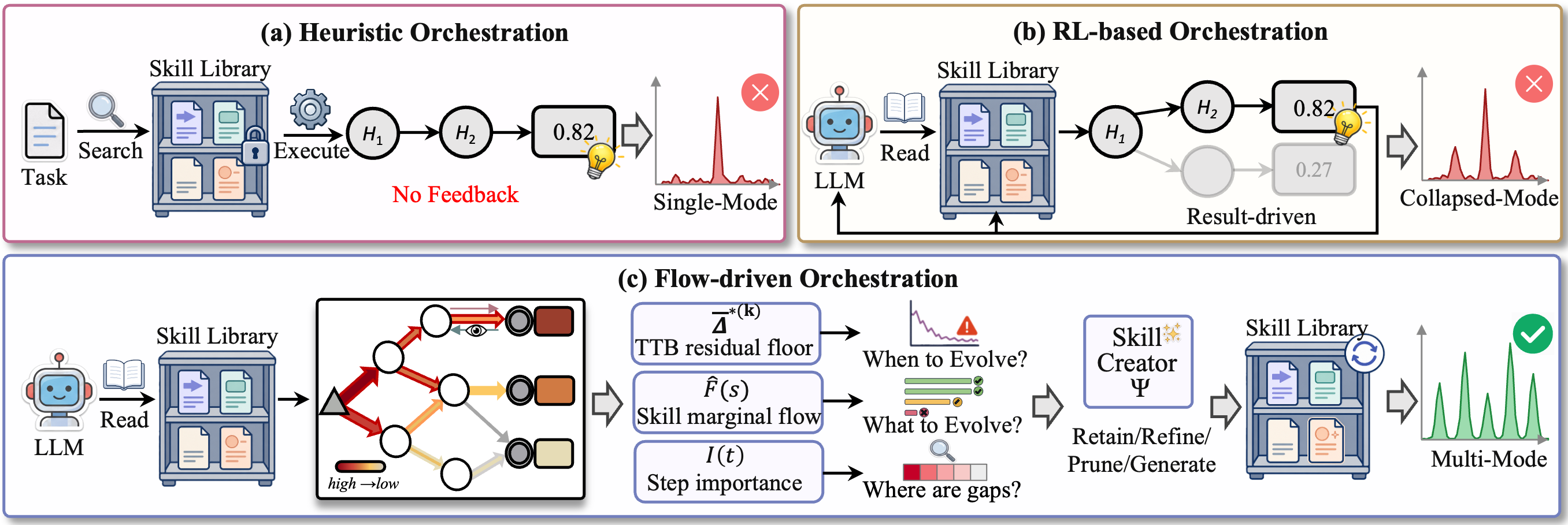}
\caption{Three orchestration paradigms. (a) heuristic dispatch over a locked library; (b) learning-based with terminal reward, prone to trajectory-level mode collapse; (c) SkillFlow co-trains a forward and backward policy under TTB and uses flow diagnostics to drive recursive skill evolution.}
\label{fig:framework}
\afterfloatsep
\end{figure}

In this process, task orchestration has become a key bridge from task goals to reproducible execution: by organizing primitive actions and reusable skills into structured trajectories on an action-level orchestration DAG (Fig.~\ref{fig:teaser}, left), where each node is an interaction history and each edge is one orchestration action, agents can complete complex tasks with improved controllability, compositionality, and reusability~\citep{zeng2023agenttuning,qian2024scaling}. However, in practice, orchestration still heavily relies on manual skill design and rigid action definitions~\citep{hu2024adas,agent_skills_survey2026}, making it costly to transfer across new tasks, new skill configurations, or different agent capabilities.

To address these issues, early approaches rely on heuristic orchestration (Fig.~\ref{fig:framework}a), retrieving skills from prebuilt libraries or searching over workflow graphs at test time, e.g., matching a corresponding tool to each subtask by description, or using tree search to explore multi-step action compositions~\citep{wang2023voyager,wang2024executable,zhang2025aflow,zhuge2024gptswarm}, but their orchestration policies remain bound to pre-designed heuristics and fixed skill repertoires, with no mechanism to adapt from execution outcomes. Recent work has shifted to \emph{learning-based orchestration} (Fig.~\ref{fig:framework}b)~\citep{zeng2023agenttuning,chen2023fireact,shao2024deepseekmath,deepseek2025r1,pan2026agentflow,zhang2026flowsteer,wang2026workflowr1,wang2025ragen,xia2026skillrl}, which trains orchestration policies directly from task-completion signals---sampling orchestration trajectories, collecting terminal rewards such as answer correctness or task success rate, and updating the orchestration model accordingly.

However, these methods still face three challenges.
\textbf{(i) Strategy collapse.} REINFORCE-family objectives, which uniformly raise or lower all action probabilities along a trajectory based on a single terminal reward, converge to a single reward-maximizing mode, forfeiting diverse, equally effective strategies~\citep{yu2025dapo,deepseek2025r1,zhang2026flowsteer,wang2026workflowr1,wang2025ragen}. This brittleness is amplified when executor capabilities evolve---once a tool is upgraded, the agent has no suitable fallback strategy~\citep{zhang2023robust_scheduling,li2025choice_divergence}.
\textbf{(ii) High gradient variance and opaque credit assignment.} Terminal-only rewards create long credit-assignment horizons, where policy-gradient variance and intra-group advantage collapse make per-step attribution opaque, when a multi-step trajectory succeeds or fails, it remains unclear which step is responsible~\citep{schulman2016gae,hcapo2026,process_outcome_credit2026,pan2026agentflow,wang2025ragen}.
\textbf{(iii) Unguided skill evolution.} Existing frameworks with dynamic skill libraries rely on heuristic triggers, fixed schedules, or direct LLM-as-judge prompting, lacking a principled signal for \emph{when} to update the library, \emph{what} skills to prune or create, or \emph{where} the critical decision points lie---so the library evolves blindly~\citep{wang2023voyager,xia2026skillrl,self_evolving_survey2025,agent_skills_survey2026,alzubi2026evoskill,coevoskills2026,sage2026skill_rl,skillx2026}.

To address these challenges, we propose \textbf{SkillFlow} (Fig.~\ref{fig:framework}c), a flow-based framework~\citep{bengio2021gflownet,bengio2023gflownet_foundations,malkin2022trajectory} for general task orchestration with recursive skill evolution. A trainable Supervisor interacts with a structured environment containing a dynamic skill library and a frozen executor.
At the core of SkillFlow (Fig.~\ref{fig:teaser}) is \textbf{Tempered Trajectory Balance (TTB)}, a regression-style flow-matching loss that drives each trajectory's sampling probability to be \emph{proportional to its reward}, rather than concentrated on a single best outcome. In contrast, REINFORCE-family objectives push nearly all probability mass onto one reward-maximizing trajectory; this reward-proportional sampling instead keeps multiple high-reward sub-trajectories---distinct successful paths through the orchestration DAG---alive under a single loss, preserving strategic diversity.
The same loss jointly trains a \emph{backward policy} that, once an orchestration terminates, attributes the trajectory-level outcome to its individual steps at zero additional inference cost---flagging which decisions actually drove success and which were incidental.
Building on these per-step and per-skill credit signals, a recursive skill evolution mechanism answers three questions directly from the training signal itself: \emph{when} the current library starts limiting performance (signaled by TTB convergence), \emph{what} skills to create or prune (ranked by the skill marginal flow $\hat{F}(s)$), and \emph{where} decision gaps lie (localized by the step importance $I(t)$)---closing the loop from training signal to autonomous capability growth.

We evaluate on benchmarks across question answering, mathematical reasoning, interactive decision making, and code generation. Results show SkillFlow outperforms direct LLMs, REINFORCE-style RL baselines~\citep{pan2026agentflow,zhang2026flowsteer,xia2026skillrl}, and skill-evolution methods~\citep{xskill2026,autoskill2026,skillclaw2026,alzubi2026evoskill} in task accuracy, strategy diversity, and orchestration cost---a foundation for self-evolving, flow-based agentic orchestration.


\section{Related Work}\label{sec:related}

\textbf{Agent Task Orchestration.} LLM-era task orchestration has evolved from rule-based automation to feedback-driven plan--act--feedback loops~\citep{yao2023react,wang2024executable,openhands2025}. Existing work spans single-agent sequential decision making~\citep{qin2024toolllm,yao2023react}, LLM-controller-based routing and constrained API planning~\citep{ong2024routellm,wang2024executable,daao2026}, and multi-agent SOP/role collaboration~\citep{hong2024metagpt,wu2023autogen,dang2025evolving_orchestration}. Reusable skill memory~\citep{wang2023voyager,agent_skills_survey2026,sok_agentic_skills2026,skillx2026} and self-evolving architectures~\citep{self_evolving_survey2025,single_agent_skills2026,alzubi2026evoskill,coevoskills2026,sage2026skill_rl,skillsd2026} further improve efficiency, yet all these approaches operate with fixed skill sets and lack a principled mechanism to expand the action space during training. SkillFlow fills this gap by learning orchestration from execution feedback while autonomously evolving its skill repertoire.

\textbf{Reinforcement Learning for Agents.} Agent RL models multi-turn interaction as a long-horizon MDP~\citep{zhou2024lats,rl_long_horizon_agents2026}. Credit-assignment advances include implicit step rewards~\citep{shen2026token_level}, hindsight advantage estimation~\citep{hcapo2026}, plan-execute decomposition~\citep{hiper2026}, progressive reward shaping~\citep{xi2025prs}, and process reward models~\citep{zhu2025prm_lessons,rprm2025}. Policy optimization has converged on GRPO-style objectives~\citep{deepseek2025r1,shao2024deepseekmath} with extensions (DAPO~\citep{yu2025dapo}, VAPO~\citep{liu2025vapo}, multi-agent variants~\citep{magrpo2025,graph_grpo2026}), also applied to workflow orchestration~\citep{zhang2026flowsteer,pan2026agentflow} and search-augmented reasoning~\citep{jin2025searchr1,int_credit2026}. All rely on REINFORCE-family objectives that converge to a single mode, lacking diversity-preserving signals~\citep{li2025choice_divergence}. While reward-matching flow training has been explored in LLM reasoning~\citep{xu2024flow_of_reasoning} and robust scheduling~\citep{zhang2023robust_scheduling}, SkillFlow targets a different setting---multi-turn agentic orchestration with a co-evolving skill library---and contributes reward-proportional trajectory sampling together with zero-cost per-step credit.


\section{Preliminaries}\label{sec:preliminaries}

\textbf{Definition 1: Orchestration State Graph.}
We model the task orchestration process as a directed acyclic graph $\mathcal{G} = (\mathcal{V}, \mathcal{A})$, where each vertex $H_t \in \mathcal{V}$ represents an interaction history and each edge $(H_{t-1}, H_t) \in \mathcal{A}$ corresponds to an orchestration action.
Acyclicity follows from the state update rule $H_t = H_{t-1} \oplus (r_t, a_t, o_t^{\text{exec}})$, which yields strict history growth $|H_t| > |H_{t-1}|$.

\textbf{Definition 2: Orchestration Trajectory.}
A complete path through $\mathcal{G}$ from initial state $H_0$ to a terminal state defines an orchestration trajectory:
\begin{equation}\label{eq:trajectory}
  \tau = \bigl\{(r_t,\; a_t,\; o_t^{\text{exec}})\bigr\}_{t=1}^{T}
  \;\Rightarrow\; y_q,
\end{equation}
where $r_t$ denotes the reasoning reflection at step $t$, $a_t = (\alpha_t, o_t)$ is the action with type $\alpha_t \in \{\texttt{skill},\, \texttt{act},\, \texttt{accept}\}$ and parameters $o_t$, and $o_t^{\text{exec}}$ is the execution feedback.
The episode terminates when $\alpha_t = \texttt{accept}$ or $t = T_{\max}$.

\textbf{Problem Statement.}
Given task $q \in \Dset$, environment $\mathcal{E}$, and executor $\mathcal{M}_{\text{exec}}$, we augment $\mathcal{G}$ with a non-negative flow function $F: \mathcal{V} \to \R_{\geq 0}$ satisfying conservation (incoming flow equals outgoing flow at each non-terminal state), with terminal condition $F(x) = \Rtilde(\tau_x)^\beta$.
The resulting flow network~\citep{bengio2021gflownet,bengio2023gflownet_foundations} (formal foundations in Appendix~\ref{app:gflownet_foundations}) induces a policy that samples trajectories in proportion to reward:
\begin{equation}\label{eq:problem}
  \pi^*(\tau \mid q) \propto \Rtilde(\tau)^\beta,
  \quad
  \Rtilde(\tau) = R(\tau) + \varepsilon_{\min},
\end{equation}
where $R(\tau)$ is task completion quality (per-task definitions in Appendix~\ref{app:reward}), $\varepsilon_{\min} > 0$ is a small constant that shifts rewards to be strictly positive---a prerequisite of flow networks---so that even zero-reward trajectories receive non-zero flow (details in Appendix~\ref{app:epsilon_smoothing}), and $\beta > 0$ controls the diversity--quality tradeoff.
This is equivalent to entropy-regularized RL with temperature $T = 1/\beta$ (Appendix~\ref{app:rl_equivalence}).

\begin{figure}[!t]
\centering
\includegraphics[width=\textwidth]{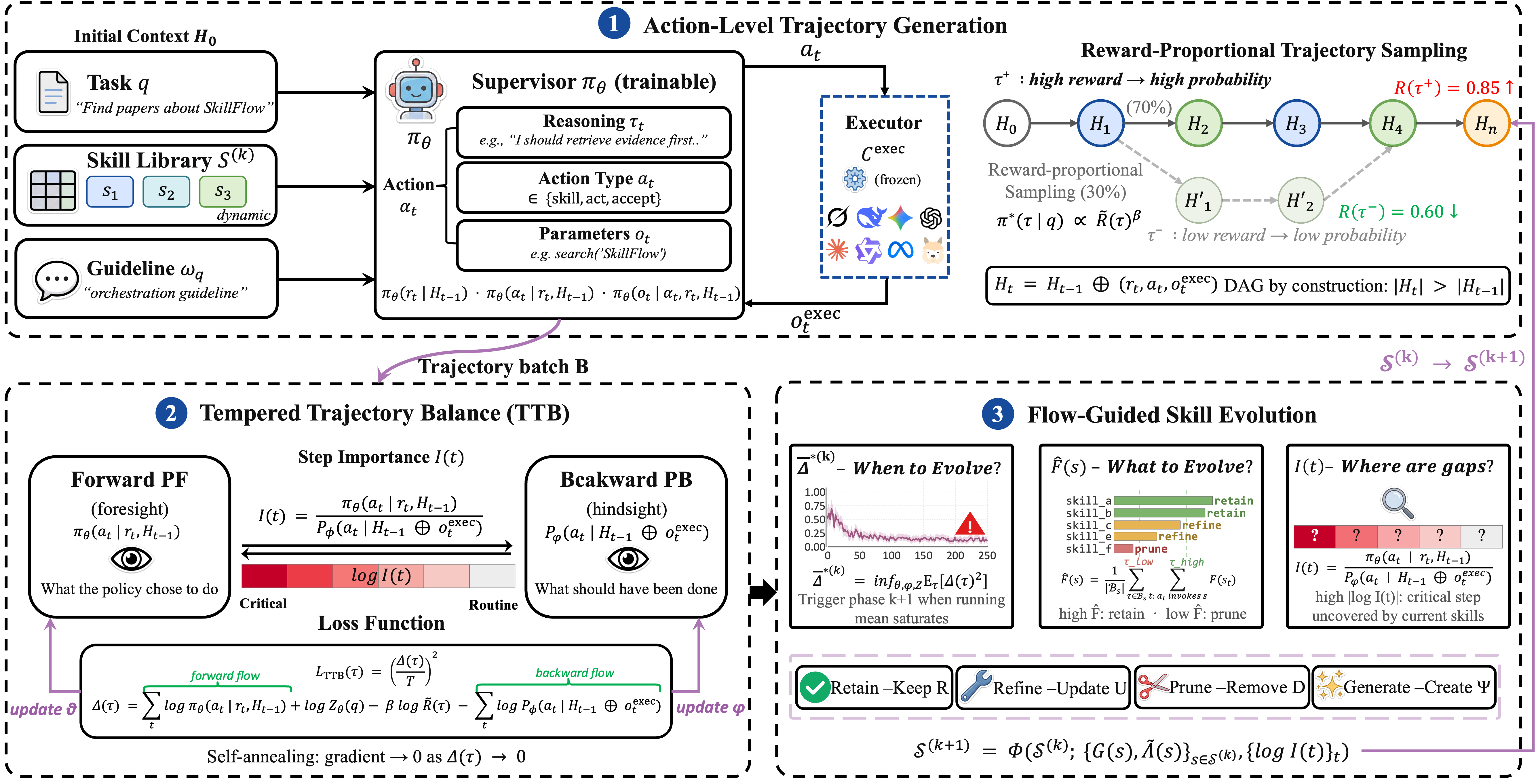}
\caption{SkillFlow architecture. The Supervisor rolls out a tree-structured DAG against a frozen Executor; TTB jointly trains the forward and hindsight backward policies via a flow-matching residual loss; flow diagnostics ($\bar{\Delta}^{*(k)}$, $\hat{F}(s)$, $I(t)$) drive recursive skill curation at phase boundaries.}
\label{fig:architecture}
\afterfloatsep[-1.2]
\end{figure}

\section{Methodology: SkillFlow}\label{sec:methodology}

As illustrated in Figure~\ref{fig:architecture}, this section introduces the SkillFlow framework, including environment design and task modeling (Section~\ref{sec:modeling}), flow-based end-to-end training (Section~\ref{sec:training}), and flow-driven recursive skill evolution (Section~\ref{sec:evolution}).

\beforesubsectionsep
\subsection{Environment Design and Task Modeling}\label{sec:modeling}

SkillFlow follows a Supervisor-Executor paradigm~\citep{yao2023react}: a trainable Supervisor $\pi_\theta$ interacts with the structured environment $\mathcal{E}$ to construct orchestration trajectories. Unlike static frameworks that operate on a fixed skill set, SkillFlow's environment carries a \textbf{dynamic skill library} that co-evolves with the policy through the curation operator $\Phi$ formalised in §\ref{sec:evolution}.

\textbf{Structured Environment $\mathcal{E}$.}
The environment maintains the skill library, skill creator, and executor:
\begin{equation}\label{eq:env}
  \mathcal{E} = (\Slib,\; \Psi,\; \mathcal{M}_{\text{exec}}),
\end{equation}
where $\Slib$ is the dynamic skill library, $\Psi$ is the Skill Creator that evolves $\Slib$ via $s_{\text{new}} = \Psi(c, \mathcal{T}, \Slib)$ using creation context $c$ and trajectory evidence $\mathcal{T}$ (Section~\ref{sec:evolution}), and $\mathcal{M}_{\text{exec}}$ is the pluggable executor.
$\Psi$ updates $\Slib$ only at episode boundaries, keeping the action space constant within each training phase.

\textbf{State Space $\mathcal{H}$ and Orchestration Target.}
The initial state concatenates the task with retrieved context: $H_0 = [q \oplus \Slib_{\text{ret}} \oplus \omega_q]$, where $\Slib_{\text{ret}}$ are retrieved skills and $\omega_q$ is a task-category-specific orchestration guideline. The state evolves via
\begin{equation}\label{eq:state_update}
  H_t = H_{t-1} \oplus (r_t,\, a_t,\, o_t^{\text{exec}}).
\end{equation}
The Supervisor interacts with $\mathcal{E}$ until $\alpha_t = \texttt{accept}$ (early termination) or $t = T_{\max}$ (budget exhausted); the resulting trajectory $\tau$ determines the answer $y_q$ via the terminal action's output.

\textbf{Step-wise Orchestration Policy.}
At each step $t$, the Supervisor generates reasoning $r_t$, selects action type $\alpha_t$, and produces parameters $o_t$, modeled as a hierarchical policy conditioned on $H_{t-1}$:
\begin{equation}\label{eq:policy}
  \pi_\theta(r_t, a_t \mid H_{t-1})
  = \pi_\theta(r_t \mid H_{t-1})
  \!\cdot\! \pi_\theta(\alpha_t \mid r_t, H_{t-1})
  \!\cdot\! \pi_\theta(o_t \mid \alpha_t, r_t, H_{t-1}).
\end{equation}

\textbf{Multi-Turn Interaction and Trajectory Distribution.}
Given action $a_t$, the executor returns feedback $o_t^{\text{exec}} \sim \mathcal{C}_{\text{exec}}(\cdot \mid H_{t-1}, a_t)$, the state evolves via Eq.~\ref{eq:state_update}, and the Supervisor continues until termination.
Marginalising the joint policy--executor draws over the $T$ steps yields the trajectory distribution:
\begin{equation}\label{eq:traj_dist}
  P_\theta(\tau) = \prod_{t=1}^{T}\bigl[\pi_\theta(r_t, a_t \mid H_{t-1})
  \cdot \mathcal{C}_{\text{exec}}(o_t^{\text{exec}} \mid H_{t-1}, a_t)\bigr],
\end{equation}
where only $\pi_\theta$ is trainable and $\mathcal{C}_{\text{exec}}$ is frozen.

\begin{proposition}\label{prop:modeling}
The Supervisor--Executor environment admits a flow-conservative DAG structure suitable for end-to-end flow-based training.
\end{proposition}
\emph{Proof.}\; Main-result gains (§\ref{subsec:main-results}, RQ1) and cross-backbone transferability (§\ref{subsec:transferability}, RQ3) provide empirical validation; formal acyclicity in Appendix~\ref{app:dag_acyclicity}. \hfill$\square$

\beforesubsectionsep
\subsection{Flow-Based End-to-End Training}\label{sec:training}

Building on the DAG structure of $\mathcal{G}$ (Proposition~\ref{prop:modeling}), we train the flow network introduced in §3 to make the action-sequence distribution reward-proportional given the realized reasoning and execution context: $\pi_\theta(a_{1:T} \mid r_{1:T}, o^{\text{exec}}_{1:T}, q) \propto \Rtilde(\tau)^\beta$.

\textbf{Forward Policy $P_F$ and Backward Policy $P_\phi$.}
Because $\mathcal{C}_{\text{exec}}$ is frozen and reasoning $r_t$ is fixed context, the forward policy reduces to action selection, while the backward policy $P_\phi$ conditions on the \textbf{hindsight state} $H_{t-1} \oplus o_t^{\text{exec}}$ to incorporate the execution observation unavailable to $\pi_\theta$:
\begin{equation}\label{eq:fb_policies}
  P_F(H_t \!\mid\! H_{t-1}) = \pi_\theta(a_t \!\mid\! r_t, H_{t-1}), \quad
  P_B(H_{t-1} \!\mid\! H_t) = P_\phi\!\left(a_t \;\middle|\; \underbrace{H_{t-1} \oplus o_t^{\text{exec}}}_{\text{hindsight state}}\right).
\end{equation}

\textbf{Tempered Trajectory Balance (TTB).}
Given $P_F$ and $P_\phi$, the (tempered, hindsight-conditioned) \textbf{Trajectory Balance} (TB) condition~\citep{bengio2023gflownet_foundations} requires for every trajectory:
\begin{equation}\label{eq:tb_condition}
  \log Z_\theta(q) + \sum_{t=1}^{T} \log P_F(H_t \!\mid\! H_{t-1})
  = \beta\log \Rtilde(\tau) + \sum_{t=1}^{T} \log P_B(H_{t-1} \!\mid\! H_t).
\end{equation}
The \textbf{TTB loss} is the squared, length-normalized residual of this condition~\citep{ace_gflownet2026} (full derivation in Appendix~\ref{app:ttb_derivation}):
\begin{gather}
  \Delta(\tau) \coloneqq \log Z_\theta(q) + \sum_{t=1}^{T} \log \pi_\theta(a_t \mid r_t, H_{t-1}) - \beta \log \Rtilde(\tau) - \sum_{t=1}^{T} \log P_\phi\bigl(a_t \mid H_{t-1} \oplus o_t^{\text{exec}}\bigr), \notag\\
  \Lttb(\tau) = \bigl(\Delta(\tau)/T\bigr)^{2}. \label{eq:ttb}
\end{gather}
Here $T = |\tau|$, $Z_\theta(q)$ is a task-conditioned partition function, and each per-step log-probability is \emph{per-token normalized} (details in Appendix~\ref{app:ttb_loss}). At the optimum $\Delta(\tau)=0$, the conditional action-sequence distribution becomes reward-proportional, $\pi_\theta(a_{1:T} \mid r_{1:T}, o^{\text{exec}}_{1:T}, q) \propto \Rtilde(\tau)^\beta$.

\textbf{Step Importance and Skill Marginal Flow.}
Each $H_t$ has a unique parent; $\mathcal{G}$ is therefore tree-structured, and TB convergence implies Detailed Balance $F(H)\,P_F(H'|H) = F(H')\,P_B(H|H')$ at every edge (formal treatment in Appendix~\ref{app:detailed_balance}). Rearranging yields the \textbf{step importance}:
\begin{equation}\label{eq:step_importance}
  I(t) = \frac{F(H_t)}{F(H_{t-1})} = \frac{\pi_\theta(a_t \mid r_t, H_{t-1})}{P_\phi(a_t \mid H_{t-1} \oplus o_t^{\text{exec}})},
\end{equation}
which incurs no extra inference cost. The \textbf{information asymmetry}---$\pi_\theta$ decides without $o_t^{\text{exec}}$, while $P_\phi$ evaluates with it---makes $I(t)$ a credit signal: large $|\log I(t)|$ marks decisions whose appraisal shifted after execution.
Telescoping from $F(s_0) = Z_\theta(q)$ gives the \textbf{skill marginal flow}:
\begin{equation}\label{eq:skill_flow}
  \hat{F}(s) = \frac{1}{|\Batch_s|}
  \sum_{\tau \in \Batch_s}\;
  \sum_{t:\, a_t \text{ invokes } s} F(H_t),
  \qquad \log F(H_t) = \log Z_\theta(q) + \sum_{t'=1}^{t} \log I(t'),
\end{equation}
where $\Batch_s \subseteq \Batch$ is the subset of trajectories invoking $s$ and the inner sum aggregates flow over each occurrence of $s$ within a trajectory.

\begin{proposition}\label{prop:training}
TTB training induces reward-proportional sampling and yields per-step credit at no extra inference cost.
\end{proposition}
\emph{Proof.}\; Main results (§\ref{subsec:main-results}, RQ1), OOD (§\ref{subsec:ood}, RQ2), $-$TTB ablation (§\ref{subsec:ablation}, RQ4), and algorithm comparison (§\ref{subsec:algo-cost}, RQ5) provide empirical validation; full proof in Appendix~\ref{app:proof_prop2}. \hfill$\square$

\beforesubsectionsep
\subsection{Flow-Driven Recursive Skill Evolution}\label{sec:evolution}

Prior work on skill distillation and refinement~\citep{wang2023voyager,alzubi2026evoskill,coevoskills2026,skillsd2026,xia2026skillrl} relies on heuristic schedules or LLM-as-judge prompting, leaving \emph{when}, \emph{what}, and \emph{where} to evolve underspecified. SkillFlow derives all three from the flow signals of Section~\ref{sec:training}: the TTB residual $\Delta(\tau)$ signals \emph{when}, while the step importance $I(t)$ and skill marginal flow $\hat{F}(s)$ localize \emph{what} and \emph{where}.

\textbf{When: TTB Residual Floor.}
Within phase $k$, gradient descent drives $\Delta(\tau)^2 \to 0$ (Proposition~\ref{prop:training}). The squared-residual floor under the current library is
\begin{equation}\label{eq:residual_bound}
  \bar{\Delta}^{*(k)} \;\coloneqq\; \inf_{\theta,\phi,Z_\theta}\; \E_{\tau}\!\bigl[\Delta(\tau \mid \Slib^{(k)},\,\theta,\phi,Z_\theta)^2\bigr] \;\ge\; 0.
\end{equation}
$\bar\Delta^{*(k)} = 0$ when $\Slib^{(k)}$ together with the policy class can express the reward-proportional flow; otherwise the loss plateaus at $\bar\Delta^{*(k)} > 0$. Phase $k{+}1$ is triggered when the running mean of $\Delta(\tau)^2$ saturates against this plateau.

\textbf{What and Where: Flow-Guided Skill Curation via a CGF.}
Each skill $s \in \Slib$ is an \emph{atomic tip}---a short, self-contained piece of strategic guidance, independently composable via the $\texttt{skill}$ action.
We define the \textbf{per-skill cumulant generating function} (CGF) of the telescoped log-flow:
\[
  \Lambda^{(s)}_{\lambda} \;\coloneqq\; \log\!\left(\frac{1}{|\Batch_s|}\sum_{\tau\in\Batch_s}\sum_{t:\,a_t\text{ invokes }s}\!\exp\!\Bigl(\lambda\!\sum_{t'=1}^{t}\log I(t')\Bigr)\right),\quad \lambda\in\R\ \text{(moment order)},
\]
which, by the telescoping identity $\log F(H_t) = \log Z_\theta(q) + \sum_{t'\le t}\log I(t')$ (Eq.~\ref{eq:skill_flow}), is the log $\lambda$-th moment of $F(H_t)/Z_\theta(q)$ along occurrences of $s$.
$\Lambda^{(s)}_{\lambda}$ is convex in $\lambda$; two summaries derived from it drive library evolution---the \textbf{mean log-flow}
$G(s) \coloneqq \tfrac{\partial\Lambda^{(s)}_{\lambda}}{\partial\lambda}\big|_{\lambda=0}$
measures the average flow $s$ attracts when invoked, and the \textbf{centered log-flow share}
$\widetilde\Lambda(s) \coloneqq \Lambda^{(s)}_{1} - \E_{s'}[\Lambda^{(s')}_{1}]$
ranks $s$'s marginal contribution to the library's reward-proportional sampling:
\begin{equation}\label{eq:evolution}
  \Slib^{(k+1)} \;=\; \Phi\bigl(\Slib^{(k)};\,\{G(s),\,\widetilde\Lambda(s)\}_{s\in\Slib^{(k)}},\,\{\log I(t)\}_t\bigr),
\end{equation}
where $\Lambda^{(s)}_{1} = \log\hat{F}(s)-\log Z_\theta(q)$ recovers the skill marginal flow in log-space. The operator $\Phi$ partitions $\Slib^{(k)}$ into four disjoint classes: \emph{retain} (high $G(s)$, small Jensen gap), \emph{refine} (high $G(s)$, persistently large Jensen gap), and \emph{prune} (persistently negative $\widetilde\Lambda(s)$); on top of these it invokes the Skill Creator $\Psi$ to \emph{generate} new atomic tips at high-$\log I(t)$ steps from same-query success/failure pairs $(\tau^+,\tau^-)$. The Jensen gap $\Lambda^{(s)}_{1}-G(s)$ equals the cross-visit variance of $\log F(H_t)$ plus higher-order cumulants, serving as a stability diagnostic that distinguishes \emph{retain} from \emph{refine} for context-inconsistent skills.

\begin{proposition}\label{prop:evolution}
Flow-driven recursive evolution autonomously expands the skill library while preserving its atomic composability.
\end{proposition}
\emph{Proof.}\; Leave-one-out ablation (§\ref{subsec:ablation}, RQ4) and skill-evolution cost savings (§\ref{subsec:algo-cost}, RQ5) provide empirical validation; full proof in Appendix~\ref{app:proof_prop3}. \hfill$\square$


\begin{table}[t]
\centering
\small

\tabcolsep=2pt
\renewcommand{\arraystretch}{1.15}

\resizebox{\linewidth}{!}{
\begin{tabular}{@{}ll|cccccccc|c@{}}
\toprule
\textbf{} & \textbf{} & \multicolumn{2}{c}{\textbf{Baseline}} & \textbf{SFT} & \textbf{GRPO} & \textbf{AFlow} & \multicolumn{3}{c|}{\textbf{Agent+RL}} & \textbf{Ours} \\
\cmidrule(lr){3-4}\cmidrule(lr){5-5}\cmidrule(lr){6-6}\cmidrule(lr){7-7}\cmidrule(lr){8-10}\cmidrule(lr){11-11}
\textbf{Dataset} & \textbf{Metric} & \textbf{Qwen3.5} & \textbf{v4-flash} & \textbf{Qwen3.5} & \textbf{Qwen3.5} & \textbf{Qwen3.5} & \textbf{AgentFlow} & \textbf{FlowSteer} & \textbf{SkillRL} & \textbf{SkillFlow ($\Delta\uparrow$)} \\
\midrule
\multicolumn{11}{@{}l}{\textit{(a) In-Distribution (IID) benchmarks}} \\
\midrule
\textbf{HotpotQA} & Ans EM & 60.94\std{0.8} & 72.66\std{0.1} & 64.84\std{0.4} & 69.53\std{0.3} & 88.28\std{0.9} & 88.28\std{0.8} & 89.84\std{1.1} & 87.50\std{0.2} & \textbf{92.19}\,\textcolor{darkgreen}{(+31.3)} \\
\textbf{} & Ans F1 & 75.70\std{0.1} & 85.95\std{0.3} & 79.08\std{0.7} & 81.52\std{0.1} & 90.92\std{0.3} & 90.11\std{0.8} & 91.20\std{0.7} & 89.16\std{0.3} & \textbf{93.95}\,\textcolor{darkgreen}{(+18.3)} \\
\textbf{TriviaQA} & Ans EM & 44.88\std{1.0} & 74.02\std{0.1} & 47.24\std{1.0} & 47.24\std{0.9} & 92.97\std{0.5} & 89.84\std{0.3} & 91.41\std{1.2} & 89.06\std{0.5} & \textbf{96.09}\,\textcolor{darkgreen}{(+51.2)} \\
\textbf{} & Ans F1 & 54.30\std{0.2} & 81.99\std{1.0} & 55.73\std{0.8} & 56.93\std{1.0} & 93.25\std{0.9} & 90.47\std{0.7} & 92.06\std{1.2} & 90.63\std{0.5} & \textbf{96.94}\,\textcolor{darkgreen}{(+42.6)} \\
\textbf{AIME 2026} & Acc. & 46.67\std{1.0} & 50.00\std{0.8} & 36.67\std{1.0} & 33.33\std{0.7} & 53.33\std{0.9} & 60.00\std{0.2} & 63.33\std{0.4} & 56.67\std{0.4} & \textbf{70.00}\,\textcolor{darkgreen}{(+23.3)} \\
\textbf{MedQA} & Acc. & 69.53\std{0.4} & 83.59\std{0.2} & 73.44\std{0.4} & 77.34\std{0.8} & 89.84\std{0.5} & 87.50\std{0.5} & 90.63\std{0.3} & 85.94\std{0.4} & \textbf{92.19}\,\textcolor{darkgreen}{(+22.7)} \\
\textbf{WebShop} & Avg Score & 56.93\std{0.8} & 85.17\std{0.8} & 50.80\std{0.3} & 54.10\std{0.9} & 71.09\std{0.3} & 83.61\std{0.5} & 85.43\std{1.2} & 87.96\std{0.8} & \textbf{94.73}\,\textcolor{darkgreen}{(+37.8)} \\
\textbf{} & SR & 32.03\std{0.9} & 67.97\std{1.0} & 35.16\std{1.0} & 36.71\std{0.4} & 59.38\std{0.1} & 74.22\std{0.4} & 79.20\std{0.4} & 82.81\std{0.3} & \textbf{93.75}\,\textcolor{darkgreen}{(+61.7)} \\
\textbf{ALFWorld} & SR & 48.28\std{1.1} & 61.21\std{0.4} & 40.52\std{0.8} & 50.78\std{0.5} & 75.00\std{1.1} & 80.47\std{0.6} & 82.81\std{0.4} & 85.16\std{0.4} & \textbf{96.09}\,\textcolor{darkgreen}{(+47.8)} \\
\textbf{SWE-bench} & Resolved & 17.19\std{0.4} & 38.28\std{0.7} & 16.41\std{1.1} & 18.75\std{0.5} & 30.47\std{0.3} & 41.41\std{1.2} & 43.75\std{0.7} & 39.06\std{0.2} & \textbf{52.34}\,\textcolor{darkgreen}{(+35.2)} \\
\cmidrule(lr){1-11}
\multirow{3}{*}{\textbf{Avg.\,(IID)}} & Ans EM    & 52.91\std{0.9} & 73.34\std{0.1} & 56.04\std{0.7} & 58.39\std{0.6} & 90.63\std{0.7} & 89.06\std{0.6} & 90.63\std{1.2} & 88.28\std{0.4} & \textbf{94.14}\,\textcolor{darkgreen}{(+41.2)} \\
                                     & Ans F1    & 65.00\std{0.2} & 83.97\std{0.7} & 67.41\std{0.8} & 69.23\std{0.6} & 92.09\std{0.6} & 90.29\std{0.8} & 91.63\std{1.0} & 89.90\std{0.4} & \textbf{95.45}\,\textcolor{darkgreen}{(+30.5)} \\
                                     & Acc./Pass & 45.11\std{0.8} & 64.37\std{0.7} & 42.17\std{0.8} & 45.17\std{0.6} & 63.19\std{0.5} & 71.20\std{0.6} & 74.19\std{0.6} & 72.93\std{0.4} & \textbf{83.18}\,\textcolor{darkgreen}{(+38.1)} \\
\midrule
\multicolumn{11}{@{}l}{\textit{(b) Out-of-Distribution (OOD) benchmarks}} \\
\midrule
\textbf{MuSiQue} & Ans EM & 39.06\std{0.2} & 50.78\std{0.8} & 39.10\std{1.0} & 42.20\std{0.6} & 78.13\std{0.2} & 79.69\std{0.5} & 79.69\std{1.2} & 75.78\std{0.7} & \textbf{85.16}\,\textcolor{darkgreen}{(+46.1)} \\
\textbf{} & Ans F1 & 47.79\std{1.0} & 59.32\std{0.1} & 47.90\std{0.9} & 50.30\std{0.8} & 84.38\std{0.7} & 84.90\std{0.4} & 85.22\std{0.8} & 81.38\std{0.2} & \textbf{89.29}\,\textcolor{darkgreen}{(+41.5)} \\
\textbf{NQ-Open} & Ans EM & 21.88\std{0.6} & 37.50\std{1.1} & 23.44\std{1.1} & 21.88\std{0.4} & 75.79\std{0.7} & 78.91\std{0.3} & 80.47\std{1.1} & 78.91\std{1.1} & \textbf{82.81}\,\textcolor{darkgreen}{(+60.9)} \\
\textbf{} & Ans F1 & 30.03\std{0.8} & 47.90\std{0.8} & 30.97\std{0.3} & 28.84\std{0.9} & 80.23\std{0.7} & 83.44\std{1.0} & 84.88\std{0.7} & 83.09\std{0.1} & \textbf{86.04}\,\textcolor{darkgreen}{(+56.0)} \\
\textbf{MATH-Hard} & Acc. & 89.06\std{0.1} & 91.41\std{1.1} & 88.28\std{1.1} & 87.50\std{1.0} & 90.63\std{0.4} & 92.19\std{0.2} & 93.75\std{1.1} & 91.41\std{1.1} & \textbf{96.09}\,\textcolor{darkgreen}{(+7.0)} \\
\textbf{GPQA Diamond} & Acc. & 61.72\std{0.6} & 73.44\std{0.2} & 71.88\std{0.9} & 75.00\std{0.9} & 75.78\std{0.2} & 81.25\std{0.6} & 84.38\std{0.7} & 78.13\std{0.4} & \textbf{89.84}\,\textcolor{darkgreen}{(+28.1)} \\
\textbf{HumanEval} & pass@1 & 89.06\std{0.6} & 96.88\std{0.3} & 80.47\std{0.7} & 85.94\std{0.9} & 92.97\std{0.3} & 93.75\std{0.4} & 93.75\std{1.2} & 92.19\std{0.8} & \textbf{98.44}\,\textcolor{darkgreen}{(+9.4)} \\
\textbf{ScienceWorld} & Success & 25.78\std{0.7} & 43.43\std{0.2} & 15.62\std{0.3} & 24.22\std{0.5} & 29.69\std{0.7} & 44.53\std{0.4} & 47.66\std{0.3} & 39.06\std{0.2} & \textbf{57.81}\,\textcolor{darkgreen}{(+32.0)} \\
\textbf{} & Avg Score & 41.86\std{0.4} & 58.64\std{1.1} & 34.69\std{1.0} & 40.00\std{0.2} & 43.65\std{0.4} & 56.85\std{0.8} & 58.91\std{0.3} & 50.11\std{0.2} & \textbf{67.87}\,\textcolor{darkgreen}{(+26.0)} \\
\textbf{Mind2Web} & Step Acc & 26.64\std{0.7} & 27.57\std{0.6} & 18.74\std{1.0} & 23.51\std{1.0} & 30.47\std{0.3} & 39.84\std{0.2} & 41.41\std{0.6} & 35.16\std{0.6} & \textbf{51.56}\,\textcolor{darkgreen}{(+24.9)} \\
\textbf{} & Action F1 & 63.85\std{0.9} & 66.21\std{0.8} & 50.49\std{1.2} & 66.98\std{0.2} & 62.29\std{0.5} & 71.77\std{0.5} & 75.22\std{1.0} & 69.54\std{0.4} & \textbf{84.89}\,\textcolor{darkgreen}{(+21.0)} \\
\cmidrule(lr){1-11}
\multirow{3}{*}{\textbf{Avg.\,(OOD)}} & EM        & 30.47\std{0.4} & 44.14\std{1.0} & 31.27\std{1.1} & 32.04\std{0.5} & 76.96\std{0.5} & 79.30\std{0.4} & 80.08\std{1.2} & 77.34\std{0.9} & \textbf{83.99}\,\textcolor{darkgreen}{(+53.5)} \\
                                     & F1        & 47.22\std{0.9} & 57.81\std{0.6} & 43.12\std{0.8} & 48.71\std{0.6} & 75.63\std{0.6} & 80.04\std{0.6} & 81.77\std{0.8} & 78.00\std{0.2} & \textbf{86.74}\,\textcolor{darkgreen}{(+39.5)} \\
                                     & Acc./Pass & 55.69\std{0.5} & 65.23\std{0.6} & 51.61\std{0.8} & 56.03\std{0.8} & 60.53\std{0.4} & 68.07\std{0.4} & 69.98\std{0.7} & 64.34\std{0.6} & \textbf{76.94}\,\textcolor{darkgreen}{(+21.3)} \\
\bottomrule
\end{tabular}}
\vspace{5pt}
\caption{Main results on 14 IID and OOD benchmarks. SkillFlow uses Qwen3.5-9B as the Supervisor. ``Agent+RL'' covers AgentFlow, FlowSteer, SkillRL. $\Delta\uparrow$ over Qwen3.5-9B.}
\label{tab:main-results}
\label{tab:ood-results}
\afterfloatsep[-2]
\end{table}

\section{Experiments}\label{sec:experiments}

We evaluate SkillFlow through the following research questions (RQs):
\textbf{RQ1:} Can SkillFlow outperform existing workflow orchestration methods on in-distribution benchmarks?
\textbf{RQ2:} How does SkillFlow generalize to out-of-distribution benchmarks?
\textbf{RQ3:} How transferable is SkillFlow across different LLM backbones?
\textbf{RQ4:} What are the contributions of core components such as TTB training, backward policy, and skill evolution?
\textbf{RQ5:} How does SkillFlow compare against other agent RL algorithms (GRPO, Tree-GRPO, HCAPO) and skill-evolution baselines in accuracy and computational cost?

\beforesubsectionsep
\subsection{Experimental Setup}\label{subsec:setup}

\textbf{Datasets.}
We evaluate SkillFlow on 14 benchmarks covering four task categories.
\textit{In-distribution (IID, 7):} HotpotQA~\citep{yang2018hotpotqa}, TriviaQA~\citep{joshi2017triviaqa}, MedQA~\citep{jin2021medqa}, AIME 2026, WebShop~\citep{yao2022webshop}, ALFWorld~\citep{shridhar2021alfworld}, SWE-bench~\citep{jimenez2024swebench}.
\textit{Out-of-distribution (OOD, 7):} MuSiQue~\citep{trivedi2022musique}, NQ-Open, MATH-Hard, GPQA Diamond, HumanEval~\citep{chen2021humaneval}, ScienceWorld, Mind2Web. More details in Appendix~\ref{app:datasets}.

\textbf{Baselines.}
We compare against \textit{(i)} direct LLMs (Qwen3.5-9B, v4-flash, Claude Haiku 4.5), \textit{(ii)} fine-tuning (SFT, GRPO~\citep{shao2024deepseekmath}), \textit{(iii)} search-based workflows (AFlow~\citep{zhang2025aflow}), and \textit{(iv)} RL agents (AgentFlow~\citep{pan2026agentflow}, FlowSteer~\citep{zhang2026flowsteer}, SkillRL~\citep{xia2026skillrl}). See Appendix~\ref{app:baselines} for details.

\textbf{Evaluation Metrics.}
F1 for QA, Accuracy for AIME/MedQA/MATH/GPQA, Average Score and Success Rate for WebShop/ALFWorld, Resolved Rate for SWE-bench, pass@1 for HumanEval. Details in Appendix~\ref{app:metrics}.

\FloatBarrier
\begin{table}[t]
\centering
\small
\begin{adjustbox}{max width=\textwidth}\tabcolsep=1pt
\renewcommand{\arraystretch}{1.05}
\begin{tabular}{@{}l|ccccccc|ccccccc@{}}
\toprule
\textbf{Variant} & \multicolumn{7}{c|}{\textbf{IID}} & \multicolumn{7}{c}{\textbf{OOD}} \\
\cmidrule(lr){2-8}\cmidrule(lr){9-15}
 & \textbf{Hotpot} & \textbf{Trivia} & \textbf{AIME} & \textbf{MedQA} & \textbf{WShop} & \textbf{ALFW} & \textbf{SWE} & \textbf{MuSQ} & \textbf{NQ} & \textbf{MATH} & \textbf{GPQA} & \textbf{HEval} & \textbf{SciW} & \textbf{M2W} \\
 & EM & EM & Acc & Acc & SR & SR & Res & EM & EM & Acc & Acc & pass@1 & Succ & Step Acc \\
\midrule
$-$ TTB & 83.59 & 87.50 & 53.33 & 85.16 & 82.03 & 82.81 & 40.63 & 73.44 & 79.69 & 87.50 & 79.69 & 89.06 & 39.10 & 33.59 \\
\midrule
$-$ Backward policy & 85.94 & 90.63 & 53.33 & 86.72 & 80.47 & 88.28 & 43.75 & 75.00 & 78.91 & 89.06 & 83.09 & 90.63 & 42.20 & 36.72 \\
\midrule
$-$ $\dot{\mathcal{L}}_{\text{TTB}}$ for \emph{when} & 82.81 & 84.38 & 60.00 & 89.06 & 83.02 & 85.93 & 46.88 & 75.78 & 71.88 & 90.63 & 84.88 & 92.19 & 48.44 & 40.63 \\
$-$ $I(t)$ for \emph{where} & 85.16 & 88.26 & 63.33 & 89.06 & 87.50 & 90.63 & 44.53 & 78.13 & 75.00 & 88.28 & 82.81 & 93.75 & 47.90 & 42.97 \\
$-$ $\hat F(s)$ for \emph{what} & 84.38 & 82.03 & 50.00 & 83.59 & 79.69 & 80.47 & 42.20 & 76.56 & 75.78 & 87.50 & 80.47 & 90.63 & 45.31 & 42.19 \\
\midrule
\textbf{SkillFlow (Full)} & \textbf{92.19} & \textbf{96.09} & \textbf{70.00} & \textbf{92.19} & \textbf{93.75} & \textbf{96.09} & \textbf{52.34} & \textbf{85.16} & \textbf{82.81} & \textbf{96.09} & \textbf{89.84} & \textbf{98.44} & \textbf{57.81} & \textbf{51.56} \\
\bottomrule
\end{tabular}
\end{adjustbox}\vspace{5pt}
\caption{Component ablation across IID and OOD benchmarks (RQ4). Each row removes one component. \emph{$-$TTB} replaces TTB with GRPO (C1); \emph{$-$Backward policy} removes the hindsight $P_\phi$ (C2); the last three rows ablate the \emph{when}/\emph{where}/\emph{what} signals of skill evolution (C3).}
\label{tab:ablation}
\afterfloatsep
\end{table}

\begin{figure}[t]

\centering
\begin{subfigure}[c]{0.708\textwidth}
\zerosep
\centering
  \includegraphics[width=\linewidth]{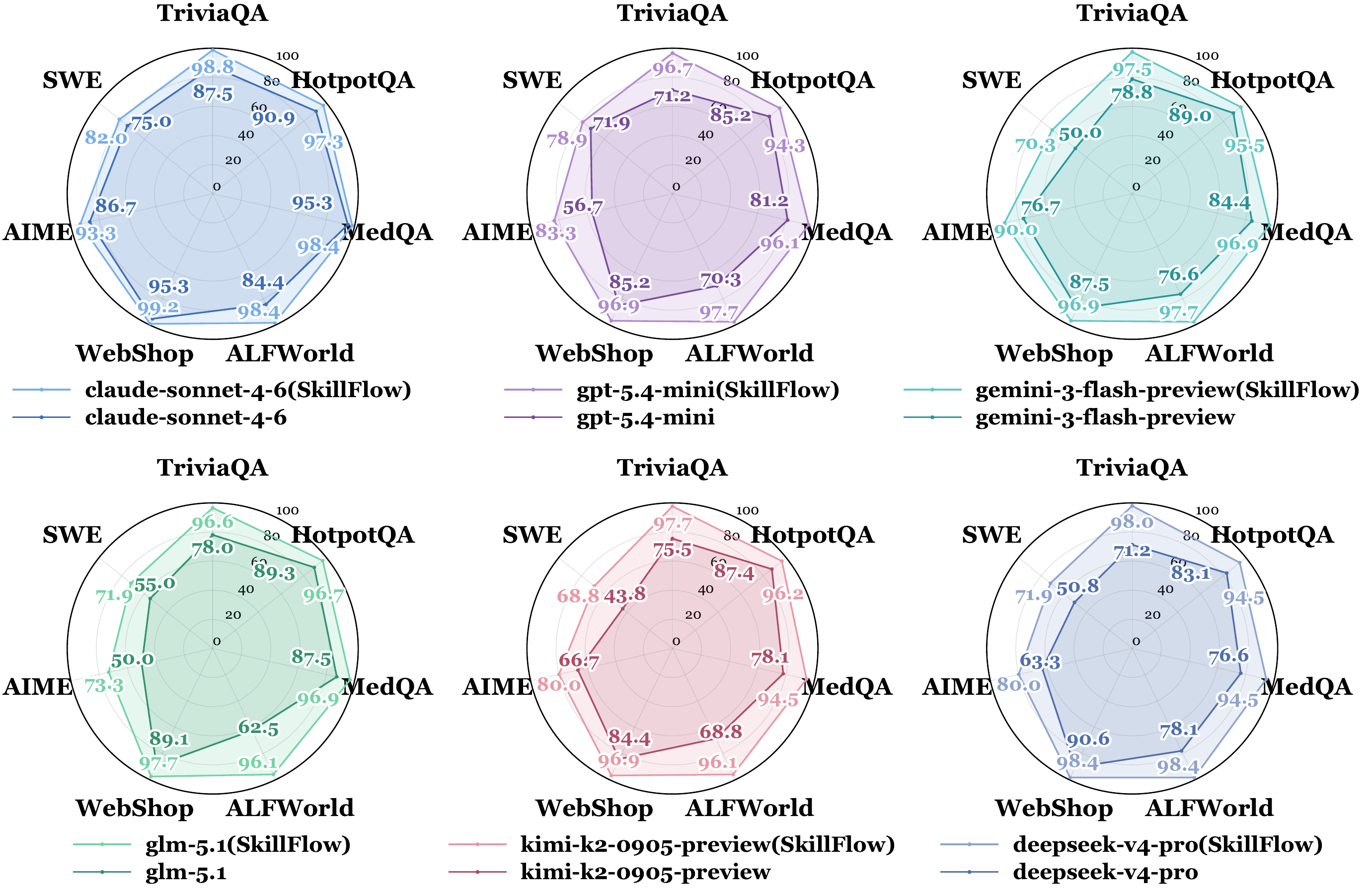}
  \caption{Radar charts on different backbones}\label{fig:transfer-radar}
\end{subfigure}\hfill
\begin{subfigure}[c]{0.271\textwidth}
\zerosep
\centering
  \includegraphics[width=\linewidth]{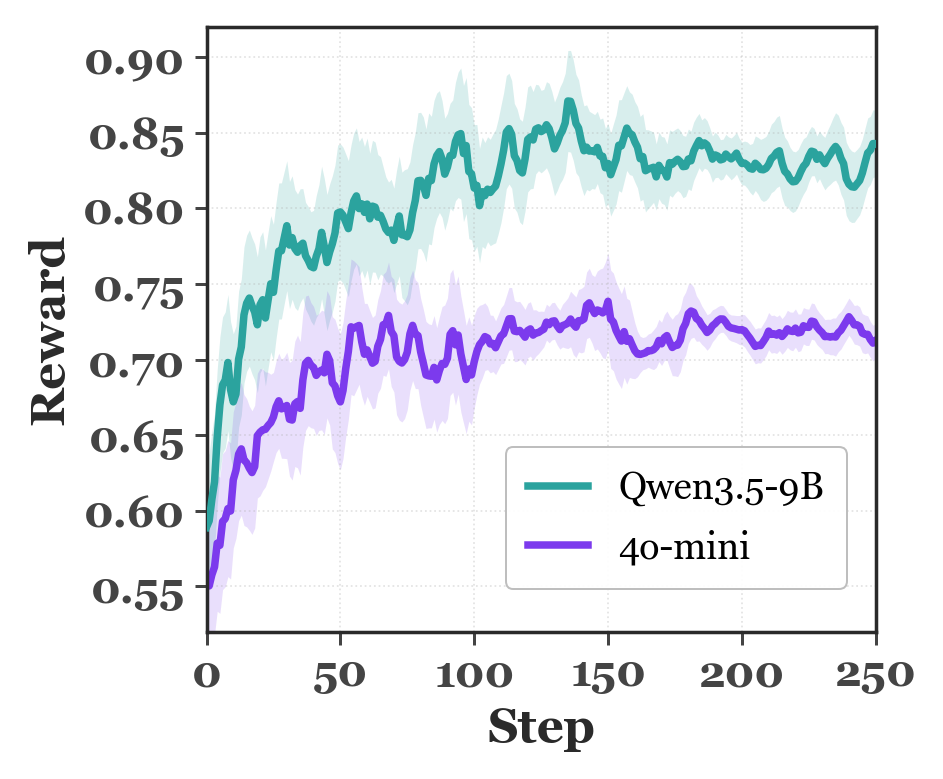}\\[0.3em]
  \includegraphics[width=\linewidth]{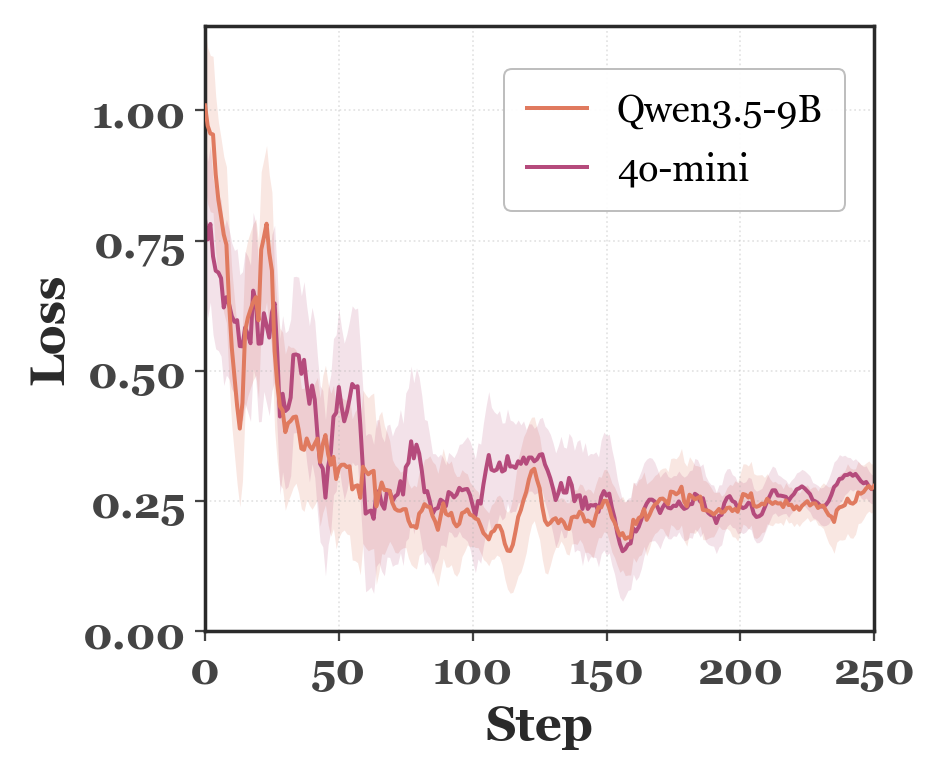}
  \caption{Training dynamics}\label{fig:transfer-dynamics}
\end{subfigure}\\[0.5em]
\begin{subfigure}{\textwidth}
  \centering
  \includegraphics[width=\linewidth]{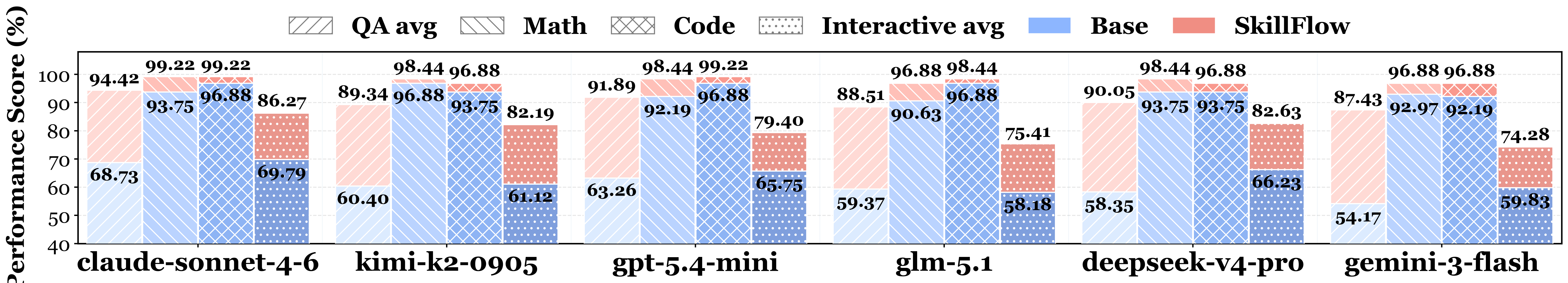}
  \caption{Aggregated performance by task type}\label{fig:transfer-bars}
\end{subfigure}
\caption{Backbone transferability. (a) Per-backbone radar across seven IID benchmarks (six proprietary LLMs), with vs.\ without SkillFlow. (b) Training dynamics on two trainable backbones. (c) Aggregated gain by task category. SkillFlow lifts every backbone, with weaker ones gaining most.}
\label{fig:transferability}
\afterfloatsep
\end{figure}

\subsection{Main Results (RQ1)}\label{subsec:main-results}

Table~\ref{tab:main-results} shows SkillFlow leads across all 14 benchmarks and surpasses even much stronger direct-LLM baselines (v4-flash, Claude Haiku 4.5), indicating gains come from how the orchestration policy is trained rather than from backbone capacity. Margins concentrate on WebShop, ALFWorld, and SWE-bench — where REINFORCE-family baselines suffer strategy collapse and AFlow's static workflow runs out of moves — and persist across both IID and OOD halves. The gains compound from three mechanisms: TTB's regression loss has lower variance than REINFORCE (formal bound in Appendix~\ref{app:variance}), the backward policy supplies zero-cost per-step credit, and recursive skill evolution removes the static-library bottleneck of AgentFlow, FlowSteer, and SkillRL.

\beforesubsectionsep
\subsection{Out-of-Distribution Generalization (RQ2)}\label{subsec:ood}

The seven OOD benchmarks (Table~\ref{tab:main-results}(b)) share the four task categories with the IID set but contain no training data, and the skill library is frozen at end-of-training — a strict transfer test. SkillFlow retains its IID lead with comparable margins; notably, the gap to REINFORCE-style baselines \emph{widens} on OOD because reward-proportional sampling preserves multiple solution paths, whereas a single-mode policy fails once surface forms shift. On Avg.(OOD) F1, the lift over FlowSteer reaches $+6.08\%$---roughly $1.5\times$ the IID gap. That a frozen library still helps unseen tasks indicates the evolved skills capture transferable orchestration primitives rather than benchmark-specific shortcuts.

\begin{figure}[t]
\vspace{-1.5em}
\centering
\begin{subfigure}[t]{0.65\textwidth}
\zerosep
\centering
\begin{adjustbox}{max width=\linewidth}
\footnotesize
\setlength{\tabcolsep}{3pt}
\renewcommand{\arraystretch}{1.10}
\begin{tabular}{@{}l|ccccccc@{}}
\toprule
\multicolumn{8}{@{}l}{\textit{(i) IID benchmarks}} \\
\midrule
\textbf{Method} & \textbf{Hotpot} & \textbf{Trivia} & \textbf{AIME} & \textbf{MedQA} & \textbf{WShop} & \textbf{ALFW} & \textbf{SWE} \\
 & EM & EM & Acc & Acc & SR & SR & Res \\
\midrule
Qwen3.5      & 60.94 & 44.88 & 46.67 & 69.53 & 32.03 & 48.28 & 17.19 \\
GRPO         & 83.59 & 87.50 & 53.33 & 85.16 & 82.03 & 82.81 & 40.63 \\
Tree-GRPO    & 87.50 & 93.75 & 60.00 & 88.28 & 85.94 & 90.92 & 48.44 \\
HCAPO        & 85.94 & 93.75 & 63.33 & 90.62 & 87.50 & 92.19 & 47.24 \\
\textbf{SkillFlow (Ours)} & \textbf{92.19} & \textbf{96.09} & \textbf{70.00} & \textbf{92.19} & \textbf{93.75} & \textbf{96.09} & \textbf{52.34} \\
\midrule
\multicolumn{8}{@{}l}{\textit{(ii) OOD benchmarks}} \\
\midrule
\textbf{Method} & \textbf{MuSQ} & \textbf{NQ} & \textbf{MATH} & \textbf{GPQA} & \textbf{HEval} & \textbf{SciW} & \textbf{M2W} \\
 & Ans EM & Ans EM & Acc & Acc & pass@1 & Succ & Step Acc \\
\midrule
Qwen3.5      & 39.06 & 21.88 & 89.06 & 61.72 & 89.06 & 25.78 & 26.64 \\
GRPO         & 73.44 & 79.69 & 87.50 & 79.69 & 89.06 & 39.10 & 33.59 \\
Tree-GRPO    & 77.34 & 80.47 & 92.19 & 80.47 & 92.19 & 42.20 & 36.00 \\
HCAPO        & 79.69 & 81.25 & 93.75 & 84.38 & 93.75 & 43.75 & 39.10 \\
\textbf{SkillFlow (Ours)} & \textbf{85.16} & \textbf{82.81} & \textbf{96.09} & \textbf{89.84} & \textbf{98.44} & \textbf{57.81} & \textbf{51.56} \\
\bottomrule
\end{tabular}
\end{adjustbox}
\caption{Algorithm comparison on IID and OOD benchmarks (RQ5)}\label{tab:algo-comparison}
\end{subfigure}\hfill
\begin{subfigure}[t]{0.33\textwidth}
\zerosep
\centering
\renewcommand{\arraystretch}{1.10}
\resizebox{\linewidth}{!}{%
\begin{tabular}{@{}l|ccc@{}}
\toprule
\textbf{Method} & \textbf{collect} & \textbf{s/turn} & \textbf{s/traj} \\
\midrule
GRPO         & 216.8 & 1.79 & 35.6 \\
HCAPO        & 137.6 & 1.72 & 48.1 \\
Tree-GRPO    & 162.8 & 1.99 & 45.1 \\
\textbf{Ours} & \textbf{130.5} & \textbf{1.13} & \textbf{33.9} \\
\bottomrule
\end{tabular}%
}\\[0.4em]
\includegraphics[width=\linewidth]{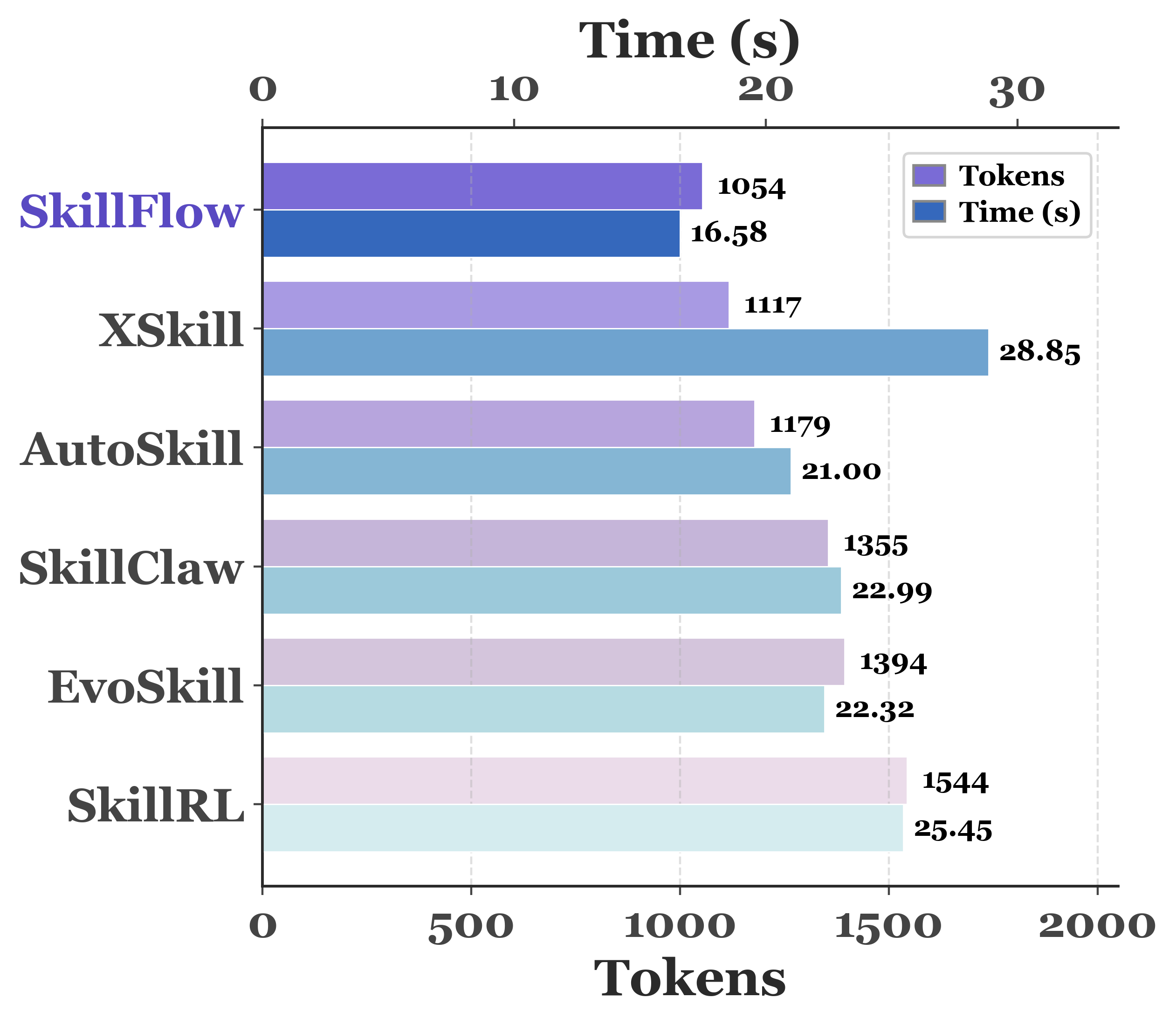}\vspace{1.5pt}
\caption{Training \& per-call cost}\label{fig:mech-cost}
\end{subfigure}\\[0.6em]
\begin{subfigure}[t]{0.33\textwidth}
  \includegraphics[width=\linewidth]{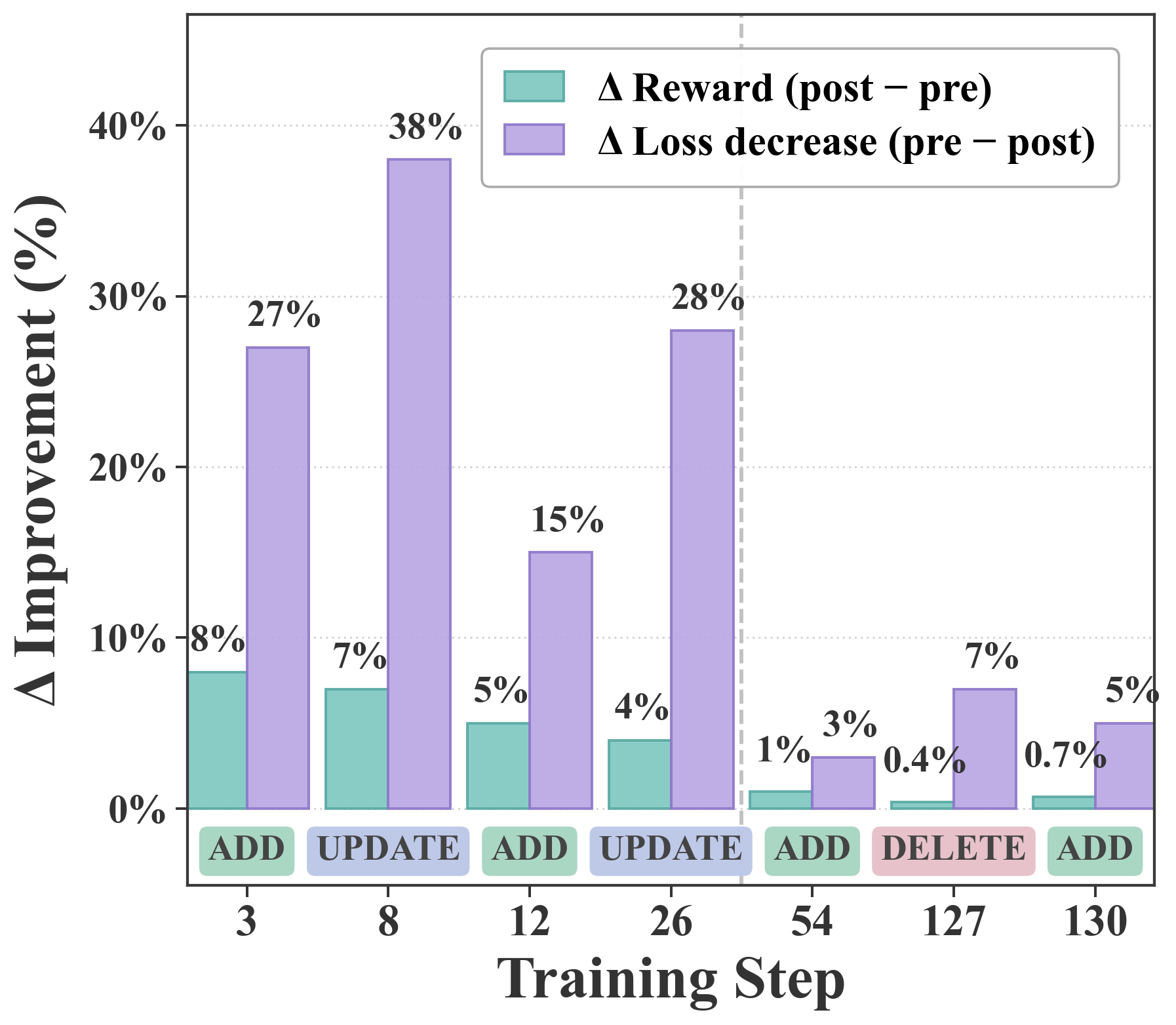}
  \caption{Evolution events}\label{fig:mech-evolution}
\end{subfigure}%
\begin{subfigure}[t]{0.33\textwidth}
  \includegraphics[width=\linewidth]{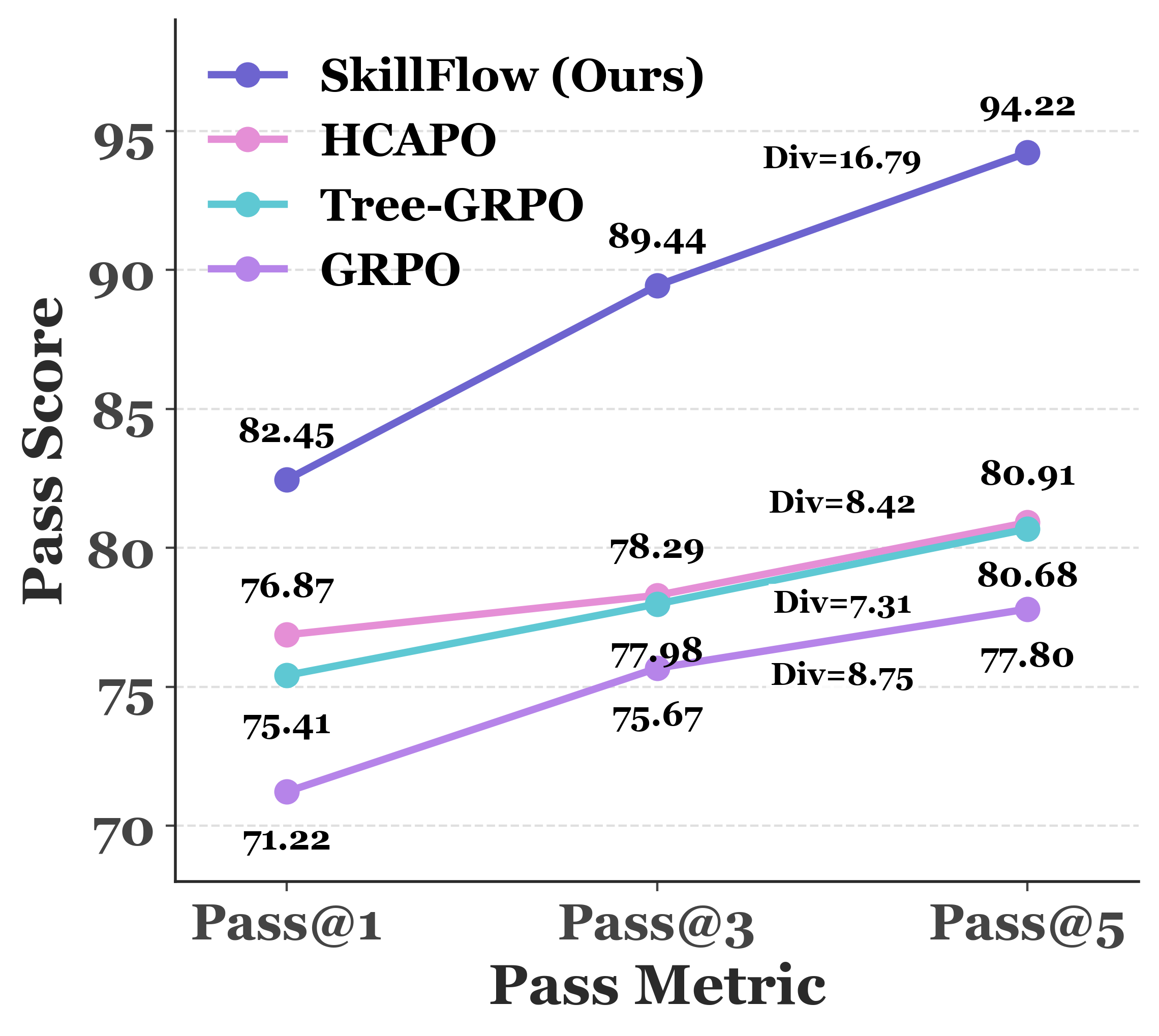}
  \caption{Pass@$K$}\label{fig:mech-passk}
\end{subfigure}%
\begin{subfigure}[t]{0.33\textwidth}
  \includegraphics[width=\linewidth]{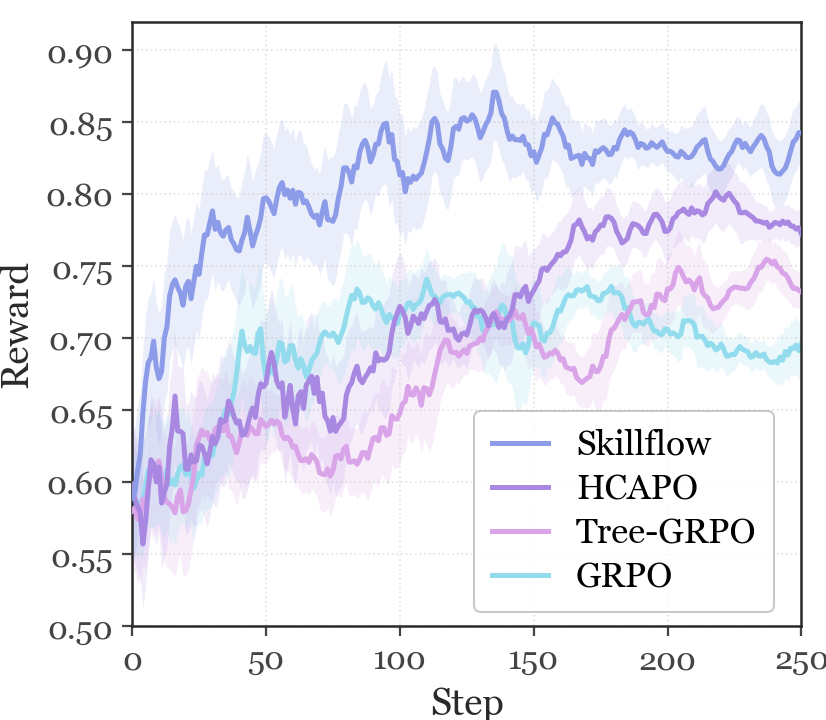}
  \caption{Reward curves}\label{fig:mech-reward}
\end{subfigure}
\caption{Mechanism analysis and algorithm comparison. (a) Algorithm comparison on IID and OOD. (b) Per-call cost vs five skill-evolution baselines. (c) Skill-evolution events. (d) Pass@$K$ vs diversity. (e) Reward curves. SkillFlow leads on accuracy, diversity, and cost simultaneously.}
\label{fig:mechanism}
\afterfloatsep
\end{figure}

\subsection{Backbone Transferability (RQ3)}\label{subsec:transferability}

Figure~\ref{fig:transferability} swaps the Supervisor for six proprietary LLMs. Weaker backbones gain most---explicit credit and diversity-preserving sampling offset shaky base reasoning---while stronger ones still gain on agent-style tasks where orchestration is the bottleneck. The per-category lift hierarchy (Fig.~\ref{fig:transferability}c) is preserved across backbones, locating SkillFlow at the training-recipe level rather than backbone-specific tuning.

\beforesubsectionsep
\subsection{Component Ablation (RQ4)}\label{subsec:ablation}

Table~\ref{tab:ablation} reports a leave-one-out ablation mapping onto claims C1--C3. \emph{Loss (C1):} replacing TTB with GRPO hurts most on diversity-sensitive tasks (AIME, WebShop, ALFWorld) — mode collapse, not reward-curve shape, is the dominant failure mode. \emph{Credit (C2):} removing $P_\phi$ hurts harder multi-step tasks (AIME, WebShop, ScienceWorld, Mind2Web) more than fact-based QA, where per-step credit matters most for identifying decisive decisions. \emph{Evolution (C3):} ablating any of the three flow signals \emph{when}, \emph{where}, or \emph{what} independently degrades performance, indicating none is redundant.

\beforesubsectionsep
\subsection{Algorithm Comparison and Computational Cost (RQ5)}\label{subsec:algo-cost}

Holding the backbone and training data fixed, SkillFlow leads GRPO, Tree-GRPO~\citep{ji2025treegrpo}, and HCAPO~\citep{hcapo2026} on all 14 benchmarks (Figure~\ref{fig:mechanism}(a)); Figure~\ref{fig:mechanism}(d,e) explains the gap: REINFORCE-style methods reinforce a single mode and hit a Pass@5 ceiling at low diversity, while reward-proportional sampling keeps multiple high-reward paths alive (same ranking holds OOD). Skill-evolution events concentrate at TTB plateaus (Fig.~\ref{fig:mechanism}(c)), confirming the plateau-driven trigger of §\ref{sec:evolution}. On cost (Figure~\ref{fig:mechanism}(b)), SkillFlow uses the fewest tokens and lowest time against five skill-evolution baselines~\citep{xskill2026,autoskill2026,skillclaw2026,alzubi2026evoskill,xia2026skillrl}; the $\sim$32\% / $\sim$35\% saving over SkillRL pins the gain on flow-signal-driven evolution.

\enlargethispage{4\baselineskip}
\vspace*{-1.0em}
\section{Conclusion}\label{sec:conclusion}
\vspace*{-0.4em}

SkillFlow unifies orchestration training and recursive skill evolution under TTB: reward-proportional sampling preserves diversity, hindsight backward gives zero-cost per-step credit, and flow diagnostics drive skill curation. Across 14 benchmarks it outperforms direct-LLM, RL, and skill-evolution baselines on accuracy, diversity, and cost.


\bibliographystyle{plainnat}
\bibliography{references}

@inproceedings{joshi2017triviaqa,
  title     = {Triviaqa: A large scale distantly supervised challenge dataset for reading comprehension},
  author    = {Joshi, Mandar and Choi, Eunsol and Weld, Daniel S and Zettlemoyer, Luke},
  booktitle = {Proceedings of the 55th Annual Meeting of the Association for Computational Linguistics (Volume 1: Long Papers)},
  pages     = {1601--1611},
  year      = {2017}
}

@article{jin2021medqa,
  title     = {What disease does this patient have? a large-scale open domain question answering dataset from medical exams},
  author    = {Jin, Di and Pan, Eileen and Oufattole, Nassim and Weng, Wei-Hung and Fang, Hanyi and Szolovits, Peter},
  journal   = {Applied Sciences},
  volume    = {11},
  number    = {14},
  pages     = {6421},
  year      = {2021},
  publisher = {MDPI}
}

@article{yao2022webshop,
  title     = {Webshop: Towards scalable real-world web interaction with grounded language agents},
  author    = {Yao, Shunyu and Chen, Howard and Yang, John and Narasimhan, Karthik},
  journal   = {Advances in Neural Information Processing Systems},
  volume    = {35},
  pages     = {20744--20757},
  year      = {2022}
}

@article{shridhar2021alfworld,
  title     = {Alfworld: Aligning text and embodied environments for interactive learning},
  author    = {Shridhar, Mohit and Yuan, Xingdi and C{\^o}t{\'e}, Marc-Alexandre and Bisk, Yonatan and Trischler, Adam and Hausknecht, Matthew},
  journal   = {arXiv preprint arXiv:2010.03768},
  year      = {2020}
}

@inproceedings{jimenez2024swebench,
  title     = {Swe-bench: Can language models resolve real-world github issues?},
  author    = {Jimenez, Carlos E and Yang, John and Wettig, Alexander and Yao, Shunyu and Pei, Kexin and Press, Ofir and Narasimhan, Karthik R},
  booktitle = {The twelfth international conference on learning representations},
  year      = {2023}
}

@inproceedings{yang2018hotpotqa,
  title     = {HotpotQA: A dataset for diverse, explainable multi-hop question answering},
  author    = {Yang, Zhilin and Qi, Peng and Zhang, Saizheng and Bengio, Yoshua and Cohen, William and Salakhutdinov, Ruslan and Manning, Christopher D},
  booktitle = {Proceedings of the 2018 conference on empirical methods in natural language processing},
  pages     = {2369--2380},
  year      = {2018}
}

@article{trivedi2022musique,
  title     = {♫ MuSiQue: Multihop Questions via Single-hop Question Composition},
  author    = {Trivedi, Harsh and Balasubramanian, Niranjan and Khot, Tushar and Sabharwal, Ashish},
  journal   = {Transactions of the Association for Computational Linguistics},
  volume    = {10},
  pages     = {539--554},
  year      = {2022},
  publisher = {MIT Press One Broadway, 12th Floor, Cambridge, Massachusetts 02142, USA~…}
}

@article{chen2021humaneval,
  title     = {Evaluating large language models trained on code},
  author    = {Chen, Mark and Tworek, Jerry and Jun, Heewoo and Yuan, Qiming and Pinto, Henrique Ponde De Oliveira and Kaplan, Jared and Edwards, Harri and Burda, Yuri and Joseph, Nicholas and Brockman, Greg and others},
  journal   = {arXiv preprint arXiv:2107.03374},
  year      = {2021}
}

@article{schulman2016gae,
  title     = {High-dimensional continuous control using generalized advantage estimation},
  author    = {Schulman, John and Moritz, Philipp and Levine, Sergey and Jordan, Michael and Abbeel, Pieter},
  journal   = {arXiv preprint arXiv:1506.02438},
  year      = {2015}
}

@article{bengio2021gflownet,
  title     = {Flow network based generative models for non-iterative diverse candidate generation},
  author    = {Bengio, Emmanuel and Jain, Moksh and Korablyov, Maksym and Precup, Doina and Bengio, Yoshua},
  journal   = {Advances in neural information processing systems},
  volume    = {34},
  pages     = {27381--27394},
  year      = {2021}
}

@article{malkin2022trajectory,
  title     = {Trajectory balance: Improved credit assignment in gflownets},
  author    = {Malkin, Nikolay and Jain, Moksh and Bengio, Emmanuel and Sun, Chen and Bengio, Yoshua},
  journal   = {Advances in Neural Information Processing Systems},
  volume    = {35},
  pages     = {5955--5967},
  year      = {2022}
}

@article{bengio2023gflownet_foundations,
  title     = {Gflownet foundations},
  author    = {Bengio, Yoshua and Lahlou, Salem and Deleu, Tristan and Hu, Edward J and Tiwari, Mo and Bengio, Emmanuel},
  journal   = {Journal of Machine Learning Research},
  volume    = {24},
  number    = {210},
  pages     = {1--55},
  year      = {2023}
}

@article{yao2023react,
  title     = {React: Synergizing reasoning and acting in language models},
  author    = {Yao, Shunyu and Zhao, Jeffrey and Yu, Dian and Du, Nan and Shafran, Izhak and Narasimhan, Karthik and Cao, Yuan},
  journal   = {arXiv preprint arXiv:2210.03629},
  year      = {2022}
}

@article{zhang2023robust_scheduling,
  title     = {Robust scheduling with gflownets},
  author    = {Zhang, David W and Rainone, Corrado and Peschl, Markus and Bondesan, Roberto},
  journal   = {arXiv preprint arXiv:2302.05446},
  year      = {2023}
}

@article{wang2023voyager,
  title     = {Voyager: An open-ended embodied agent with large language models},
  author    = {Wang, Guanzhi and Xie, Yuqi and Jiang, Yunfan and Mandlekar, Ajay and Xiao, Chaowei and Zhu, Yuke and Fan, Linxi and Anandkumar, Anima},
  journal   = {arXiv preprint arXiv:2305.16291},
  year      = {2023}
}

@article{chen2023fireact,
  title     = {Fireact: Toward language agent fine-tuning},
  author    = {Chen, Baian and Shu, Chang and Shareghi, Ehsan and Collier, Nigel and Narasimhan, Karthik and Yao, Shunyu},
  journal   = {arXiv preprint arXiv:2310.05915},
  year      = {2023}
}

@inproceedings{zeng2023agenttuning,
  title     = {Agenttuning: Enabling generalized agent abilities for llms},
  author    = {Zeng, Aohan and Liu, Mingdao and Lu, Rui and Wang, Bowen and Liu, Xiao and Dong, Yuxiao and Tang, Jie},
  booktitle = {Findings of the Association for Computational Linguistics: ACL 2024},
  pages     = {3053--3077},
  year      = {2024}
}

@inproceedings{wu2023autogen,
  title     = {Autogen: Enabling next-gen LLM applications via multi-agent conversations},
  author    = {Wu, Qingyun and Bansal, Gagan and Zhang, Jieyu and Wu, Yiran and Li, Beibin and Zhu, Erkang and Jiang, Li and Zhang, Xiaoyun and Zhang, Shaokun and Liu, Jiale and others},
  booktitle = {First conference on language modeling},
  year      = {2024}
}

@inproceedings{hong2024metagpt,
  title     = {MetaGPT: Meta programming for a multi-agent collaborative framework},
  author    = {Hong, Sirui and Zhuge, Mingchen and Chen, Jonathan and Zheng, Xiawu and Cheng, Yuheng and Wang, Jinlin and Zhang, Ceyao and Wang, Zili and Yau, Steven Ka Shing and Lin, Zijuan and others},
  booktitle = {The twelfth international conference on learning representations},
  year      = {2023}
}

@article{shao2024deepseekmath,
  title     = {Deepseekmath: Pushing the limits of mathematical reasoning in open language models},
  author    = {Shao, Zhihong and Wang, Peiyi and Zhu, Qihao and Xu, Runxin and Song, Junxiao and Bi, Xiao and Zhang, Haowei and Zhang, Mingchuan and Li, YK and Wu, Yang and others},
  journal   = {arXiv preprint arXiv:2402.03300},
  year      = {2024}
}

@inproceedings{zhuge2024gptswarm,
  title     = {Gptswarm: Language agents as optimizable graphs},
  author    = {Zhuge, Mingchen and Wang, Wenyi and Kirsch, Louis and Faccio, Francesco and Khizbullin, Dmitrii and Schmidhuber, J{\"u}rgen},
  booktitle = {Forty-first International Conference on Machine Learning},
  year      = {2024}
}

@article{zhou2024lats,
  title     = {Language agent tree search unifies reasoning acting and planning in language models},
  author    = {Zhou, Andy and Yan, Kai and Shlapentokh-Rothman, Michal and Wang, Haohan and Wang, Yu-Xiong},
  journal   = {arXiv preprint arXiv:2310.04406},
  year      = {2023}
}

@inproceedings{qin2024toolllm,
  title     = {Toolllm: Facilitating large language models to master 16000+ real-world apis},
  author    = {Qin, Yujia and Liang, Shihao and Ye, Yining and Zhu, Kunlun and Yan, Lan and Lu, Yaxi and Lin, Yankai and Cong, Xin and Tang, Xiangru and Qian, Bill and others},
  booktitle = {The twelfth international conference on learning representations},
  year      = {2023}
}

@inproceedings{wang2024executable,
  title     = {Executable code actions elicit better llm agents},
  author    = {Wang, Xingyao and Chen, Yangyi and Yuan, Lifan and Zhang, Yizhe and Li, Yunzhu and Peng, Hao and Ji, Heng},
  booktitle = {Forty-first International Conference on Machine Learning},
  year      = {2024}
}

@article{qian2024scaling,
  title     = {Scaling large language model-based multi-agent collaboration},
  author    = {Qian, Chen and Xie, Zihao and Wang, Yifei and Liu, Wei and Zhu, Kunlun and Xia, Hanchen and Dang, Yufan and Du, Zhuoyun and Chen, Weize and Yang, Cheng and others},
  journal   = {arXiv preprint arXiv:2406.07155},
  year      = {2024}
}

@article{hu2024adas,
  title     = {Automated design of agentic systems},
  author    = {Hu, Shengran and Lu, Cong and Clune, Jeff},
  journal   = {arXiv preprint arXiv:2408.08435},
  year      = {2024}
}

@article{ong2024routellm,
  title     = {Routellm: Learning to route llms with preference data},
  author    = {Ong, Isaac and Almahairi, Amjad and Wu, Vincent and Chiang, Wei-Lin and Wu, Tianhao and Gonzalez, Joseph E and Kadous, M Waleed and Stoica, Ion},
  journal   = {arXiv preprint arXiv:2406.18665},
  year      = {2024}
}

@article{xu2024flow_of_reasoning,
  title     = {Flow of reasoning: Training llms for divergent reasoning with minimal examples},
  author    = {Yu, Fangxu and Jiang, Lai and Kang, Haoqiang and Hao, Shibo and Qin, Lianhui},
  journal   = {arXiv preprint arXiv:2406.05673},
  year      = {2024}
}

@inproceedings{zhang2025aflow,
  title     = {Aflow: Automating agentic workflow generation},
  author    = {Zhang, Jiayi and Xiang, Jinyu and Yu, Zhaoyang and Teng, Fengwei and Chen, Xiong-Hui and Chen, Jiaqi and Zhuge, Mingchen and Cheng, Xin and Hong, Sirui and Wang, Jinlin and others},
  booktitle = {The Thirteenth International Conference on Learning Representations},
  year      = {2024}
}

@article{openhands2025,
  title     = {Openhands: An open platform for ai software developers as generalist agents},
  author    = {Wang, Xingyao and Li, Boxuan and Song, Yufan and Xu, Frank F and Tang, Xiangru and Zhuge, Mingchen and Pan, Jiayi and Song, Yueqi and Li, Bowen and Singh, Jaskirat and others},
  journal   = {arXiv preprint arXiv:2407.16741},
  year      = {2024}
}

@article{deepseek2025r1,
  title     = {Deepseek-r1: Incentivizing reasoning capability in llms via reinforcement learning},
  author    = {Guo, Daya and Yang, Dejian and Zhang, Haowei and Song, Junxiao and Wang, Peiyi and Zhu, Qihao and Xu, Runxin and Zhang, Ruoyu and Ma, Shirong and Bi, Xiao and others},
  journal   = {arXiv preprint arXiv:2501.12948},
  year      = {2025}
}

@article{yu2025dapo,
  title     = {Dapo: An open-source llm reinforcement learning system at scale},
  author    = {Yu, Qiying and Zhang, Zheng and Zhu, Ruofei and Yuan, Yufeng and Zuo, Xiaochen and Yue, Yu and Dai, Weinan and Fan, Tiantian and Liu, Gaohong and Liu, Lingjun and others},
  journal   = {arXiv preprint arXiv:2503.14476},
  year      = {2025}
}

@article{liu2025vapo,
  title     = {Vapo: Efficient and reliable reinforcement learning for advanced reasoning tasks},
  author    = {Yue, Yu and Yuan, Yufeng and Yu, Qiying and Zuo, Xiaochen and Zhu, Ruofei and Xu, Wenyuan and Chen, Jiaze and Wang, Chengyi and Fan, TianTian and Du, Zhengyin and others},
  journal   = {arXiv preprint arXiv:2504.05118},
  year      = {2025}
}

@article{dang2025evolving_orchestration,
  title     = {Multi-agent collaboration via evolving orchestration},
  author    = {Dang, Yufan and Qian, Chen and Luo, Xueheng and Fan, Jingru and Xie, Zihao and Shi, Ruijie and Chen, Weize and Yang, Cheng and Che, Xiaoyin and Tian, Ye and others},
  journal   = {arXiv preprint arXiv:2505.19591},
  year      = {2025}
}

@inproceedings{magrpo2025,
  title     = {Llm collaboration with multi-agent reinforcement learning},
  author    = {Liu, Shuo and Liang, Zeyu and Lyu, Xueguang and Amato, Christopher},
  booktitle = {Proceedings of the AAAI Conference on Artificial Intelligence},
  volume    = {40},
  number    = {38},
  pages     = {32150--32158},
  year      = {2026}
}

@article{xi2025prs,
  title     = {Enhancing Agentic RL with Progressive Reward Shaping and Value-based Sampling Policy Optimization},
  author    = {Su, Jianghao and Zeng, Xia and Liu, Luhui and Luo, Chao and Chen, Ye and Zhuang, Zhuoran},
  journal   = {arXiv preprint arXiv:2512.07478},
  year      = {2025}
}

@inproceedings{zhu2025prm_lessons,
  title     = {The lessons of developing process reward models in mathematical reasoning},
  author    = {Zhang, Zhenru and Zheng, Chujie and Wu, Yangzhen and Zhang, Beichen and Lin, Runji and Yu, Bowen and Liu, Dayiheng and Zhou, Jingren and Lin, Junyang},
  booktitle = {Findings of the Association for Computational Linguistics: ACL 2025},
  pages     = {10495--10516},
  year      = {2025}
}

@inproceedings{rprm2025,
  title     = {R-prm: Reasoning-driven process reward modeling},
  author    = {She, Shuaijie and Liu, Junxiao and Liu, Yifeng and Chen, Jiajun and Huang, Xin and Huang, Shujian},
  booktitle = {Proceedings of the 2025 Conference on Empirical Methods in Natural Language Processing},
  pages     = {13449--13462},
  year      = {2025}
}

@article{self_evolving_survey2025,
  title     = {A comprehensive survey of self-evolving ai agents: A new paradigm bridging foundation models and lifelong agentic systems},
  author    = {Fang, Jinyuan and Peng, Yanwen and Zhang, Xi and Wang, Yingxu and Yi, Xinhao and Zhang, Guibin and Xu, Yi and Wu, Bin and Liu, Siwei and Li, Zihao and others},
  journal   = {arXiv preprint arXiv:2508.07407},
  year      = {2025}
}

@inproceedings{daao2026,
  title     = {Difficulty-aware agentic orchestration for query-specific multi-agent workflows},
  author    = {Su, Jinwei and Lan, Qizhen and Xia, Yinghui and Sun, Lifan and Tian, Weiyou and Shi, Tianyu and He, Lewei},
  booktitle = {Proceedings of the ACM Web Conference 2026},
  pages     = {2060--2070},
  year      = {2026}
}

@article{rl_long_horizon_agents2026,
  title     = {Reinforcement learning for long-horizon interactive llm agents},
  author    = {Chen, Kevin and Cusumano-Towner, Marco and Huval, Brody and Petrenko, Aleksei and Hamburger, Jackson and Koltun, Vladlen and Kr{\"a}henb{\"u}hl, Philipp},
  journal   = {arXiv preprint arXiv:2502.01600},
  year      = {2025}
}

@article{shen2026token_level,
  title     = {Agentic reinforcement learning with implicit step rewards},
  author    = {Liu, Xiaoqian and Wang, Ke and Wu, Yuchuan and Huang, Fei and Li, Yongbin and Zhang, Junge and Jiao, Jianbin},
  journal   = {arXiv preprint arXiv:2509.19199},
  year      = {2025}
}

@article{sage2026skill_rl,
  title     = {Reinforcement learning for self-improving agent with skill library},
  author    = {Wang, Jiongxiao and Yan, Qiaojing and Wang, Yawei and Tian, Yijun and Mishra, Soumya Smruti and Xu, Zhichao and Gandhi, Megha and Xu, Panpan and Cheong, Lin Lee},
  journal   = {arXiv preprint arXiv:2512.17102},
  year      = {2025}
}

@article{wang2025ragen,
  title     = {Ragen: Understanding self-evolution in llm agents via multi-turn reinforcement learning},
  author    = {Wang, Zihan and Wang, Kangrui and Wang, Qineng and Zhang, Pingyue and Li, Linjie and Yang, Zhengyuan and Jin, Xing and Yu, Kefan and Nguyen, Minh Nhat and Liu, Licheng and others},
  journal   = {arXiv preprint arXiv:2504.20073},
  year      = {2025}
}

@article{jin2025searchr1,
  title     = {Search-r1: Training llms to reason and leverage search engines with reinforcement learning},
  author    = {Jin, Bowen and Zeng, Hansi and Yue, Zhenrui and Yoon, Jinsung and Arik, Sercan and Wang, Dong and Zamani, Hamed and Han, Jiawei},
  journal   = {arXiv preprint arXiv:2503.09516},
  year      = {2025}
}

@article{li2025choice_divergence,
  title     = {The choice of divergence: A neglected key to mitigating diversity collapse in reinforcement learning with verifiable reward},
  author    = {Li, Long and Zhou, Zhijian and Hao, Jiaran and Liu, Jason Klein and Miao, Yanting and Pang, Wei and Tan, Xiaoyu and Chu, Wei and Wang, Zhe and Pan, Shirui and others},
  journal   = {arXiv preprint arXiv:2509.07430},
  year      = {2025}
}

@article{ji2025treegrpo,
  title     = {Tree search for llm agent reinforcement learning},
  author    = {Ji, Yuxiang and Ma, Ziyu and Wang, Yong and Chen, Guanhua and Chu, Xiangxiang and Wu, Liaoni},
  journal   = {arXiv preprint arXiv:2509.21240},
  year      = {2025}
}

@article{zhang2026flowsteer,
  title     = {FlowSteer: Interactive Agentic Workflow Orchestration via End-to-End Reinforcement Learning},
  author    = {Zhang, Mingda and Luo, Haoran and Shen, Tiesunlong and Lin, Qika and Tang, Xiaoying and Mao, Rui and Cambria, Erik},
  journal   = {arXiv preprint arXiv:2602.01664},
  year      = {2026}
}

@article{ace_gflownet2026,
  title     = {Avoid What You Know: Divergent Trajectory Balance for GFlowNets},
  author    = {Dall'Antonia, Pedro and da Silva, Tiago and Csillag, Daniel and Lahlou, Salem and Mesquita, Diego},
  journal   = {arXiv preprint arXiv:2602.17827},
  year      = {2026}
}

@article{hcapo2026,
  title     = {Hindsight Credit Assignment for Long-Horizon LLM Agents},
  author    = {Tan, Hui-Ze and Yang, Xiao-Wen and Chen, Hao and Shao, Jie-Jing and Wen, Yi and Shen, Yuteng and Luo, Weihong and Du, Xiku and Guo, Lan-Zhe and Li, Yu-Feng},
  journal   = {arXiv preprint arXiv:2603.08754},
  year      = {2026}
}

@article{hiper2026,
  title     = {Hiper: Hierarchical reinforcement learning with explicit credit assignment for large language model agents},
  author    = {Peng, Jiangweizhi and Liu, Yuanxin and Zhou, Ruida and Fleming, Charles and Wang, Zhaoran and Garcia, Alfredo and Hong, Mingyi},
  journal   = {arXiv preprint arXiv:2602.16165},
  year      = {2026}
}

@article{agent_skills_survey2026,
  title     = {Agent skills for large language models: Architecture, acquisition, security, and the path forward},
  author    = {Xu, Renjun and Yan, Yang},
  journal   = {arXiv preprint arXiv:2602.12430},
  year      = {2026}
}

@article{single_agent_skills2026,
  title     = {When single-agent with skills replace multi-agent systems and when they fail},
  author    = {Li, Xiaoxiao},
  journal   = {arXiv preprint arXiv:2601.04748},
  year      = {2026}
}

@article{process_outcome_credit2026,
  title     = {Discovering Process-Outcome Credit in Multi-Step LLM Reasoning},
  author    = {Wang, Xiangwei and Wang, Wei and Chen, Ken and Nimalsiri, Nanduni and Halgamuge, Saman},
  journal   = {arXiv preprint arXiv:2602.01034},
  year      = {2026}
}

@article{graph_grpo2026,
  title     = {Graph-GRPO: Stabilizing Multi-Agent Topology Learning via Group Relative Policy Optimization},
  author    = {Cang, Yueyang and Zhang, Xiaoteng and Zhao, Erlu and Ji, Zehua and Liu, Yuhang and He, Yuchen and Ning, Zhiyuan and Yijun, Chen and Que, Wenge and Shi, Li},
  journal   = {arXiv preprint arXiv:2603.02701},
  year      = {2026}
}

@article{int_credit2026,
  title     = {InT: Self-Proposed Interventions Enable Credit Assignment in LLM Reasoning},
  author    = {Yang, Matthew YR and Bai, Hao and Wu, Ian and Yang, Gene and Setlur, Amrith and Kumar, Aviral},
  journal   = {arXiv preprint arXiv:2601.14209},
  year      = {2026}
}

@article{xia2026skillrl,
  title     = {Skillrl: Evolving agents via recursive skill-augmented reinforcement learning},
  author    = {Xia, Peng and Chen, Jianwen and Wang, Hanyang and Liu, Jiaqi and Zeng, Kaide and Wang, Yu and Han, Siwei and Zhou, Yiyang and Zhao, Xujiang and Chen, Haifeng and others},
  journal   = {arXiv preprint arXiv:2602.08234},
  year      = {2026}
}

@article{alzubi2026evoskill,
  title     = {Evoskill: Automated skill discovery for multi-agent systems},
  author    = {Alzubi, Salaheddin and Provenzano, Noah and Bingham, Jaydon and Chen, Weiyuan and Vu, Tu},
  journal   = {arXiv preprint arXiv:2603.02766},
  year      = {2026}
}

@article{coevoskills2026,
  title     = {EvoSkills: Self-Evolving Agent Skills via Co-Evolutionary Verification},
  author    = {Zhang, Hanrong and Fan, Shicheng and Zou, Henry Peng and Chen, Yankai and Wang, Zhenting and Zhou, Jiayu and Li, Chengze and Huang, Wei-Chieh and Yao, Yifei and Zheng, Kening and others},
  journal   = {arXiv preprint arXiv:2604.01687},
  year      = {2026}
}

@article{skillx2026,
  title     = {SkillX: Automatically constructing skill knowledge bases for agents},
  author    = {Wang, Chenxi and Yu, Zhuoyun and Xie, Xin and Yao, Wuguannan and Fang, Runnan and Qiao, Shuofei and Cao, Kexin and Zheng, Guozhou and Qi, Xiang and Zhang, Peng and others},
  journal   = {arXiv preprint arXiv:2604.04804},
  year      = {2026}
}

@article{xskill2026,
  title     = {Xskill: Continual learning from experience and skills in multimodal agents},
  author    = {Jiang, Guanyu and Su, Zhaochen and Qu, Xiaoye and Fung, Yi R},
  journal   = {arXiv preprint arXiv:2603.12056},
  year      = {2026}
}

@article{autoskill2026,
  title     = {Autoskill: Experience-driven lifelong learning via skill self-evolution},
  author    = {Yang, Yutao and Li, Junsong and Pan, Qianjun and Zhan, Bihao and Cai, Yuxuan and Du, Lin and Zhou, Jie and Chen, Kai and Chen, Qin and Li, Xin and others},
  journal   = {arXiv preprint arXiv:2603.01145},
  year      = {2026}
}

@article{skillclaw2026,
  title     = {SkillClaw: Let Skills Evolve Collectively with Agentic Evolver},
  author    = {Ma, Ziyu and Yang, Shidong and Ji, Yuxiang and Wang, Xucong and Wang, Yong and Hu, Yiming and Huang, Tongwen and Chu, Xiangxiang},
  journal   = {arXiv preprint arXiv:2604.08377},
  year      = {2026}
}

@article{skillsd2026,
  title     = {Skill-SD: Skill-Conditioned Self-Distillation for Multi-turn LLM Agents},
  author    = {Wang, Hao and Wang, Guozhi and Xiao, Han and Zhou, Yufeng and Pan, Yue and Wang, Jichao and Xu, Ke and Wen, Yafei and Ruan, Xiaohu and Chen, Xiaoxin and others},
  journal   = {arXiv preprint arXiv:2604.10674},
  year      = {2026}
}

@article{sok_agentic_skills2026,
  title     = {SoK: Agentic Skills--Beyond Tool Use in LLM Agents},
  author    = {Jiang, Yanna and Li, Delong and Deng, Haiyu and Ma, Baihe and Wang, Xu and Wang, Qin and Yu, Guangsheng},
  journal   = {arXiv preprint arXiv:2602.20867},
  year      = {2026}
}

@article{wang2026workflowr1,
  title     = {Workflow-R1: Group Sub-sequence Policy Optimization for Multi-turn Workflow Construction},
  author    = {Kong, Mingze and Qu, Zikun and Zhou, Zhongquan and Liang, Pengyu and Li, Xiang and Shang, Zhiwei and Hong, Zhi and Huang, Kaiyu and Wang, Zhiyong and Dai, Zhongxiang},
  journal   = {arXiv preprint arXiv:2602.01202},
  year      = {2026}
}

@article{pan2026agentflow,
  title     = {In-the-flow agentic system optimization for effective planning and tool use},
  author    = {Li, Zhuofeng and Zhang, Haoxiang and Han, Seungju and Liu, Sheng and Xie, Jianwen and Zhang, Yu and Choi, Yejin and Zou, James and Lu, Pan},
  journal   = {arXiv preprint arXiv:2510.05592},
  year      = {2025}
}


\newpage
\appendix
\section*{Appendix: Theoretical Foundations of SkillFlow}
\addcontentsline{toc}{section}{Appendix}

This appendix provides a complete theoretical treatment of SkillFlow's flow-based learning framework, following the progressive development style of GFlowNet Foundations (Bengio et al., JMLR 2023).
We establish the mathematical foundations for flow networks on DAGs, derive the Tempered Trajectory Balance (TTB) loss, prove convergence properties, and detail all supporting theoretical results.
Each section builds explicitly on previous results, providing complete proofs with intermediate steps.

\section{Flow-Theoretic Foundations}\label{app:gflownet_foundations}

This section introduces the fundamental concepts of flow networks on directed acyclic graphs (DAGs) and establishes the connection to reward-proportional sampling.

\subsection{Flow Networks on DAGs}\label{app:flow_networks}

We begin with the basic definitions of flow networks.

\begin{definition}[Flow Network]\label{def:flow_network}
A flow network is a tuple $(\mathcal{G}, F)$ where $\mathcal{G} = (\mathcal{V}, \mathcal{A})$ is a DAG with a unique source node $s_0 \in \mathcal{V}$ having no incoming edges and a set of terminal nodes $\mathcal{X} \subset \mathcal{V}$ having no outgoing edges, and $F: \mathcal{A} \to \R_{\geq 0}$ is a non-negative edge flow function assigning a non-negative real number to each directed edge.
\end{definition}

For each state $s \in \mathcal{V}$, we denote by $\mathrm{Pa}(s)$ the set of predecessors (parents) and $\mathrm{Ch}(s)$ the set of successors (children).
We now introduce the crucial concept of state flow.

\begin{definition}[State Flow and Flow Conservation]\label{def:state_flow}
For any non-terminal state $s \in \mathcal{V} \setminus \mathcal{X}$, the \textbf{state flow} $F(s) \in \R_{\geq 0}$ is defined implicitly by the flow conservation law:
\begin{equation}\label{eq:app_state_flow_def}
  F(s)
  \;\coloneqq\; \sum_{s' \in \mathrm{Ch}(s)} F(s \to s')
  \;=\; \sum_{s' \in \mathrm{Pa}(s)} F(s' \to s).
\end{equation}
That is, the total flow entering a non-terminal state equals the total flow leaving it.
This is the fundamental conservation law of flow networks, analogous to Kirchhoff's current law in electrical networks.
\end{definition}

\begin{remark}
Flow conservation is the defining property that relates edge flows to state flows.
By Definition~\ref{def:state_flow}, we may consistently compute $F(s)$ by summing either incoming or outgoing edge flows.
\end{remark}

The partition function represents the total flow circulating through the network.

\begin{definition}[Partition Function and Total Flow]\label{def:partition}
The partition function $Z$ is the total flow in the network, equal to the source flow:
\begin{equation}\label{eq:app_partition}
  Z \;\coloneqq\; F(s_0)
  \;=\; \sum_{s' \in \mathrm{Ch}(s_0)} F(s_0 \to s')
  \;=\; \sum_{x \in \mathcal{X}} F(x).
\end{equation}
The last equality (terminal flow equals source flow) follows from recursive application of flow conservation at each state: every unit of flow leaving the source must traverse some path and arrive at some terminal.
\end{definition}

We now introduce policies derived from flow normalization.

\begin{definition}[Forward Policy and Backward Policy]\label{def:policies}
Given a flow $F$, we define:

\emph{Forward policy}: $P_F(s' \mid s) = \frac{F(s \to s')}{F(s)}$ for each edge $s \to s'$ where $s$ is non-terminal.

\emph{Backward policy}: $P_B(s \mid s') = \frac{F(s \to s')}{F(s')}$ for each edge $s \to s'$ where $s' \in \mathcal{V} \setminus \{\text{source}\}$.

By flow conservation (Definition~\ref{def:state_flow}), both $P_F$ and $P_B$ are valid probability distributions: $\sum_{s'} P_F(s' \mid s) = 1$ and $\sum_s P_B(s \mid s') = 1$.
\end{definition}

We now define trajectory flow, which is essential for understanding how the forward policy generates trajectories.

\begin{definition}[Trajectory and Trajectory Flow]\label{def:traj_flow}
A complete trajectory is a path $\tau = (s_0 \to s_1 \to \cdots \to s_T = x)$ from the source $s_0$ to some terminal $x \in \mathcal{X}$.
The trajectory flow is:
\begin{equation}\label{eq:app_traj_flow}
  F(\tau) \;\coloneqq\; Z \cdot \prod_{t=1}^{T} P_F(s_t \mid s_{t-1}).
\end{equation}
\end{definition}

\begin{lemma}[Trajectory Flow Identity]\label{lem:traj_telescoping}
On a general DAG, the trajectory flow $F(\tau) = Z \cdot \prod_{t=1}^{T} P_F(s_t \mid s_{t-1})$ admits no further factorization in terms of edge or state flows alone; in particular, $F(\tau) \neq \prod_{t=1}^T F(s_{t-1}\to s_t)$ in general because edge flows and state flows are distinct quantities. When the DAG additionally satisfies that each non-source state has a unique parent (as in SkillFlow's tree-structured $\mathcal{G}$), $F(s_{t-1}\to s_t) = F(s_t)$ holds for every edge (all flow into $s_t$ comes from the single edge $s_{t-1}\to s_t$), and the trajectory flow telescopes to the terminal state flow:
\begin{equation}\label{eq:traj_terminal}
  F(\tau) \;=\; Z \cdot \prod_{t=1}^{T} \frac{F(s_t)}{F(s_{t-1})}
  \;=\; \frac{F(s_T)}{F(s_0)} \cdot Z
  \;=\; F(s_T),
\end{equation}
using the standard ratio-product telescope on \emph{state} flows together with $F(s_0) = Z$ (Definition~\ref{def:partition}).
\end{lemma}

We now establish the key decomposition of state flow.

\begin{proposition}[Flow Decomposition]\label{prop:app_flow_decomposition}
For any state $s \in \mathcal{V}$, the state flow equals the sum of flows of all trajectories passing through $s$:
\begin{equation}
  F(s) = \sum_{\tau \ni s} F(\tau),
\end{equation}
where $\{\tau \ni s\}$ denotes all complete trajectories passing through state $s$.
In particular, for terminal states:
\begin{equation}
  F(x) = \sum_{\tau: s_T = x} F(\tau),
\end{equation}
i.e., the flow into a terminal equals the sum of flows of trajectories terminating at that state.
\end{proposition}

\begin{proof}
We prove by structural induction on the depth of states in the DAG, where depth is the longest path length from the source.

\emph{Base case} ($\text{depth}(s) = 0$): $s = s_0$ is the source.
By Definition~\ref{def:partition}, $F(s_0) = Z = \sum_\tau F(\tau)$, since every trajectory starts at $s_0$ and summing over all complete trajectories gives the total flow.

\emph{Inductive step}: Assume the claim holds for all states at depth $< d$.
Consider a state $s$ at depth $d$.
By Definition~\ref{def:state_flow} (flow conservation):
\begin{equation}
  F(s) = \sum_{s' \in \mathrm{Pa}(s)} F(s' \to s).
\end{equation}
By Definition~\ref{def:policies}, $F(s' \to s) = F(s') \cdot P_F(s \mid s')$.
By the inductive hypothesis, $F(s') = \sum_{\tau' \ni s'} F(\tau')$, so:
\begin{align}
  F(s) &= \sum_{s' \in \mathrm{Pa}(s)} F(s') \cdot P_F(s \mid s')\\
  &= \sum_{s' \in \mathrm{Pa}(s)} \left(\sum_{\tau' \ni s'} F(\tau')\right) \cdot P_F(s \mid s').
\end{align}
Each trajectory $\tau'$ through $s'$ can be extended to a unique trajectory through $s$ by following the edge $s' \to s$.
Since every complete trajectory through $s$ passes through exactly one parent $s' \in \mathrm{Pa}(s)$, the above sum equals:
\begin{equation}
  F(s) = \sum_{\tau \ni s} F(\tau).
\end{equation}
\qedhere
\end{proof}

\subsection{Reward-Matching and GFlowNet Sampling}\label{app:reward_matching}

We now establish when a flow induces sampling proportional to a reward function.

\begin{definition}[Reward-Matching Flow]\label{def:reward_matching}
Given a strictly positive reward function $R: \mathcal{X} \to \R_{> 0}$ on terminal states, a flow $F$ is \textbf{reward-matching} if:
\begin{equation}
  F(x) = R(x) \quad \text{for all } x \in \mathcal{X}.
\end{equation}
That is, the flow into each terminal state equals the reward of that terminal.
\end{definition}

The central theorem of GFlowNets states that reward-matching flows induce a distribution over terminals proportional to rewards.

\begin{theorem}[GFlowNet Sampling Property]\label{thm:gflownet_sampling}
If $F$ is reward-matching with respect to $R: \mathcal{X} \to \R_{> 0}$, then sampling a complete trajectory by following the forward policy $P_F$ produces a distribution over terminals given by:
\begin{equation}\label{eq:app_sampling}
  P_F(x) \;\coloneqq\; \Pr_{\tau \sim P_F}[s_T = x] = \frac{R(x)}{Z},
\end{equation}
where $Z = \sum_{x' \in \mathcal{X}} R(x')$ is the partition function (total reward).
\end{theorem}

\begin{proof}
Let $\tau_x = \{x\}$ denote any complete trajectory ending at $x$.
The probability of sampling a trajectory ending at $x$ is:
\begin{equation}
  P_F(x) = \sum_{\tau: s_T = x} \prod_{t=1}^{|\tau|} P_F(s_t \mid s_{t-1}).
\end{equation}
Substituting Equation~\eqref{eq:app_traj_flow}:
\begin{equation}
  P_F(x) = \sum_{\tau: s_T = x} \frac{F(\tau)}{Z}.
\end{equation}
By Proposition~\ref{prop:app_flow_decomposition}, $F(x) = \sum_{\tau: s_T = x} F(\tau)$, so:
\begin{equation}
  P_F(x) = \frac{F(x)}{Z}.
\end{equation}
Since $F$ is reward-matching, $F(x) = R(x)$, thus:
\begin{equation}
  P_F(x) = \frac{R(x)}{Z} = \frac{R(x)}{\sum_{x'} R(x')}.
\end{equation}
\qedhere
\end{proof}

\begin{remark}[Positive Reward Requirement]\label{rem:positive_reward}
The requirement $R(x) > 0$ for all $x$ is essential to Theorem~\ref{thm:gflownet_sampling}.
If $R(x) = 0$ for some terminal, the reward-matching condition $F(x) = 0$ implies no flow reaches $x$ under any forward policy that respects Definition~\ref{def:policies} (since $P_F(s' \mid s)$ must be derived from non-negative edge flows).
In SkillFlow, trajectories with zero or negative reward must be handled via smoothing (Appendix~\ref{app:epsilon_smoothing}).
\end{remark}

\subsection{Application to Task Orchestration}\label{app:task_orch}

We now instantiate the abstract flow network framework to SkillFlow's task orchestration setting.

\begin{proposition}[Flow Network Instantiation for Task Orchestration]\label{prop:instantiation}
The orchestration trajectory distribution is a flow network $(\mathcal{G}, F)$ where:
\begin{itemize}[leftmargin=2em, itemsep=2pt]
  \item \textbf{Vertices $\mathcal{V}$}: Each vertex is an interaction history state $H_t$ (Definition 1, main text).
  \item \textbf{Edges $\mathcal{A}$}: Each directed edge $(H_{t-1}, H_t)$ corresponds to an orchestration action $a_t = (\alpha_t, o_t)$, where the action type $\alpha_t \in \{\text{skill}, \text{act}, \text{accept}\}$ is selected by the supervisor.
  \item \textbf{Source $s_0$}: The initial state $H_0 = q \oplus \Slib_{\mathrm{ret}} \oplus \omega_q$, formed by concatenating task $q$, retrieved skills, and orchestration guideline.
  \item \textbf{Terminals $\mathcal{X}$}: States where $\alpha_t = \text{accept}$ or $t = T_{\max}$ (maximum trajectory length).
  \item \textbf{Forward policy}: $P_F(H_t \mid H_{t-1}) = \pi_\theta(a_t \mid H_{t-1})$, the learned supervisor policy.
  \item \textbf{Edge flow}: $F(H_{t-1} \to H_t) = Z \cdot P_F(H_t \mid H_{t-1})$ for edges in sampled trajectories.
  \item \textbf{Terminal flow}: $F(x) = \Rtilde(\tau_x)^\beta$ for terminal $x$ reached by trajectory $\tau_x$, where $\Rtilde(\tau) = R(\tau) + \varepsilon_{\min}$ is the smoothed reward and $\beta > 0$ is the temperature.
\end{itemize}
\end{proposition}

The key property is acyclicity, established by the strictly increasing state representation.

\begin{lemma}[DAG Acyclicity by State Growth]\label{lem:dag_acyclic}
Under the definition $H_t = H_{t-1} \oplus (r_t, a_t, o_t^{\text{exec}})$, each edge strictly increases the state size: $|H_t| > |H_{t-1}|$.
Therefore, the orchestration graph is acyclic: no path can visit the same state twice, as that would require returning to a state with smaller or equal size.
\end{lemma}

This acyclicity is essential for applying flow matching algorithms, which require the DAG structure.

\section{Trajectory Balance and TTB Derivation}\label{app:ttb_derivation}

This section derives the Trajectory Balance (TB) condition and the Tempered Trajectory Balance (TTB) loss, the core objective for SkillFlow training.

\paragraph{Convention (DAG edge $=$ action; node $=$ state $H_t$).}
Throughout this appendix, all flow-theoretic quantities are defined at the \emph{action level} of the orchestration DAG: each node $s_t$ corresponds to the interaction history $H_t$ (Definition~1, main text), and each directed edge $(s_{t-1}, s_t)$ corresponds to one orchestration action $a_t = (\alpha_t, o_t)$ that itself comprises $K_t$ tokens produced autoregressively by the LLM. The \emph{per-token normalization} introduced in Definition~\ref{def:ttb_residual} is a within-edge device that averages over the $K_t$ tokens of one action to produce a single, length-robust edge log-probability; it does \emph{not} switch the DAG to a token-level granularity.

\subsection{Backward Policy Definition}\label{app:backward_policy}

The backward policy is the reverse-direction distribution derived from reward-matching flows.

\begin{definition}[Backward Policy]\label{def:backward_policy}
Given a reward-matching flow $F$, the backward policy is:
\begin{equation}
  P_B(s \mid s') = \frac{F(s \to s')}{F(s')},
  \quad\text{for } s \in \mathrm{Pa}(s').
\end{equation}
By flow conservation (Definition~\ref{def:state_flow}), $\sum_{s \in \mathrm{Pa}(s')} P_B(s \mid s') = \frac{1}{F(s')} \sum_s F(s \to s') = \frac{F(s')}{F(s')} = 1$, so $P_B(\cdot \mid s')$ is a valid probability distribution over parents.
\end{definition}

The backward policy encodes the reverse probability of transitioning from $s'$ back to each parent $s$.
For learning, we will use a learned backward policy $P_\phi$ parameterized by $\phi$.

\begin{remark}[SkillFlow's Hindsight-Conditioned Backward]\label{rem:hindsight_backward_def}
In SkillFlow, the learned backward policy $P_\phi$ conditions on the \emph{hindsight state} $\tilde{s}_t \coloneqq H_{t-1} \oplus o_t^{\text{exec}}$, which augments the pre-action history with the post-action execution observation:
\begin{equation}\label{eq:hindsight_backward_def}
  P_\phi(a_t \mid \tilde{s}_t) \;=\; P_\phi(a_t \mid H_{t-1} \oplus o_t^{\text{exec}}).
\end{equation}
This hindsight conditioning is well-defined as a backward policy on the action-DAG: by Eq.~\eqref{eq:state_update}, the post-action state $s_t = H_t$ deterministically encodes $(a_t, o_t^{\text{exec}})$ given $s_{t-1} = H_{t-1}$, so conditioning on $s_t$ is equivalent to conditioning on $(\tilde{s}_t, a_t)$. The information asymmetry $\tilde{s}_t \supset s_{t-1}$ is what makes the step-importance ratio $I(t) = \pi_\theta(a_t \mid r_t, H_{t-1}) / P_\phi(a_t \mid \tilde{s}_t)$ a non-trivial credit signal (Lemma~\ref{lem:hindsight_backward_eq} in B.4).
\end{remark}

\subsection{Trajectory Balance Condition and Full Derivation}\label{app:tb_condition}

The Trajectory Balance (TB) condition is the equivalence characterizing when a flow is reward-matching.

\begin{theorem}[Trajectory Balance (Malkin et al., 2022)]\label{thm:tb}
A flow $F$ on a DAG is reward-matching iff for every complete trajectory $\tau = (s_0, s_1, \ldots, s_T)$ with $s_T \in \mathcal{X}$:
\begin{equation}\label{eq:app_tb}
  Z \cdot \prod_{t=1}^{T} P_F(s_t \mid s_{t-1})
  = F(s_T) \cdot \prod_{t=1}^{T} P_B(s_{t-1} \mid s_t).
\end{equation}
Equivalently, in log-space:
\begin{equation}\label{eq:app_tb_log}
  \log Z + \sum_{t=1}^{T} \log P_F(s_t \mid s_{t-1})
  = \log F(s_T) + \sum_{t=1}^{T} \log P_B(s_{t-1} \mid s_t).
\end{equation}
\end{theorem}

\begin{proof}
We prove both directions explicitly.

\textbf{Forward direction ($F$ reward-matching $\Rightarrow$ TB holds):}

Assume $F$ is reward-matching, so $F(x) = R(x)$ for all terminals $x$.
Expand the LHS and RHS of Equation~\eqref{eq:app_tb}:
\begin{align}
  \text{LHS} &= Z \cdot \prod_{t=1}^{T} \frac{F(s_{t-1} \to s_t)}{F(s_{t-1})},\\
  \text{RHS} &= R(s_T) \cdot \prod_{t=1}^{T} \frac{F(s_{t-1} \to s_t)}{F(s_t)}.
\end{align}
Taking the ratio $\text{LHS} / \text{RHS}$:
\begin{equation}
  \frac{\text{LHS}}{\text{RHS}} = \frac{Z}{R(s_T)} \cdot \prod_{t=1}^{T} \frac{F(s_t)}{F(s_{t-1})}.
\end{equation}
The product telescopes:
\begin{equation}
  \prod_{t=1}^{T} \frac{F(s_t)}{F(s_{t-1})} = \frac{F(s_1)}{F(s_0)} \cdot \frac{F(s_2)}{F(s_1)} \cdots \frac{F(s_T)}{F(s_{T-1})} = \frac{F(s_T)}{F(s_0)} = \frac{R(s_T)}{Z},
\end{equation}
where we used $F(s_0) = Z$ (Definition~\ref{def:partition}) and $F(s_T) = R(s_T)$ (reward-matching).
Thus $\text{LHS} / \text{RHS} = \frac{Z}{R(s_T)} \cdot \frac{R(s_T)}{Z} = 1$, so LHS $=$ RHS.

\textbf{Reverse direction (TB holds $\Rightarrow$ $F$ reward-matching):}

Assume TB holds for every trajectory.
Sum Equation~\eqref{eq:app_tb} over all trajectories ending at a fixed terminal $x$:
\begin{equation}
  Z \cdot \sum_{\tau: s_T = x} \prod_{t=1}^{|\tau|} P_F(s_t \mid s_{t-1})
  = \sum_{\tau: s_T = x} F(x) \cdot \prod_{t=1}^{|\tau|} P_B(s_{t-1} \mid s_t).
\end{equation}
The LHS is $Z \cdot P_F(x)$ by definition of the forward policy probability.
For the RHS, by the product structure and properties of the backward policy (which forms a distribution over reverse trajectories), the sum $\sum_{\tau: s_T = x} \prod_t P_B(s_{t-1} \mid s_t) = 1$ (the reverse-direction probability of reaching $x$ from all trajectories sums to 1).
Thus:
\begin{equation}
  Z \cdot P_F(x) = F(x).
\end{equation}
By Theorem~\ref{thm:gflownet_sampling}, this implies $P_F(x) = F(x) / Z$, which combined with $Z = \sum_{x'} F(x')$ gives $F(x) = R(x)$. \qedhere
\end{proof}

\subsection{TTB Loss Derivation with Temperature and Length Normalization}\label{app:ttb_loss}

We now derive the Tempered Trajectory Balance loss as the squared, length-normalized log-space TB residual on the action-DAG, with two technical devices: (i) reasoning $r_t$ enters the forward conditional but is treated as fixed context (Lemma~\ref{lem:conditional_tb_reasoning}, B.4); (ii) each action edge's log-probability is averaged across its $K_t$ tokens (per-token tempering, Lemma~\ref{lem:per_token_tempering}, B.4). Both devices are formally justified in B.4 and do \emph{not} alter the action-level DAG structure.

\begin{definition}[Per-Token-Tempered Edge Log-Probabilities]\label{def:per_token_edge}
Let action $a_t$ comprise $K_t$ tokens $\mathrm{tok}^{(t)}_1, \ldots, \mathrm{tok}^{(t)}_{K_t}$ produced autoregressively. The per-token-tempered forward and backward edge log-probabilities are:
\begin{align}
  \widetilde{\log\pi}_\theta(a_t \mid r_t, H_{t-1})
    &\coloneqq \frac{1}{K_t}\sum_{j=1}^{K_t} \log \pi_\theta\bigl(\mathrm{tok}^{(t)}_j \,\big|\, \mathrm{tok}^{(t)}_{<j},\, r_t,\, H_{t-1}\bigr),
    \label{eq:per_token_pi}\\
  \widetilde{\log P}_\phi(a_t \mid H_{t-1}\oplus o_t^{\text{exec}})
    &\coloneqq \frac{1}{K_t}\sum_{j=1}^{K_t} \log P_\phi\bigl(\mathrm{tok}^{(t)}_j \,\big|\, \mathrm{tok}^{(t)}_{<j},\, H_{t-1}\oplus o_t^{\text{exec}}\bigr).
    \label{eq:per_token_pphi}
\end{align}
Each is the geometric mean (in log-space) of token probabilities along the $K_t$ tokens of one action edge; the edge thus carries a single, length-robust log-probability.
\end{definition}

\begin{definition}[TTB Residual]\label{def:ttb_residual}
Given the per-token-tempered edge log-probabilities of Definition~\ref{def:per_token_edge}, the trajectory balance residual is:
\begin{equation}\label{eq:app_delta}
  \Delta(\tau) \;\coloneqq\;
  \log Z_\theta(q)
  + \sum_{t=1}^{T} \widetilde{\log\pi}_\theta(a_t \mid r_t, H_{t-1})
  - \beta \log \Rtilde(\tau)
  - \sum_{t=1}^{T} \widetilde{\log P}_\phi(a_t \mid H_{t-1}\oplus o_t^{\text{exec}}),
\end{equation}
where:
\begin{itemize}[leftmargin=2em, itemsep=2pt]
  \item $Z_\theta(q)$ is a task-conditioned partition function (learned parameter),
  \item $\beta > 0$ is the temperature parameter controlling diversity,
  \item $\Rtilde(\tau) = R(\tau) + \varepsilon_{\min}$ is the smoothed reward,
  \item $r_t$ is the reasoning emitted at step $t$, treated as fixed conditioning context (Lemma~\ref{lem:conditional_tb_reasoning}),
  \item $T = |\tau|$ is the trajectory length, i.e., the number of action edges.
\end{itemize}
At the optimum $\Delta(\tau)=0$, the induced \emph{tempered} forward distribution satisfies $\tilde\pi_\theta(\tau) \propto \Rtilde(\tau)^{\beta}$ (Lemma~\ref{lem:per_token_tempering}).
\end{definition}

The TB condition in log-space (Equation~\eqref{eq:app_tb_log}) requires $\Delta(\tau) = 0$. To enforce this across all trajectories, we use a regression loss:

\begin{definition}[Tempered Trajectory Balance Loss]\label{def:ttb_loss}
The TTB loss is:
\begin{equation}\label{eq:ttb_full}
  \mathcal{L}_{\mathrm{TTB}}(\tau) = \left(\frac{\Delta(\tau)}{T}\right)^{\!2},
\end{equation}
with $\Delta(\tau)$ as in Definition~\ref{def:ttb_residual} and division by $T$ comparing trajectories of different lengths on equal footing.
\end{definition}

\begin{lemma}[Length Normalization Property]\label{lem:length_norm}
Each tempered edge log-probability is $O(1)$ by per-token averaging (Definition~\ref{def:per_token_edge}), so $|\sum_{t=1}^T \widetilde{\log\pi}_\theta(a_t\mid r_t,H_{t-1})| = O(T)$ and unmormalized $|\Delta(\tau)|\propto T$ for typical trajectories. Dividing by $T$ before squaring penalizes the \emph{average} per-step balance rather than the cumulative balance, keeping $|\Delta(\tau)/T| = O(1)$ and stabilizing gradients across length variations.
\end{lemma}

\begin{proposition}[TTB Self-Annealing Property]\label{prop:ttb_self_anneal}
Treating $T = |\tau|$ as fixed for a given $\tau$, the gradient of the TTB loss satisfies:
\begin{equation}\label{eq:ttb_grad}
  \nabla_\theta \mathcal{L}_{\mathrm{TTB}}(\tau)
  = 2 \cdot \frac{\Delta(\tau)}{T} \cdot \nabla_\theta\!\left(\frac{\Delta(\tau)}{T}\right)
  = \underbrace{\frac{2\,\Delta(\tau)}{T^2}}_{\text{length-scaled annealing factor}} \cdot\, \nabla_\theta \Delta(\tau).
\end{equation}
As $\Delta(\tau)\to 0$, the prefine shrinks and $\|\nabla_\theta\mathcal{L}_{\mathrm{TTB}}\|\to 0$ automatically. Because $T$ varies across trajectories, the $1/T^2$ factor is a per-trajectory length normalizer and cannot be absorbed into a global constant.
\end{proposition}

The temperature $\beta$ is a critical hyperparameter controlling the diversity-quality tradeoff:

\begin{lemma}[Temperature Effect on Policy]\label{lem:temperature_effect}
At convergence, $\pi^*(\tau) \propto \Rtilde(\tau)^\beta$.
\begin{itemize}[leftmargin=2em, itemsep=2pt]
  \item As $\beta \to 0^+$: the policy approaches uniform distribution over trajectories (maximum diversity, ignoring rewards).
  \item As $\beta \to \infty$: the policy concentrates on the single highest-reward trajectory (maximum quality, no diversity).
  \item Intermediate $\beta$ trades off diversity and quality.
\end{itemize}
\end{lemma}

\subsection{Conditional TB, Hindsight Backward, and Per-Token Tempering}\label{app:conditional_tb}

The TTB residual in Definition~\ref{def:ttb_residual} departs from textbook TB in three ways: (i) the forward policy conditions on reasoning $r_t$; (ii) the backward policy conditions on the hindsight state $H_{t-1}\oplus o_t^{\text{exec}}$; (iii) each edge log-probability is per-token averaged. We establish three lemmas showing that none of these alters the action-DAG structure or the reward-matching guarantee.

\paragraph{(i) Reasoning as fixed context.}

\begin{lemma}[TB on the Action Sub-DAG with Reasoning as Fixed Context]\label{lem:conditional_tb_reasoning}
The hierarchical step policy (Eq.~\ref{eq:policy}, main text) factorizes as
\begin{equation}
  \pi_\theta(r_t, a_t \mid H_{t-1}) \;=\; \pi_\theta(r_t \mid H_{t-1}) \cdot \pi_\theta(a_t \mid r_t, H_{t-1}).
\end{equation}
Conditioning on the realized reasoning sequence $\{r_t\}_{t=1}^T$ and execution observations $\{o_t^{\text{exec}}\}_{t=1}^T$, only the action choices $\{a_t\}$ remain random. The conditional trajectory probability is
\begin{equation}
  P_\theta\!\bigl(\tau \,\big|\, \{r_t\},\{o_t^{\text{exec}}\}\bigr) \;=\; \prod_{t=1}^T \pi_\theta(a_t \mid r_t, H_{t-1}),
\end{equation}
which is precisely the forward policy on the \emph{action sub-DAG} $\mathcal{G}^{\mathrm{act}} \subset \mathcal{G}$. The TB condition (Theorem~\ref{thm:tb}) is a per-trajectory equality and therefore applies pointwise to each realized $(\{r_t\},\{o_t^{\text{exec}}\})$ context, yielding the log-space form
\begin{equation}\label{eq:conditional_tb}
  \log Z_\theta(q) + \sum_{t=1}^{T} \log \pi_\theta(a_t \mid r_t, H_{t-1})
  \;=\; \log F(\tau) + \sum_{t=1}^{T} \log P_B(s_{t-1} \mid s_t),
\end{equation}
which matches Eq.~\ref{eq:tb_condition} with $r_t$ in the conditional. Reward-matching of the resulting flow on $\mathcal{G}^{\mathrm{act}}$ is therefore equivalent to TB holding on every trajectory at every realized reasoning context.
\end{lemma}

\begin{proof}
Direct from the hierarchical factorization and Theorem~\ref{thm:tb}. Reasoning $r_t$ enters $\pi_\theta(a_t\mid r_t, H_{t-1})$ as a determining context; once realized, it is treated as part of the conditioning state for the action edge, exactly as the source-state $s_{t-1}=H_{t-1}$ is. The sub-DAG $\mathcal{G}^{\mathrm{act}}$ inherits its DAG structure from $\mathcal{G}$ (Theorem~\ref{thm:dag}). \qedhere
\end{proof}

\paragraph{(ii) Hindsight-asymmetric backward policy.}

\begin{lemma}[Hindsight-Conditioned Backward Equals Standard Backward on the Augmented State]\label{lem:hindsight_backward_eq}
Define the hindsight-enriched pre-action state $\tilde s_t \coloneqq H_{t-1} \oplus o_t^{\text{exec}}$. By Eq.~\ref{eq:state_update}, the post-action state $s_t = H_t$ deterministically encodes the joint $(H_{t-1},a_t,o_t^{\text{exec}})$ given $s_{t-1}=H_{t-1}$. Therefore conditioning on $s_t$ is equivalent to conditioning on $(\tilde s_t,a_t)$, and
\begin{equation}
  P_B(s_{t-1} \mid s_t) \;=\; P_\phi(a_t \mid \tilde s_t) \;=\; P_\phi(a_t \mid H_{t-1}\oplus o_t^{\text{exec}}),
\end{equation}
in the sense that both express the same conditional reverse mass. Substituting this equality into Eq.~\eqref{eq:conditional_tb} yields the SkillFlow TB condition
\begin{equation}\label{eq:skillflow_tb}
  \log Z_\theta(q) + \sum_{t=1}^{T} \log \pi_\theta(a_t \mid r_t, H_{t-1})
  \;=\; \beta\log\Rtilde(\tau) + \sum_{t=1}^{T} \log P_\phi(a_t \mid H_{t-1}\oplus o_t^{\text{exec}}),
\end{equation}
where $\log F(\tau)=\beta\log\Rtilde(\tau)$ at the reward-matching terminal. The information asymmetry $\tilde s_t \supset s_{t-1}$ (post-execution observation $o_t^{\text{exec}}$ added) is precisely what gives the step-importance ratio $I(t)=\pi_\theta(a_t\mid r_t,H_{t-1})/P_\phi(a_t\mid \tilde s_t)$ a non-trivial credit signal: high $I(t)$ marks decisions whose quality became clear only after execution.
\end{lemma}

\begin{proof}
By the strict-history-growth lemma (Lemma~\ref{lem:dag_acyclic}), $H_t$ uniquely determines its parent $H_{t-1}$ and the appended triple $(r_t,a_t,o_t^{\text{exec}})$. Hence $s_t \mapsto (s_{t-1},r_t,a_t,o_t^{\text{exec}})$ is a bijection. Marginalizing over $r_t$ (which is fixed in the conditional TB of Lemma~\ref{lem:conditional_tb_reasoning}) leaves the joint $(s_{t-1},a_t,o_t^{\text{exec}})$, equivalent to $(\tilde s_t,a_t)$. The reverse conditional therefore equals $P_\phi(a_t\mid \tilde s_t)$ by definition of $P_\phi$ as a learned discriminative reverse model on the hindsight state. \qedhere
\end{proof}

\paragraph{(iii) Per-token tempering preserves convergence semantics.}

\begin{lemma}[Per-Token Tempering as Edge-Level Geometric-Mean Rescaling]\label{lem:per_token_tempering}
Let $\widetilde{\log\pi}_\theta(a_t\mid r_t,H_{t-1})$ and $\widetilde{\log P}_\phi(a_t\mid H_{t-1}\oplus o_t^{\text{exec}})$ be as in Definition~\ref{def:per_token_edge}, and define the tempered edge probabilities
\begin{equation}
  \tilde\pi_\theta(a_t\mid r_t, H_{t-1}) \coloneqq \exp\!\bigl(\widetilde{\log\pi}_\theta(a_t\mid r_t,H_{t-1})\bigr),
  \quad
  \tilde P_\phi(a_t\mid \tilde s_t) \coloneqq \exp\!\bigl(\widetilde{\log P}_\phi(a_t\mid \tilde s_t)\bigr).
\end{equation}
Then:
\begin{enumerate}[label=(\roman*), leftmargin=2em, itemsep=2pt]
  \item $\tilde\pi_\theta$ is the geometric mean of token probabilities: $\tilde\pi_\theta(a_t)= \prod_j \pi_\theta(\mathrm{tok}_j)^{1/K_t}$, a length-normalized policy on the action edge.
  \item Substituting $\tilde\pi_\theta,\tilde P_\phi$ into the TB condition (Eq.~\eqref{eq:skillflow_tb}) and setting $\Delta(\tau)=0$ defines a tempered flow $\tilde F$ with $\tilde F(x) = \Rtilde(\tau_x)^\beta$. By Theorem~\ref{thm:gflownet_sampling} applied to $\tilde F$, the induced sampling distribution satisfies $\tilde\pi_\theta(\tau\mid q) \propto \Rtilde(\tau)^\beta$.
  \item Geometric-mean tempering is a strictly monotone positive rescaling of edge probabilities; trajectory probabilities are products of edge probabilities, so the trajectory-level reward-proportional ranking is preserved: $\tilde\pi_\theta(\tau_1) > \tilde\pi_\theta(\tau_2) \Leftrightarrow \Rtilde(\tau_1)^\beta > \Rtilde(\tau_2)^\beta \Leftrightarrow R(\tau_1) > R(\tau_2)$ (Proposition~\ref{prop:smoothing_monotone}).
\end{enumerate}
\end{lemma}

\begin{proof}
(i) Direct from $\exp\!\bigl(\tfrac{1}{K_t}\sum_j \log \pi_\theta(\mathrm{tok}_j)\bigr) = \prod_j \pi_\theta(\mathrm{tok}_j)^{1/K_t}$. (ii) The TTB residual is the same algebraic identity as TB with $P_F,P_B$ replaced by $\tilde\pi_\theta,\tilde P_\phi$; the proof of Theorem~\ref{thm:tb} carries through verbatim with these tempered policies, defining a reward-matching tempered flow. (iii) Monotone rescaling preserves $\arg\max$ and ordering. \qedhere
\end{proof}

\begin{remark}[Why Per-Token Tempering, Not Token-Level GFlowNet]\label{rem:per_token_clarification}
The per-token sums in Eq.~\eqref{eq:per_token_pi}--\eqref{eq:per_token_pphi} average \emph{within} a single action edge to produce one length-robust edge log-probability. The DAG nodes remain action-level states $H_t$; tokens are the LLM's internal autoregressive representation of one edge, not separate DAG nodes. This distinguishes our setting from token-level GFlowNet variants where each token is its own DAG edge.
\end{remark}

Together, Lemmas~\ref{lem:conditional_tb_reasoning}, \ref{lem:hindsight_backward_eq}, and \ref{lem:per_token_tempering} provide the formal justification for the SkillFlow TTB residual (Definition~\ref{def:ttb_residual}, equivalently main-text Eq.~\ref{eq:ttb}): it is the standard log-space TB residual on the action-DAG, evaluated with $r_t$-conditioned forward policy, hindsight-conditioned backward policy, and per-token-tempered edge log-probabilities.

\section{Entropy-Regularized RL Equivalence}\label{app:rl_equivalence}

This section proves the fundamental equivalence between GFlowNet training and entropy-regularized reinforcement learning.

\begin{theorem}[GFlowNet-RL Equivalence]\label{thm:rl_equiv}
The optimal policy induced by GFlowNet training with temperature $\beta$ is identical to the optimal policy of entropy-regularized maximum expected reward, with temperature parameter $T = 1/\beta$:
\begin{equation}\label{eq:app_rl_equiv}
  \pi^*_{\mathrm{GFN}}(\tau \mid q) = \frac{\Rtilde(\tau)^\beta}{\sum_{\tau'} \Rtilde(\tau')^\beta}
  = \frac{\exp\bigl[\beta \log \Rtilde(\tau)\bigr]}{Z_\beta(q)}.
\end{equation}
This matches the optimal policy from the entropy-regularized RL objective:
\begin{equation}\label{eq:rl_objective}
  \pi^*_{\mathrm{RL}}(\tau \mid q) = \arg\max_\pi \E_{\tau \sim \pi}\left[\log R(\tau)\right] + \frac{1}{\beta} \mathcal{H}[\pi],
\end{equation}
where $\mathcal{H}[\pi] = -\E_\tau[\log \pi(\tau)]$ is the policy entropy.
\end{theorem}

\begin{proof}
The entropy-regularized RL objective (with temperature $T = 1/\beta$) is:
\begin{equation}
  J(\pi) = \E_{\tau \sim \pi}\left[\log R(\tau)\right] + T \cdot \mathcal{H}[\pi]
  = \E_{\tau \sim \pi}\left[\log R(\tau) - \log \pi(\tau)\right] + T \log(|\mathcal{T}|),
\end{equation}
where $|\mathcal{T}|$ is the number of trajectories (constant in $\pi$).

Taking the functional derivative with respect to $\pi(\tau)$ and setting it to zero:
\begin{equation}
  \frac{\delta J}{\delta \pi(\tau)} = \log R(\tau) - \log \pi^*(\tau) - 1 + T \log(|\mathcal{T}|) = 0.
\end{equation}
Solving for $\pi^*(\tau)$:
\begin{equation}
  \log \pi^*(\tau) = \log R(\tau) + T \log(|\mathcal{T}|) - 1
  = \frac{1}{T} \log R(\tau) + \text{const}.
\end{equation}
Exponentiating:
\begin{equation}
  \pi^*(\tau) \propto \exp\left(\frac{1}{T} \log R(\tau)\right) = R(\tau)^{1/T} = R(\tau)^\beta.
\end{equation}
Normalizing:
\begin{equation}
  \pi^*(\tau) = \frac{R(\tau)^\beta}{\sum_{\tau'} R(\tau')^\beta}.
\end{equation}
With smoothing $\Rtilde(\tau) = R(\tau) + \varepsilon_{\min}$, this becomes:
\begin{equation}
  \pi^*(\tau) = \frac{\Rtilde(\tau)^\beta}{\sum_{\tau'} \Rtilde(\tau')^\beta} = \frac{\exp[\beta \log \Rtilde(\tau)]}{Z_\beta(q)},
\end{equation}
which exactly matches the GFlowNet sampling distribution from Theorem~\ref{thm:gflownet_sampling}. \qedhere
\end{proof}

\begin{remark}[Practical Implications]
This equivalence provides a principled interpretation of GFlowNet training: maximizing reward while maintaining policy entropy, under the weighting controlled by $\beta$.
Unlike standard policy gradient methods (REINFORCE, GRPO), which optimize toward a single reward-maximum strategy, the entropy regularization preserves diverse solutions with comparable rewards.
\end{remark}

\section{Detailed Balance and Flow Metrics}\label{app:detailed_balance}

This section establishes the Detailed Balance condition and derives the flow-based credit assignment metrics.

\subsection{Detailed Balance Theorem}\label{app:db_theorem}

\begin{theorem}[Detailed Balance]\label{thm:db}
For a reward-matching flow $F$ satisfying the Trajectory Balance condition (Theorem~\ref{thm:tb}), the Detailed Balance (DB) condition holds at every edge:
\begin{equation}\label{eq:db}
  F(s) \cdot P_F(s' \mid s) = F(s') \cdot P_B(s \mid s'),
\end{equation}
where $P_F(s' \mid s) = F(s \to s') / F(s)$ and $P_B(s \mid s') = F(s \to s') / F(s')$.
\end{theorem}

\begin{proof}
By definition of $P_F$ and $P_B$ from Definition~\ref{def:policies}:
\begin{align}
  \text{LHS} &= F(s) \cdot \frac{F(s \to s')}{F(s)} = F(s \to s'),\\
  \text{RHS} &= F(s') \cdot \frac{F(s \to s')}{F(s')} = F(s \to s').
\end{align}
Both sides equal the edge flow $F(s \to s')$. \qedhere
\end{proof}

Detailed Balance is a pointwise (per-edge) condition, stronger than Trajectory Balance (which is trajectory-wise).
For a reward-matching flow, DB emerges as a consequence.
In SkillFlow's tree-structured DAG (each $H_t$ has a unique parent by strict history growth), every trajectory through a state is unique, so TB at convergence uniquely determines the edge flows and DB follows. The next lemma makes this uniqueness explicit.

\begin{lemma}[Tree-DAG Specialization: Uniqueness of the Reward-Matching Flow and Per-Edge DB]\label{lem:tree_db_uniqueness}
On a tree-structured DAG (every non-source node has a unique parent), as is the case for SkillFlow's orchestration graph by Lemma~\ref{lem:dag_acyclic}, a reward-matching flow $F$ with terminal values $\{F(x) = R(x)\}_{x \in \mathcal{X}}$ is uniquely determined: for every non-terminal state $s$,
\begin{equation}\label{eq:tree_F_unique}
  F(s) \;=\; \sum_{x \in \mathrm{Desc}(s)\,\cap\,\mathcal{X}} R(x),
\end{equation}
where $\mathrm{Desc}(s)$ is the set of descendants of $s$. Consequently, TB convergence (Theorem~\ref{thm:tb}) on a tree-DAG simultaneously enforces (i) reward-matching at terminals, (ii) flow uniqueness at every state, and (iii) per-edge Detailed Balance (Theorem~\ref{thm:db}) at every edge.
\end{lemma}

\begin{proof}
We prove Eq.~\eqref{eq:tree_F_unique} by reverse induction on the depth of $s$ (longest path from the source).

\emph{Base case} (terminal $s = x \in \mathcal{X}$): $F(x) = R(x)$ by reward-matching, and $\mathrm{Desc}(x) \cap \mathcal{X} = \{x\}$, so Eq.~\eqref{eq:tree_F_unique} is the trivial identity $R(x) = R(x)$.

\emph{Inductive step}: Suppose Eq.~\eqref{eq:tree_F_unique} holds for all states deeper than $s$. By flow conservation (Definition~\ref{def:state_flow}),
\begin{equation}
  F(s) = \sum_{s' \in \mathrm{Ch}(s)} F(s \to s') = \sum_{s' \in \mathrm{Ch}(s)} F(s'),
\end{equation}
where the second equality uses tree structure: each child $s'$ has $s$ as its unique parent, so all flow into $s'$ comes from the single edge $s \to s'$, i.e., $F(s') = F(s \to s')$. By the inductive hypothesis,
\begin{equation}
  F(s) = \sum_{s' \in \mathrm{Ch}(s)} \;\;\sum_{x \in \mathrm{Desc}(s')\,\cap\,\mathcal{X}} R(x)
       = \sum_{x \in \mathrm{Desc}(s)\,\cap\,\mathcal{X}} R(x),
\end{equation}
where the last equality uses the disjoint partition $\mathrm{Desc}(s)\setminus\{s\} = \bigsqcup_{s' \in \mathrm{Ch}(s)} (\{s'\}\cup\mathrm{Desc}(s'))$, valid because subtrees of distinct children of $s$ are disjoint on a tree.

For (iii), Theorem~\ref{thm:db} (proved for general DAGs) yields DB at every edge once $F$ is reward-matching. \qedhere
\end{proof}

\subsection{Flow Ratio and Step Importance}\label{app:flow_ratio}

\begin{corollary}[Flow Ratio]\label{cor:flow_ratio}
From Detailed Balance (Theorem~\ref{thm:db}), $F(s_{t-1}) \cdot P_F(s_t \mid s_{t-1}) = F(s_t) \cdot P_B(s_{t-1} \mid s_t)$. Dividing both sides by $F(s_{t-1}) \cdot P_B(s_{t-1} \mid s_t)$:
\begin{equation}\label{eq:flow_ratio}
  \frac{F(s_t)}{F(s_{t-1})} = \frac{P_F(s_t \mid s_{t-1})}{P_B(s_{t-1} \mid s_t)}.
\end{equation}
Substituting the SkillFlow policy realizations from Lemmas~\ref{lem:conditional_tb_reasoning} and \ref{lem:hindsight_backward_eq}:
\begin{equation}
  \frac{F(s_t)}{F(s_{t-1})} \;=\; \frac{\pi_\theta(a_t \mid r_t, H_{t-1})}{P_\phi(a_t \mid H_{t-1} \oplus o_t^{\text{exec}})}.
\end{equation}
\end{corollary}

This leads to the key quantity for credit assignment:

\begin{definition}[Step Importance]\label{def:step_importance}
The step importance at time $t$ is:
\begin{equation}\label{eq:step_importance_def}
  I(t) \;\coloneqq\; \frac{F(s_t)}{F(s_{t-1})} \;=\; \frac{\pi_\theta(a_t \mid r_t, H_{t-1})}{P_\phi(a_t \mid H_{t-1} \oplus o_t^{\text{exec}})}.
\end{equation}
This ratio quantifies the ``flow amplification'' at step $t$: how much the state flow increases (if $I(t) > 1$) or decreases (if $I(t) < 1$) due to action $a_t$, measured against the hindsight-asymmetric backward (Lemma~\ref{lem:hindsight_backward_eq}).
\end{definition}

Intuitively, $I(t)$ measures the information value of decision $a_t$:
\begin{itemize}[leftmargin=2em, itemsep=2pt]
  \item $I(t) \gg 1$: The decision was high-probability for the forward policy but would have been low-probability in hindsight; this decision had high impact.
  \item $I(t) \approx 1$: The forward and backward policies agree; the decision had typical impact.
  \item $I(t) \ll 1$: The decision was low-probability in hindsight relative to the forward policy; this decision was sub-optimal in retrospect.
\end{itemize}

\subsection{Skill Marginal Flow}\label{app:skill_marginal_flow}

\begin{definition}[Skill Marginal Flow]\label{def:skill_marginal_flow}
For a skill $s \in \Slib$, the marginal flow is the average post-action \emph{state flow} at nodes where $s$ is invoked, aggregated over trajectories in the batch:
\begin{equation}\label{eq:skill_flow_def}
  \hat{F}(s) \;\coloneqq\; \frac{1}{|\Batch_s|} \sum_{\tau \in \Batch_s} \sum_{t:\, a_t \text{ invokes } s} F(s_t),
\end{equation}
where $\Batch_s \subseteq \Batch$ is the subset of sampled trajectories that invoke skill $s$, and the post-action state flow is recovered by telescoping the step importance from the source $F(s_0) = Z_\theta(q)$:
\begin{equation}\label{eq:state_flow_telescope}
  F(s_t) \;=\; Z_\theta(q) \cdot \prod_{t'=1}^{t} I(t')
  \quad\Longleftrightarrow\quad
  \log F(s_t) \;=\; \log Z_\theta(q) + \sum_{t'=1}^{t}\log I(t').
\end{equation}
This matches main-text Eq.~\ref{eq:skill_flow} exactly.
\end{definition}

Skill marginal flow is the key signal for detecting which skills contribute most to reward-proportional sampling.
High $\hat{F}(s)$ indicates that $s$ is invoked at high-flow states---which by Theorem~\ref{thm:gflownet_sampling} are visited with probability proportional to downstream reward.
Low $\hat{F}(s)$ indicates that $s$ either (i) is rarely invoked, (ii) is invoked along low-flow / low-reward trajectories, or (iii) is invoked at decision points where the cumulative information advantage of the policy is small (small $\sum_{t'\le t}\log I(t')$).

\subsection{Zero-Cost Computation}\label{app:zero_cost}

\begin{proposition}[Zero-Cost Flow Metrics]\label{prop:zero_cost}
All flow metrics---state flows $F(s_t)$, step importances $I(t)$, and skill marginal flows $\hat{F}(s)$---are computable from the (per-token-tempered) forward and backward log-probabilities already evaluated during the TTB loss computation (Eq.~\eqref{eq:ttb_full}).

Concretely:
\begin{align}
  \log I(t) &= \widetilde{\log\pi}_\theta(a_t \mid r_t, H_{t-1}) - \widetilde{\log P}_\phi(a_t \mid H_{t-1} \oplus o_t^{\text{exec}}),
  \label{eq:logI_zero_cost}\\
  \log F(s_t) &= \log Z_\theta(q) + \sum_{t'=1}^{t} \log I(t'),
  \label{eq:logF_zero_cost}\\
  \log \hat{F}(s) &= \log Z_\theta(q) + \log\!\left(\frac{1}{|\Batch_s|} \sum_{\tau \in \Batch_s} \sum_{t:\, a_t \text{ invokes } s} \exp\!\Bigl(\sum_{t'=1}^{t}\log I(t')\Bigr)\right).
  \label{eq:logFhat_zero_cost}
\end{align}
Equation~\eqref{eq:logFhat_zero_cost} is the log of an arithmetic mean (matching the linear-scale definition in Eq.~\eqref{eq:skill_flow_def}), \emph{not} the mean-of-log: by the telescoping identity each visit contributes $F(s_t) = Z_\theta(q)\cdot\exp(\sum_{t'\le t}\log I(t'))$, and we average those visit-flows in linear space before re-taking the log. This expression is precisely the per-skill CGF at $\lambda=1$ shifted by $\log Z_\theta(q)$ (cf.\ main-text Eq.~13 and the equality $\Lambda^{(s)}_1 = \log\hat F(s) - \log Z_\theta(q)$). No additional forward or backward passes through the policy models are required.
\end{proposition}

This zero-cost property is crucial for SkillFlow: the flow signals that drive skill evolution incur no computational overhead beyond the standard TTB loss evaluation.

\subsection{Cumulant Generating Function (CGF) Properties}\label{app:cgf_properties}

Section~\ref{sec:evolution} of the main text defines, for each skill $s \in \Slib$, the per-skill cumulant generating function (CGF) of the telescoped log step-importance:
\begin{equation}\label{eq:app_cgf}
  \Lambda^{(s)}_\lambda \;\coloneqq\; \log\!\left(\frac{1}{|\Batch_s|} \sum_{\tau \in \Batch_s} \sum_{t:\, a_t \text{ invokes } s} \exp\!\Bigl(\lambda\!\sum_{t'=1}^{t}\log I(t')\Bigr)\right),\qquad \lambda \in \R.
\end{equation}
Let $V_s \coloneqq \{(\tau,t)\,:\,\tau\in\Batch_s,\, a_t \text{ invokes } s\}$ denote the set of $s$-\emph{visits} in the batch, and for each visit $v=(\tau,t)\in V_s$ define the telescoped log-flow share
\begin{equation}\label{eq:Xv_def}
  X_v \;\coloneqq\; \sum_{t'=1}^{t} \log I(t') \;\stackrel{\text{Eq.~}\eqref{eq:state_flow_telescope}}{=}\; \log F(s_t) \,-\, \log Z_\theta(q).
\end{equation}
This appendix proves the four properties of $\Lambda^{(s)}_\lambda$ used by the curation operator $\Phi$ (Eq.~\ref{eq:evolution}, main text): convexity, $G(s)$ as the visit-mean of $X$, $\Lambda^{(s)}_1$ as the centered log skill marginal flow, and the Jensen-gap cumulant expansion.

\paragraph{Atomicity assumption.}
By the \emph{atomic-tip} property in main-text §4.3, each tip $s$ is self-contained and independently composable; we assume each atomic tip is invoked \emph{at most once} per trajectory in $\Batch_s$. Under this assumption, $|V_s| = |\Batch_s|$ and Eq.~\eqref{eq:app_cgf} reduces to the standard sample CGF over the empirical visit-distribution:
\begin{equation}\label{eq:cgf_simplified}
  \Lambda^{(s)}_\lambda \;=\; \log\!\left(\frac{1}{|V_s|}\sum_{v \in V_s} e^{\lambda X_v}\right) \;=\; \log\,\E_{V_s}\!\bigl[e^{\lambda X}\bigr].
\end{equation}
We treat $\Lambda^{(s)}_\lambda$ as the standard CGF over $\{X_v\}_{v\in V_s}$ throughout.

\begin{lemma}[Convexity of the CGF in $\lambda$]\label{lem:cgf_convexity}
$\Lambda^{(s)}_\lambda$ is convex in $\lambda \in \R$.
\end{lemma}

\begin{proof}
By Eq.~\eqref{eq:cgf_simplified}, $\Lambda^{(s)}_\lambda = \log\sum_v e^{\lambda X_v} - \log|V_s|$. Each $\lambda X_v$ is affine in $\lambda$, $\log\sum e^{(\cdot)}$ (log-sum-exp) is a standard convex function of its arguments, and convexity is preserved under affine pre-composition and constant shift. Hence $\Lambda^{(s)}_\lambda$ is convex in $\lambda$. \qedhere
\end{proof}

\begin{lemma}[Mean Log-Flow Identity for $G(s)$]\label{lem:G_is_mean}
\begin{equation}\label{eq:G_def_proof}
  G(s) \;\coloneqq\; \frac{\partial \Lambda^{(s)}_\lambda}{\partial \lambda}\bigg|_{\lambda=0}
  \;=\; \frac{1}{|V_s|}\sum_{v \in V_s} X_v
  \;=\; \E_{V_s}\!\bigl[\log F(s_t)\bigr] \,-\, \log Z_\theta(q).
\end{equation}
That is, $G(s)$ is the visit-average of the centered log state-flow $\log F(s_t) - \log Z_\theta(q)$ over occurrences of $s$.
\end{lemma}

\begin{proof}
Differentiate Eq.~\eqref{eq:cgf_simplified} with respect to $\lambda$:
\begin{equation}
  \frac{\partial}{\partial\lambda}\log\!\Bigl(\tfrac{1}{|V_s|}\sum_v e^{\lambda X_v}\Bigr)
  \;=\; \frac{\sum_v X_v\, e^{\lambda X_v}}{\sum_v e^{\lambda X_v}}.
\end{equation}
At $\lambda=0$ this evaluates to $\frac{\sum_v X_v}{\sum_v 1} = \frac{1}{|V_s|}\sum_v X_v$, the empirical visit-mean. Substituting Eq.~\eqref{eq:Xv_def} for $X_v$ gives the second equality. \qedhere
\end{proof}

\begin{lemma}[CGF at $\lambda = 1$ Recovers the Skill Marginal Flow]\label{lem:cgf_at_one}
\begin{equation}\label{eq:Lambda_one}
  \Lambda^{(s)}_1 \;=\; \log \hat{F}(s) \,-\, \log Z_\theta(q),
\end{equation}
where $\hat{F}(s)$ is the skill marginal flow (Definition~\ref{def:skill_marginal_flow}). This recovers main-text Eq.~\ref{eq:skill_flow}.
\end{lemma}

\begin{proof}
Substituting $X_v = \log F(s_t^{(v)}) - \log Z_\theta(q)$ into Eq.~\eqref{eq:cgf_simplified} at $\lambda=1$:
\begin{align*}
  \Lambda^{(s)}_1
  &= \log\!\left(\frac{1}{|V_s|}\sum_v e^{X_v}\right)
   = \log\!\left(\frac{1}{|V_s|}\sum_v \frac{F(s_t^{(v)})}{Z_\theta(q)}\right)\\
  &= \log\!\left(\frac{1}{Z_\theta(q)} \cdot \frac{1}{|V_s|}\sum_v F(s_t^{(v)})\right)
   \;=\; \log \hat{F}(s) - \log Z_\theta(q),
\end{align*}
where the last step uses Definition~\ref{def:skill_marginal_flow} and the atomicity identity $|V_s|=|\Batch_s|$. \qedhere
\end{proof}

\begin{lemma}[Jensen Inequality on the CGF]\label{lem:cgf_jensen}
For every skill $s$ with $V_s \neq \emptyset$,
\begin{equation}\label{eq:cgf_jensen}
  G(s) \;\le\; \Lambda^{(s)}_1,
\end{equation}
with equality if and only if $X_v$ is constant across all visits $v \in V_s$.
\end{lemma}

\begin{proof}
Apply Jensen's inequality to the convex function $\exp$:
\begin{equation}
  \exp\!\Bigl(\tfrac{1}{|V_s|}\sum_v X_v\Bigr) \;\le\; \tfrac{1}{|V_s|}\sum_v e^{X_v},
\end{equation}
with equality iff $X_v$ is constant. Taking $\log$ on both sides yields $G(s) \le \log\bigl(\tfrac{1}{|V_s|}\sum_v e^{X_v}\bigr) = \Lambda^{(s)}_1$, where the equality on the right uses Eq.~\eqref{eq:cgf_simplified}. \qedhere
\end{proof}

\begin{proposition}[Jensen Gap as Cumulant Expansion]\label{prop:jensen_gap_cumulants}
The Jensen gap admits the cumulant expansion
\begin{equation}\label{eq:cgf_cumulants}
  \Lambda^{(s)}_1 - G(s)
  \;=\; \tfrac{1}{2}\,\Var_{V_s}\!\bigl[\log F(s_t)\bigr]
  \;+\; \sum_{k \ge 3} \frac{\kappa_k(s)}{k!},
\end{equation}
where $\kappa_k(s)$ is the $k$-th empirical cumulant of $\{X_v\}_{v \in V_s}$ (equivalently, of $\{\log F(s_t^{(v)})\}_v$, since $\log Z_\theta(q)$ is a visit-independent shift). The leading term is one-half the cross-visit variance of the log state-flow at occurrences of $s$.
\end{proposition}

\begin{proof}
The cumulant generating function admits the standard Taylor expansion
\begin{equation}
  \log\E\!\bigl[e^{\lambda X}\bigr] \;=\; \sum_{k\ge 1} \frac{\kappa_k}{k!}\,\lambda^k
  \;=\; \kappa_1\lambda + \frac{\kappa_2}{2}\lambda^2 + \sum_{k\ge 3}\frac{\kappa_k}{k!}\lambda^k,
\end{equation}
where $\kappa_1 = \mu = \E[X]$, $\kappa_2 = \sigma^2 = \Var[X]$, and $\{\kappa_k\}_{k\ge 3}$ are the higher cumulants. Apply this expansion to Eq.~\eqref{eq:cgf_simplified} (which expresses $\Lambda^{(s)}_\lambda$ as a CGF):
\begin{equation}
  \Lambda^{(s)}_\lambda \;=\; G(s)\,\lambda + \frac{\Var_{V_s}[X]}{2}\,\lambda^2 + \sum_{k\ge 3}\frac{\kappa_k(s)}{k!}\,\lambda^k.
\end{equation}
Setting $\lambda = 1$ and rearranging gives Eq.~\eqref{eq:cgf_cumulants}. The variance of $X_v$ equals the variance of $\log F(s_t)$ over $V_s$ since the additive constant $-\log Z_\theta(q)$ does not affect variance. \qedhere
\end{proof}

\begin{remark}[Stability Diagnostic via the Jensen Gap]\label{rem:stability_diagnostic}
A small Jensen gap $\Lambda^{(s)}_1 - G(s)$ means $X_v$ is approximately constant across the visits of $s$, i.e., $\log F(s_t)$ takes nearly the same value whenever $s$ is invoked: $s$ contributes \emph{consistently} to flow regardless of context. A large gap conversely indicates that $s$'s flow contribution varies substantially across contexts, signaling a \emph{context-inconsistent} skill that may merit refining---the use case for the $\mathcal{U}$ class in main-text Eq.~\ref{eq:evolution} (formal definition in Appendix~\ref{app:skill_curation}).
\end{remark}

\begin{remark}[Centered Log-Flow Share]\label{rem:centered_share}
The centered log-flow share
\begin{equation}\label{eq:wide_lambda_def}
  \widetilde\Lambda(s) \;\coloneqq\; \Lambda^{(s)}_1 - \E_{s'}\!\bigl[\Lambda^{(s')}_1\bigr]
\end{equation}
of main-text Eq.~\ref{eq:evolution} subtracts the library-mean log skill marginal flow from $\Lambda^{(s)}_1$, ranking each skill's contribution \emph{relative} to the library. By Lemma~\ref{lem:cgf_at_one}, $\widetilde\Lambda(s)$ depends on $s$ only through $\hat F(s)$ (the $\log Z_\theta(q)$ shift cancels), so $\widetilde\Lambda(s) = \log\hat F(s) - \E_{s'}\![\log\hat F(s')]$ is a relative log-flow rank.
\end{remark}

\section{DAG Acyclicity Proof}\label{app:dag_acyclicity}

\begin{theorem}[Orchestration Graph is a DAG]\label{thm:dag}
Under the SkillFlow environment definition with three-way policy factorization (Equation~\eqref{eq:policy}, main text) and frozen skill library within each training phase, the orchestration state graph $\mathcal{G}$ is a directed acyclic graph.
\end{theorem}

\begin{proof}
Define a strict order on states by their history length: $\text{depth}(s) := |H_s|$, i.e., the number of tokens in the interaction history at state $s$.

By the state-update rule (main-text Eq.~\eqref{eq:state_update}), $H_t = H_{t-1} \oplus (r_t, a_t, o_t^{\text{exec}})$ strictly appends three non-empty components to the history.
Therefore:
\begin{equation}
  \text{depth}(s_t) = |H_t| > |H_{t-1}| = \text{depth}(s_{t-1}).
\end{equation}

For any directed edge $s \to s'$ in the orchestration graph, we have:
\begin{equation}
  \text{depth}(s') > \text{depth}(s).
\end{equation}

A cycle would be a path $s_0 \to s_1 \to \cdots \to s_k \to s_0$ with $k \geq 1$.
Following the edges:
\begin{equation}
  \text{depth}(s_1) > \text{depth}(s_0),\quad \text{depth}(s_2) > \text{depth}(s_1),\quad \ldots,\quad \text{depth}(s_0) > \text{depth}(s_k).
\end{equation}
Combining: $\text{depth}(s_0) > \text{depth}(s_k) > \cdots > \text{depth}(s_0)$, a contradiction.

Therefore, no cycle exists, and $\mathcal{G}$ is a DAG. \qedhere
\end{proof}

\section{Skill Curation Details}\label{app:skill_curation}

This appendix specifies (i) the phase-boundary detection rule, (ii) the curation classes $\mathcal{D}^-,\mathcal{R},\mathcal{U}$ used by the operator $\Phi$ in main-text Eq.~\ref{eq:evolution}, (iii) the Skill Creator $\Psi$ with its trigger steps, (iv) the resulting curation algorithm, and (v) the formal definition of \emph{atomic composability} together with the proof that $\Phi$ preserves it.

\subsection{Phase-Boundary Detection Rule}\label{app:phase_detection}

Within phase $k$, the squared TTB residual $\Delta(\tau)^2$ is tracked as a running mean over a sliding window of $W$ training steps:
\begin{equation}\label{eq:running_mean}
  \overline{\Delta^2}_w^{(k)} \;\coloneqq\; \frac{1}{W} \sum_{i=w-W+1}^{w} \frac{1}{|\Batch_i|}\sum_{\tau \in \Batch_i} \Delta(\tau \mid \Slib^{(k)})^2,
\end{equation}
where $w$ is the current step within phase $k$ and $\Batch_i$ is the mini-batch at step $i$. By Proposition~\ref{prop:training} and the residual-floor identity $\bar\Delta^{*(k)} = \inf_\theta \E_\tau[\Delta(\tau\mid\Slib^{(k)},\theta)^2]$ (main-text Eq.~\ref{eq:residual_bound}), $\overline{\Delta^2}_w^{(k)}$ asymptotes to a value $\ge \bar\Delta^{*(k)}$, and gradient descent halts further reduction once that floor is reached.

\begin{definition}[Plateau Trigger]\label{def:plateau}
Let $\rho > 0$ be a relative-decrease tolerance and $M$ a window-count budget. We say training has \emph{plateaued} at step $w$ within phase $k$ if
\begin{equation}\label{eq:plateau_criterion}
  \frac{\overline{\Delta^2}_{w-W}^{(k)} \,-\, \overline{\Delta^2}_w^{(k)}}{\overline{\Delta^2}_{w-W}^{(k)}} \;<\; \rho.
\end{equation}
Phase $k+1$ is triggered at the first step $w$ for which Eq.~\eqref{eq:plateau_criterion} holds for $M$ consecutive non-overlapping windows.
\end{definition}

Concrete values for $W, \rho, M$ are fixed throughout training and detailed in the supplementary code.

\subsection{Curation Classes}\label{app:curation_classes}

At every triggered phase boundary $k \to k+1$, each existing skill $s \in \Slib^{(k)}$ is classified into one of three disjoint sets via the CGF statistics of Lemmas~\ref{lem:G_is_mean}--\ref{lem:cgf_jensen} (Appendix~\ref{app:cgf_properties}). Let $\Phi^G_{\text{thr}}, \Phi^J_{\text{thr}} \in \R$ be hyperparameters; let $n^-(s)$ count the cumulative number of past phase boundaries at which $\widetilde\Lambda(s) < 0$.

\begin{definition}[Curation Classes]\label{def:curation_classes}
The library is partitioned into three disjoint subsets:
\begin{align}
  \mathcal{D}^-_k &\;\coloneqq\; \bigl\{\, s \in \Slib^{(k)} \;:\; n^-(s) \ge K^-\bigr\}, \\
  \mathcal{R}_k &\;\coloneqq\; \bigl\{\, s \in \Slib^{(k)}\setminus \mathcal{D}^-_k \;:\; G(s) \ge \Phi^G_{\text{thr}},\; \Lambda^{(s)}_1 - G(s) \le \Phi^J_{\text{thr}}\bigr\}, \\
  \mathcal{U}_k &\;\coloneqq\; \Slib^{(k)}\setminus \bigl(\mathcal{D}^-_k \cup \mathcal{R}_k\bigr).
\end{align}
The three roles are: $\mathcal{D}^-_k$ \emph{prunes} skills with persistently negative centered share; $\mathcal{R}_k$ \emph{retains} skills with high mean log-flow and low Jensen gap; $\mathcal{U}_k$ \emph{refines} the remaining skills (high Jensen gap or low $G$).
\end{definition}

The Jensen gap threshold $\Phi^J_{\text{thr}}$ implements the stability diagnostic of Remark~\ref{rem:stability_diagnostic}: a skill with consistent flow contribution across visits has small gap and is retained; a context-inconsistent skill has large gap and is refined. Concrete threshold values are detailed in the supplementary code.

\subsection{Skill Creation from Success/Failure Trajectory Pairs}\label{app:skill_creation}

The Skill Creator $\Psi$ is invoked at high-step-importance positions where successful and failed trajectories from the same query diverge.

\begin{definition}[Trigger Steps]\label{def:trigger_steps}
For each query $q$ in the validation pool, let $\tau^+$ be a successful trajectory ($R(\tau^+) = 1$) and $\tau^-$ a same-query failed trajectory ($R(\tau^-) = 0$) sampled under $\Slib^{(k)}$ with the current policy $\pi_\theta$. The set of \emph{trigger steps} for $(q, \tau^+, \tau^-)$ is
\begin{equation}\label{eq:trigger_steps}
  \mathcal{T}_q^{\,\text{trig}} \;\coloneqq\; \Bigl\{\, t \in \{1,\dots,|\tau^+|\} \;:\; \log I(t)\big|_{\tau^+} \ge \zeta_{\text{trig}} \;\wedge\; t \notin \mathrm{cov}(\mathcal{R}_k \cup \mathcal{U}'_k) \Bigr\},
\end{equation}
where $\zeta_{\text{trig}}$ is a high-importance threshold and $\mathrm{cov}(\cdot)$ marks steps already covered by surviving (or refined) skills. By Lemma~\ref{lem:hindsight_backward_eq}, $\log I(t)\big|_{\tau^+} \gg 0$ marks decisions whose quality became clear only under the hindsight backward---precisely the gap candidates for new tips.
\end{definition}

\begin{definition}[Skill Creator $\Psi$]\label{def:skill_creator}
The Skill Creator $\Psi$ is a frozen LLM-based generator constrained to render \emph{atomic tips} (Definition~\ref{def:atomic_tip}). Given creation context $c = (q, \tau^+, \tau^-, t)$ for each $t \in \mathcal{T}_q^{\,\text{trig}}$, $\Psi$ outputs a textual atomic tip
\begin{equation}\label{eq:psi_call}
  s_{\text{new}} \;=\; \Psi(c, \mathcal{T}, \Slib^{(k)})
\end{equation}
where $\mathcal{T}$ denotes the validation buffer of $(q, \tau^+, \tau^-)$ trajectory pairs available at the phase boundary. The tip captures the strategic decision differentiating $\tau^+$ from $\tau^-$ at step $t$. The output is constrained at the prompt level to be (a) self-contained (no reference to other tips at runtime), (b) of bounded textual length $L_{\max}$, and (c) phrased as strategic guidance (not as a literal action). $\Psi$ also operates in a \emph{refine mode} on $\mathcal{U}_k$, rewriting context-inconsistent tips under the same atomic constraints.
\end{definition}

\subsection{The Curation Operator $\Phi$}\label{app:curation_operator}

\begin{definition}[Evolution Operator]\label{def:phi_operator}
The CGF-based curation operator $\Phi$ produces the next-phase library by combining three sets:
\begin{equation}\label{eq:phi_def}
  \Slib^{(k+1)} \;\coloneqq\; \Phi\bigl(\Slib^{(k)};\, \{(G(s), \widetilde\Lambda(s))\}_{s\in\Slib^{(k)}},\, \{\log I(t)\}_t\bigr)
  \;=\; \mathcal{R}_k \,\cup\, \mathcal{U}'_k \,\cup\, \Psi^{\text{new}}_k,
\end{equation}
where $\mathcal{R}_k$ is the retained set (Definition~\ref{def:curation_classes}); $\mathcal{U}'_k \coloneqq \{\Psi(s,\,\text{refine}) : s \in \mathcal{U}_k\}$ is the refined set; and
\begin{equation}\label{eq:psi_new_def}
  \Psi^{\text{new}}_k \;\coloneqq\; \bigcup_{(q,\tau^+,\tau^-)} \;\;\bigcup_{t \in \mathcal{T}_q^{\,\text{trig}}} \;\bigl\{\Psi\bigl((q,\tau^+,\tau^-,t),\,\mathcal{T},\,\Slib^{(k)}\bigr)\bigr\}
\end{equation}
is the newly-created tip set. Pruned tips $\mathcal{D}^-_k$ are absent from $\Slib^{(k+1)}$ by construction.
\end{definition}

\subsection{Atomic Composability and Its Preservation}\label{app:atomic_composability}

\begin{definition}[Atomic Tip]\label{def:atomic_tip}
A skill $s$ is an \emph{atomic tip} if (i) $s$ is a textual prompt fragment of bounded length $\le L_{\max}$, (ii) the action $\texttt{skill}(s)$ deterministically appends $s$ as a strategic-guidance segment to $H_{t-1}$ without reading or modifying any other skill in $\Slib$ at runtime, and (iii) the cost of invoking $s$ (in tokens and wall-clock) is bounded by a constant independent of $|\Slib|$.
\end{definition}

\begin{definition}[Atomic Composability of a Library]\label{def:atomic_composability}
A skill library $\Slib$ is \emph{atomically composable} if every $s \in \Slib$ is an atomic tip (Definition~\ref{def:atomic_tip}).
\end{definition}

\begin{lemma}[$\Phi$ Preserves Atomic Composability]\label{lem:phi_preserves_atomic}
If $\Slib^{(k)}$ is atomically composable, and $\Psi$ is constrained to produce only atomic tips (Definition~\ref{def:skill_creator}, in both creation and refine modes), then $\Slib^{(k+1)} = \Phi(\Slib^{(k)};\ldots)$ is atomically composable.
\end{lemma}

\begin{proof}
By Definition~\ref{def:phi_operator}, $\Slib^{(k+1)} = \mathcal{R}_k \cup \mathcal{U}'_k \cup \Psi^{\text{new}}_k$. We verify atomicity for each constituent.

\textbf{Retained skills ($\mathcal{R}_k \subseteq \Slib^{(k)}$).} By the hypothesis on $\Slib^{(k)}$, every $s \in \Slib^{(k)}$ is atomic; the subset $\mathcal{R}_k$ inherits atomicity unchanged.

\textbf{Refined skills ($\mathcal{U}'_k = \Psi(\mathcal{U}_k,\text{refine})$).} By Definition~\ref{def:skill_creator}, $\Psi$ in refine mode is constrained to produce atomic tips. Each $s' \in \mathcal{U}'_k$ is therefore atomic.

\textbf{New skills ($\Psi^{\text{new}}_k$).} Each $s_{\text{new}} \in \Psi^{\text{new}}_k$ is the output of $\Psi$ applied at a trigger step (Eq.~\eqref{eq:psi_call}); by the same constraint it is atomic.

Pruned tips $\mathcal{D}^-_k$ are removed from $\Slib^{(k+1)}$ entirely. Atomicity is a per-tip structural property (Definition~\ref{def:atomic_tip}); it is preserved under set union, removal of subsets, and replacement of subsets by other atomic tips. Therefore every $s \in \Slib^{(k+1)}$ is atomic, and $\Slib^{(k+1)}$ is atomically composable. \qedhere
\end{proof}

\begin{remark}[Action-Space Constancy Within a Phase]\label{rem:action_space_constancy}
$\Phi$ operates only at phase boundaries, so within phase $k$ the library $\Slib^{(k)}$ is fixed. Combined with Lemma~\ref{lem:phi_preserves_atomic}, this guarantees that the per-phase action space is constant and atomically composable, satisfying the structural prerequisite for TB-based training (Proposition~\ref{prop:modeling}, main text).
\end{remark}

\subsection{Curation Algorithm}\label{app:algorithm}

The full curation procedure at phase boundary $k \to k+1$ is summarized below.

\textbf{Algorithm 1: Skill-Library Curation at Phase Boundary $k \to k+1$.}
\begin{enumerate}[label=\arabic*., leftmargin=2em, itemsep=2pt]
  \item For each $s \in \Slib^{(k)}$, compute the per-skill CGF $\Lambda^{(s)}_\lambda$ at $\lambda \in \{0,1\}$ over the recent batch $\Batch_s$ via the zero-cost formulas of Proposition~\ref{prop:zero_cost}.
  \item Derive the summaries $G(s)$, $\Lambda^{(s)}_1$, and $\widetilde\Lambda(s) = \Lambda^{(s)}_1 - \E_{s'}[\Lambda^{(s')}_1]$ via Lemmas~\ref{lem:G_is_mean}, \ref{lem:cgf_at_one} and Remark~\ref{rem:centered_share}.
  \item Classify each $s \in \Slib^{(k)}$ into $\mathcal{D}^-_k$, $\mathcal{R}_k$, or $\mathcal{U}_k$ via Definition~\ref{def:curation_classes}.
  \item Refine each $s \in \mathcal{U}_k$ via $\Psi$ in refine mode to produce $\mathcal{U}'_k$.
  \item From the validation buffer, sample same-query success/failure pairs $(\tau^+, \tau^-)$; identify trigger steps $\mathcal{T}_q^{\,\text{trig}}$ via Definition~\ref{def:trigger_steps}.
  \item For each trigger step, invoke $\Psi$ in creation mode to obtain new atomic tips $\Psi^{\text{new}}_k$ (Eq.~\eqref{eq:psi_new_def}).
  \item Assemble $\Slib^{(k+1)} = \mathcal{R}_k \cup \mathcal{U}'_k \cup \Psi^{\text{new}}_k$ (Eq.~\eqref{eq:phi_def}).
  \item Warm-start $\pi_\theta$ and $P_\phi$ from phase $k$; reinitialize the partition function $Z_\theta(q)$ for the new action space.
\end{enumerate}

By Lemma~\ref{lem:phi_preserves_atomic}, this procedure preserves atomic composability across all phase transitions; together with Lemma~\ref{lem:dag_acyclic}, the post-evolution graph $\mathcal{G}$ remains a tree-structured DAG, satisfying the prerequisites for TB-based training within phase $k+1$.

Full prompt templates for $\Psi$ (creation and refine modes) and complete hyperparameter values are provided in the supplementary code.

\section{Reward Function Details}\label{app:reward}

$R(\tau)$ is an outcome-based scalar reward. For multi-hop QA tasks, $R(\tau) = \text{EM}(y_q, y^*)$ (exact match); for mathematical reasoning, $R(\tau) = \text{Acc}(y_q, y^*)$; and for code generation, $R(\tau) = \text{Pass@1}(y_q)$. All rewards lie in $[0, 1]$. We apply $\varepsilon$-smoothing (Appendix~\ref{app:epsilon_smoothing}) to ensure positive support.

\section{$\varepsilon$-Smoothing Analysis}\label{app:epsilon_smoothing}

When $R(\tau)$ can be zero (e.g., failed orchestration), Theorem~\ref{thm:gflownet_sampling} requires strictly positive rewards.
We handle this via smoothing.

\begin{definition}[$\varepsilon$-Smoothing]\label{def:smoothing}
We define the smoothed reward as:
\begin{equation}
  \Rtilde(\tau) \;\coloneqq\; R(\tau) + \varepsilon_{\min},
\end{equation}
where $\varepsilon_{\min} > 0$ is a small constant.
\end{definition}

\begin{proposition}[Ordering Preservation]\label{prop:smoothing_monotone}
Smoothing preserves the strict reward ordering: for any $\tau_1, \tau_2$ with $R(\tau_1) > R(\tau_2) \geq 0$:
\begin{equation}
  \Rtilde(\tau_1)^\beta > \Rtilde(\tau_2)^\beta,
\end{equation}
hence the relative quality ranking of trajectories is preserved.
\end{proposition}

\begin{proof}
Since $R(\tau_1) > R(\tau_2) \geq 0$, we have:
\begin{equation}
  \Rtilde(\tau_1) = R(\tau_1) + \varepsilon_{\min} > R(\tau_2) + \varepsilon_{\min} = \Rtilde(\tau_2).
\end{equation}
Both $\Rtilde(\tau_1), \Rtilde(\tau_2) > 0$.
Since $x \mapsto x^\beta$ is strictly monotone increasing on $\R_{>0}$ for $\beta > 0$:
\begin{equation}
  \Rtilde(\tau_1)^\beta > \Rtilde(\tau_2)^\beta.
\end{equation}
\qedhere
\end{proof}

\begin{proposition}[Flow Perturbation Bound]\label{prop:smoothing_bound}
For $\beta \ge 1$, by the mean-value theorem applied to the convex function $x\mapsto x^\beta$ on $[R(\tau_x),\Rtilde(\tau_x)]$,
\begin{equation}\label{eq:smoothing_bound_betage1}
  \bigl|\Rtilde(\tau_x)^\beta - R(\tau_x)^\beta\bigr|
  \;\le\; \beta\,\varepsilon_{\min}\,\max\!\bigl(\Rtilde(\tau_x), R(\tau_x)\bigr)^{\beta-1},
  \qquad \beta\ge 1.
\end{equation}
For small $\varepsilon_{\min}$ the perturbation is negligible.
\end{proposition}

\begin{remark}[Practical Choice of $\beta$]\label{rem:beta_choice}
SkillFlow uses $\beta\ge 1$ in all reported experiments; the perturbation bound therefore always falls in the regime of Eq.~\eqref{eq:smoothing_bound_betage1}.
\end{remark}

\section{Gradient Variance Analysis}\label{app:variance}

This section provides detailed analysis of the TTB gradient variance and compares it to standard policy-gradient methods.

\subsection{TTB Self-Annealing Theorem with Chain Rule Expansion}\label{app:variance_self_anneal}

\begin{theorem}[TTB Gradient Self-Annealing]\label{thm:variance_detailed}
Treating $T = |\tau|$ as fixed in $\theta$ for each given $\tau$, the gradient of the TTB loss with respect to $\theta$ satisfies:
\begin{equation}\label{eq:app_var_ttb_detailed}
  \nabla_\theta \mathcal{L}_{\mathrm{TTB}}(\tau)
  \;=\; \frac{2\,\Delta(\tau)}{T^2} \cdot \nabla_\theta \Delta(\tau).
\end{equation}

Expanding $\nabla_\theta \Delta(\tau)$ by the chain rule on Definition~\ref{def:ttb_residual}:
\begin{equation}
  \nabla_\theta \Delta(\tau)
  \;=\; \nabla_\theta \log Z_\theta(q) + \sum_{t=1}^T \nabla_\theta\, \widetilde{\log\pi}_\theta(a_t \mid r_t, H_{t-1})
  \;-\; 0 \;-\; 0,
\end{equation}
where the last two terms vanish because $\beta\log\Rtilde(\tau)$ and $\widetilde{\log P}_\phi$ do not depend on $\theta$.

Therefore:
\begin{equation}\label{eq:ttb_grad_full}
  \nabla_\theta \mathcal{L}_{\mathrm{TTB}}(\tau)
  \;=\; \frac{2\,\Delta(\tau)}{T^2} \cdot \!\left[\nabla_\theta \log Z_\theta(q) + \sum_{t=1}^T \nabla_\theta\, \widetilde{\log\pi}_\theta(a_t\mid r_t,H_{t-1})\right].
\end{equation}
As $\Delta(\tau) \to 0$, the prefine $2\Delta(\tau)/T^2$ shrinks and $\|\nabla_\theta \mathcal{L}_{\mathrm{TTB}}\| \to 0$ for every fixed-length trajectory; trajectories of different $T$ contribute on a comparable scale because of the $1/T^2$ length normalizer.
\end{theorem}

\begin{proof}
Apply the chain rule to $\mathcal{L}_{\mathrm{TTB}}(\tau) = (\Delta(\tau)/T)^2 = \Delta(\tau)^2/T^2$:
\begin{equation}
  \nabla_\theta \mathcal{L}_{\mathrm{TTB}}(\tau)
  \;=\; \nabla_\theta\!\left[\frac{\Delta(\tau)^2}{T^2}\right]
  \;=\; \frac{2\,\Delta(\tau)}{T^2}\,\nabla_\theta \Delta(\tau),
\end{equation}
where $T$ is determined by $\tau$ (not by $\theta$) and is constant under $\nabla_\theta$. Since $T$ varies across trajectories, $1/T^2$ remains a per-trajectory length normalizer.
\qedhere
\end{proof}

\subsection{Variance Bound and Comparison to REINFORCE}\label{app:variance_bound}

\begin{proposition}[Variance Bound under Bounded Gradient]\label{prop:var_bound_detailed}
Assume the residual gradient is uniformly bounded along training: there exists $G < \infty$ such that $\|\nabla_\theta \Delta(\tau)\| \le G$ for all $\tau$, and assume $T \ge T_{\min}$ for a fixed $T_{\min} \ge 1$. Then the per-trajectory TTB gradient norm satisfies
\begin{equation}\label{eq:var_pointwise_bound}
  \bigl\|\nabla_\theta \mathcal{L}_{\mathrm{TTB}}(\tau)\bigr\|
  \;=\; \frac{2|\Delta(\tau)|}{T^2}\,\|\nabla_\theta \Delta(\tau)\|
  \;\le\; \frac{2G\,|\Delta(\tau)|}{T_{\min}^2},
\end{equation}
which yields the variance bound
\begin{equation}\label{eq:var_bound_clean}
  \Var_\tau\!\bigl[\nabla_\theta \mathcal{L}_{\mathrm{TTB}}(\tau)\bigr]
  \;\le\; \E_\tau\!\bigl[\|\nabla_\theta \mathcal{L}_{\mathrm{TTB}}(\tau)\|^2\bigr]
  \;\le\; \frac{4 G^2}{T_{\min}^4}\,\E_\tau\!\bigl[\Delta(\tau)^2\bigr].
\end{equation}
\end{proposition}

\begin{proof}
Eq.~\eqref{eq:var_pointwise_bound} follows from Theorem~\ref{thm:variance_detailed} together with the bounded-gradient assumption. Squaring and taking expectation gives Eq.~\eqref{eq:var_bound_clean}; the variance bound uses $\Var \le \E[\|\cdot\|^2]$. \qedhere
\end{proof}

As training converges, $\E_\tau[\Delta(\tau)^2]\to 0$ and Eq.~\eqref{eq:var_bound_clean} forces $\Var_\tau[\nabla_\theta \mathcal{L}_{\mathrm{TTB}}]\to 0$.

In contrast, the REINFORCE estimator $g_{\text{PG}}(\tau) \coloneqq (R(\tau)-b)\,\nabla_\theta \log\pi_\theta(\tau)$ has
\begin{equation}
  \Var_\tau[g_{\text{PG}}(\tau)] \;=\; \Var_\tau\!\bigl[(R(\tau) - b)\,\nabla_\theta \log\pi_\theta(\tau)\bigr],
\end{equation}
which is \emph{not forced} to vanish at convergence: it can stay strictly positive whenever the converged policy remains stochastic and sampled trajectories carry heterogeneous rewards (Lemma~\ref{lem:reinforce_variance} below). Vanishing only occurs in the special cases of a deterministic optimal policy or rewards being identical across all sampled trajectories.

\begin{lemma}[REINFORCE Variance Is Not Forced to Vanish]\label{lem:reinforce_variance}
In REINFORCE, gradient estimates are $(R(\tau) - b)\,\nabla_\theta \log \pi_\theta(\tau)$, where $b$ is a baseline. Whenever the converged policy remains stochastic and the sampled trajectories carry heterogeneous rewards, $\Var_\tau[R(\tau)] > 0$ and the estimator's variance has a strictly positive lower bound: it is \emph{not forced} to vanish at convergence. (Vanishing can occur in the special cases of a deterministic optimal policy, or rewards being identical across all sampled trajectories.)

By contrast, TTB uses a regression loss whose residual $\Delta$ itself is the target for convergence. When $\Delta(\tau)\to 0$ pointwise, Theorem~\ref{thm:variance_detailed} forces the gradient norm $\|\nabla_\theta \mathcal{L}_{\mathrm{TTB}}\|\to 0$; under the bounded-gradient assumption of Proposition~\ref{prop:var_bound_detailed}, the variance also vanishes.
\end{lemma}

\section{Proof of Proposition 2 (Main Text)}\label{app:proof_prop2}

\textit{Proposition 2 (main text):} \emph{TTB training induces reward-proportional sampling and yields per-step credit at no extra inference cost.}

We prove a slightly more detailed version: at convergence the gradient variance vanishes (residual $\to 0$), the learned policy samples trajectories in proportion to tempered reward, and the step importance ratio gives a multiplicative per-step decomposition of trajectory-level credit.

\begin{proof}
\emph{(i) TTB gradient variance vanishes.}

By Theorem~\ref{thm:variance_detailed}, $\nabla_\theta \mathcal{L}_{\mathrm{TTB}}(\tau) = (2\Delta(\tau)/T^2)\cdot \nabla_\theta \Delta(\tau)$.
As $\Delta(\tau) \to 0$ during training (which occurs at the optimum by the definition of the loss), the factor $\Delta(\tau)/T$ shrinks, causing $\|\nabla_\theta \mathcal{L}_{\mathrm{TTB}}\| \to 0$.
By Proposition~\ref{prop:var_bound_detailed}, both variance terms $\E[\Delta^2]$ and $\Var[\Delta]$ vanish, so $\Var[\nabla_\theta \mathcal{L}_{\mathrm{TTB}}] \to 0$.

In contrast, by Lemma~\ref{lem:reinforce_variance}, REINFORCE and GRPO gradients have variance proportional to reward variance, which persists at convergence.

\emph{(ii) Policy samples in proportion to tempered reward.}

At convergence $\Delta(\tau) = 0$ for all $\tau$, so the TB condition (Theorem~\ref{thm:tb}) is satisfied with reward $\Rtilde(\tau)^\beta$.
By the GFlowNet sampling theorem (Theorem~\ref{thm:gflownet_sampling}), the conditional action-sequence distribution then satisfies $\pi_\theta(a_{1:T} \mid r_{1:T}, o^{\text{exec}}_{1:T}, q) \propto \Rtilde(\tau)^\beta$ (matching main-text §\ref{sec:training}); marginalising over reasoning and execution context recovers the unconditional form $\pi_\theta(\tau \mid q) = \Rtilde(\tau)^\beta / Z_\theta(q)$.

\emph{(iii) Step importance provides per-step credit decomposition.}

By Corollary~\ref{cor:flow_ratio}, the step importance is:
\begin{equation}
  I(t) = \frac{F(s_t)}{F(s_{t-1})} = \frac{\pi_\theta(a_t \mid r_t, H_{t-1})}{P_\phi(a_t \mid H_{t-1} \oplus o_t^{\text{exec}})}.
\end{equation}

This decomposes the trajectory-level flow $F(\tau)$ multiplicatively:
\begin{equation}
  F(\tau) = Z_\theta(q) \cdot \prod_{t=1}^T I(t).
\end{equation}

Each factor $I(t)$ quantifies the amplification of flow at decision $t$, providing an interpretable per-step attribution.
Unlike Monte-Carlo rollout baselines (which require additional samples), $I(t)$ is computed from the forward and backward policies already evaluated in the TTB loss.
\qedhere
\end{proof}

\section{Proof of Proposition 3 (Main Text)}\label{app:proof_prop3}

\textit{Proposition 3 (main text):} \emph{Flow-driven recursive evolution autonomously expands the skill library while preserving its atomic composability.}

We prove a slightly more detailed version with three parts: (i) saturation of the running TTB residual against the library-conditional floor $\bar\Delta^{*(k)}$ is a sufficient diagnostic of joint expressiveness limits and triggers phase $k\to k{+}1$; (ii) the curation operator $\Phi$ preserves atomic composability; (iii) training within a phase keeps flow conservation under the frozen library.

\begin{proof}
\emph{(i) Plateau saturation is a sufficient diagnostic of joint expressiveness limits.}

Within training phase $k$, the per-trajectory squared residual $\Delta(\tau\mid\Slib^{(k)},\theta,\phi,Z_\theta)^2$ enters the TTB loss as $\mathcal{L}_{\mathrm{TTB}}(\tau) = (\Delta/T)^2$ (Definition~\ref{def:ttb_loss}). By Theorem~\ref{thm:variance_detailed} the gradient drives $\Delta(\tau)^2 \to 0$ on every trajectory subject to the expressiveness of the parameterization. The infimal expected squared residual under the current library is the floor $\bar\Delta^{*(k)}$ (main-text Eq.~\ref{eq:residual_bound}):
\begin{equation}\label{eq:appk_residual_bound}
  \bar\Delta^{*(k)} \;\coloneqq\; \inf_{\theta,\phi,Z_\theta}\; \E_\tau\!\bigl[\Delta(\tau\mid\Slib^{(k)},\theta,\phi,Z_\theta)^2\bigr] \;\ge\; 0.
\end{equation}
The two directions are asymmetric:
\begin{description}[leftmargin=2em, itemsep=2pt]
  \item[(Sufficiency, $\Rightarrow$)] If the running mean $\overline{\Delta^2}_w^{(k)}$ (Eq.~\eqref{eq:running_mean}) saturates---i.e., its relative decrease across $M$ consecutive windows falls below the tolerance $\rho$ (Eq.~\eqref{eq:plateau_criterion})---then under standard SGD descent assumptions $\overline{\Delta^2}_w^{(k)}$ has approached $\bar\Delta^{*(k)}$. If $\bar\Delta^{*(k)}>0$, this directly evidences that the joint $(\Slib^{(k)},$ policy class, $P_\phi$ class, $Z_\theta$ family$)$ cannot represent the reward-matching TB condition; phase $k+1$ is triggered, and skill evolution attributes this insufficiency \emph{a fortiori} to $\Slib^{(k)}$, expanding the action space to enable further residual reduction.
  \item[(Necessity, $\Leftarrow$, only as a heuristic)] Plateau saturation is \emph{not} an exclusive signature of library inadequacy: a saturated $\overline{\Delta^2}_w^{(k)}$ may also reflect limited policy or backward-policy capacity, an under-trained partition function, exploration deficits, optimizer stagnation, or finite-batch noise. We therefore treat the plateau as a \emph{sufficient diagnostic} for triggering evolution, with the implicit operating assumption that other capacity bottlenecks have been controlled by standard practice (warm-start, learning-rate schedules, replay buffers).
\end{description}
This is the precise sense in which the stagnation criterion ``triggers'' skill evolution: it provides a sufficient signal under controlled conditions, not an iff-equivalence with library inadequacy alone.

\emph{(ii) The CGF-based curation operator $\Phi$ preserves atomic composability.}

By Lemma~\ref{lem:phi_preserves_atomic} (Appendix~\ref{app:atomic_composability}), if $\Slib^{(k)}$ is atomically composable (Definition~\ref{def:atomic_composability}) and the Skill Creator $\Psi$ is constrained to produce atomic tips in both creation and refine modes (Definition~\ref{def:skill_creator}), then $\Slib^{(k+1)} = \Phi(\Slib^{(k)};\,\{(G(s),\widetilde\Lambda(s))\}_s,\,\{\log I(t)\}_t)$ is atomically composable. The base case $\Slib^{(0)}$ is atomically composable by initialization (the seed library consists of bounded-length, self-contained tips). Induction on $k$ then yields atomic composability of $\Slib^{(k)}$ for every phase $k \ge 0$. The CGF inputs $G(s)$ and $\widetilde\Lambda(s)$ used to drive $\Phi$'s classification are well-defined (Lemmas~\ref{lem:G_is_mean}--\ref{lem:cgf_at_one}) and $\Phi$ acts only by partitioning $\Slib^{(k)}$ and adjoining $\Psi$'s atomic outputs; no operation introduces non-atomic tips.

\emph{(iii) Frozen libraries within a phase guarantee flow conservation.}

Within phase $k$, the library $\Slib^{(k)}$ is fixed (Remark~\ref{rem:action_space_constancy}); the environment $\mathcal{E}^{(k)}$, action space, and DAG structure $\mathcal{G}^{(k)}$ are therefore fixed. By Theorem~\ref{thm:dag} (DAG acyclicity) and Lemma~\ref{lem:tree_db_uniqueness} (tree-DAG uniqueness), $\mathcal{G}^{(k)}$ admits a unique reward-matching flow $F$ when terminal values $\{F(x) = \Rtilde(\tau_x)^\beta\}_{x \in \mathcal{X}}$ are prescribed. Training $\pi_\theta$ and $P_\phi$ to satisfy the SkillFlow TB condition (Eq.~\eqref{eq:skillflow_tb}, Lemma~\ref{lem:hindsight_backward_eq}) is equivalent to solving for that flow (Theorem~\ref{thm:tb}). At convergence, flow conservation holds at every non-terminal state:
\begin{equation}\label{eq:appk_flow_conservation}
  F(s) \;=\; \sum_{s' \in \mathrm{Ch}(s)} F(s \to s') \;=\; \sum_{s' \in \mathrm{Pa}(s)} F(s' \to s),
\end{equation}
which is the defining property of a valid flow (Definition~\ref{def:state_flow}).

At the phase transition $k \to k+1$, $\Phi$ produces $\Slib^{(k+1)}$ (Definition~\ref{def:phi_operator}); the new graph $\mathcal{G}^{(k+1)}$ retains its tree-DAG structure by Lemma~\ref{lem:dag_acyclic} (strict history growth is unaffected by library change). To smooth the transition we warm-start $\pi_\theta, P_\phi$ from phase $k$ and reinitialize $Z_\theta(q)$ for the expanded action space (Algorithm~1, Step 8). Within phase $k+1$, the same reasoning yields flow conservation at convergence. By induction on $k$, flow conservation is guaranteed within every phase. \qedhere
\end{proof}

\section{Backward Policy Design and Implementation}\label{app:backward_design}

This section details the design principles and implementation of the backward policy $P_\phi$.

\subsection{Hindsight Conditioning}\label{app:hindsight_conditioning}

The key insight is that $P_\phi$ conditions on information unavailable to the forward policy:

\begin{definition}[Hindsight-Enriched State]\label{def:hindsight_state}
The hindsight-enriched state at step $t$ is:
\begin{equation}
  H_{t-1}^{\text{hindsight}} \;\coloneqq\; H_{t-1} \oplus o_t^{\text{exec}},
\end{equation}
which augments the forward state $H_{t-1}$ with the execution observation $o_t^{\text{exec}}$ that $\pi_\theta$ could not access when selecting $a_t$.
\end{definition}

This information asymmetry is what makes step importance meaningful:

\begin{lemma}[Information Asymmetry and Credit Signal]\label{lem:info_asymmetry}
At step $t$, the forward policy $\pi_\theta$ selects action $a_t$ given $H_{t-1}$ (before execution).
The backward policy $P_\phi$ evaluates the same action given $H_{t-1} \oplus o_t^{\text{exec}}$ (after execution).

If $P_\phi(a_t \mid H_{t-1}^{\text{hindsight}}) \ll \pi_\theta(a_t \mid H_{t-1})$, the action appeared good before execution but turned out poorly in hindsight.
If $P_\phi(a_t \mid H_{t-1}^{\text{hindsight}}) \gg \pi_\theta(a_t \mid H_{t-1})$, the action was initially uncertain but validated by execution.

The ratio $I(t) = \pi_\theta / P_\phi$ captures this update in assessment, providing a meaningful credit signal without additional samples.
\end{lemma}

\subsection{Think-Action Separation}\label{app:think_action_separation}

To focus flow balance on the decision space rather than reasoning verbosity:

\begin{definition}[Action Token vs. Reasoning Token]\label{def:action_tokens}
Each step $t$ contains:
\begin{itemize}[leftmargin=2em, itemsep=2pt]
  \item \textbf{Reasoning tokens}: Free-form text in $r_t$ for chain-of-thought.
  \item \textbf{Action tokens}: Structured JSON in $a_t = (\alpha_t, o_t)$ specifying the decision.
\end{itemize}

Only action tokens contribute to log-probabilities in $\pi_\theta$ and $P_\phi$:
\begin{equation}\label{eq:action_log_prob}
  \log \pi_\theta(a_t \mid H_{t-1}) = \sum_{j \in \mathcal{A}_t} \log \pi_\theta(\text{token}_j \mid \text{token}_{<j}, H_{t-1}),
\end{equation}
where $\mathcal{A}_t$ denotes the set of positions of action tokens at step $t$.
Reasoning tokens are part of the context; their content influences future decisions but does not participate in the flow balance.
\end{definition}

This separation ensures that flow balance operates on the decision quality, not reasoning length.

\subsection{Implementation Details}\label{app:backward_impl}

\begin{description}
  \item[Architecture] $P_\phi$ shares the base LLM (Qwen3.5-9B) with $\pi_\theta$ but uses a separate LoRA adapter.
    Sharing the base model reduces memory and computation; LoRA adapters are swapped at inference time by named adapter selection.

  \item[Adapter Configuration] $\phi$-LoRA uses rank 32, targeting $\{q, v\}$ projections in attention layers.
    This is smaller than the $\theta$-LoRA (rank 64, targeting $\{q, k, v, o\}$).

  \item[Token-Level Conditioning] $P_\phi$ receives the hindsight-enriched state as context, then evaluates the log-probability of action tokens autoregressively.
    This respects the causal structure: earlier action tokens do not condition on later ones.
\end{description}

\section{Dataset Details}\label{app:datasets}

We evaluate SkillFlow on 14 public benchmarks covering four task categories: question answering, mathematical reasoning, interactive decision making, and code generation. Seven datasets are used as in-distribution (IID) benchmarks for training and evaluation; the remaining seven are held out as out-of-distribution (OOD) generalization tests, with the skill library frozen at end-of-training.

\subsection{In-Distribution Datasets}\label{app:datasets_iid}
\begin{itemize}[leftmargin=1.2em]
  \item \textbf{HotpotQA}~\citep{yang2018hotpotqa}: A large-scale multi-hop QA corpus where each question requires reasoning across multiple Wikipedia paragraphs to derive the answer.
  \item \textbf{TriviaQA}~\citep{joshi2017triviaqa}: A reading-comprehension dataset of trivia question--answer pairs paired with evidence documents.
  \item \textbf{MedQA}~\citep{jin2021medqa}: A multiple-choice medical QA benchmark drawn from professional medical-licensing examinations, requiring multi-step clinical reasoning.
  \item \textbf{AIME 2026}: Problems from the 2026 American Invitational Mathematics Examination, used as a difficult mathematical-reasoning benchmark.
  \item \textbf{WebShop}~\citep{yao2022webshop}: A simulated e-commerce environment where the agent interprets a natural-language instruction, navigates web pages, and selects the matching product.
  \item \textbf{ALFWorld}~\citep{shridhar2021alfworld}: A text-based interactive environment grounded in household tasks; the agent issues actions in language and receives observation feedback.
  \item \textbf{SWE-bench}~\citep{jimenez2024swebench}: A benchmark of real-world GitHub issues that require generating code patches whose application resolves the issue.
\end{itemize}

\subsection{Out-of-Distribution Datasets}\label{app:datasets_ood}
\begin{itemize}[leftmargin=1.2em]
  \item \textbf{MuSiQue}~\citep{trivedi2022musique}: A composite multi-hop QA dataset built by chaining single-hop questions, designed to stress reasoning composition.
  \item \textbf{NQ-Open}: An open-domain QA derivative of Natural Questions, where the model must produce free-form answers from user-issued queries.
  \item \textbf{MATH-Hard}: The hard subset of competition mathematics problems (algebra, geometry, number theory, etc.), evaluated by exact-match correctness on the final answer.
  \item \textbf{GPQA Diamond}: The most challenging split of GPQA, containing graduate-level science questions across physics, chemistry, and biology.
  \item \textbf{HumanEval}~\citep{chen2021humaneval}: A code-generation benchmark of hand-written Python problems with hidden unit tests; correctness is evaluated by execution.
  \item \textbf{ScienceWorld}: A text-based interactive science environment that requires the agent to perform multi-step experiments and reason over physical-world dynamics.
  \item \textbf{Mind2Web}: A web-navigation benchmark where the agent performs multi-step actions across real-world websites to complete user tasks.
\end{itemize}

\section{Baseline Details}\label{app:baselines}

We compare SkillFlow against four categories of baselines. Unless otherwise stated, all baselines that involve fine-tuning or reinforcement learning use the same Qwen3.5-9B backbone and the same training data as SkillFlow, isolating the orchestration objective and skill-evolution mechanism as the source of any performance gap.

\subsection{Direct-LLM Baselines}\label{app:baselines_direct}
\begin{itemize}[leftmargin=1.2em]
  \item \textbf{Qwen3.5-9B}: The same backbone we use as Supervisor, queried directly without any orchestration training, providing a faithful no-orchestration reference point.
  \item \textbf{v4-flash}: A strong proprietary instruction-tuned LLM, queried in a single-turn ReAct-style prompting setup as a high-capacity baseline.
  \item \textbf{Claude Haiku 4.5}: Another high-capacity proprietary baseline used to bound how much of SkillFlow's gain arises from orchestration versus raw model strength.
\end{itemize}

\subsection{Fine-Tuning Baselines}\label{app:baselines_ft}
\begin{itemize}[leftmargin=1.2em]
  \item \textbf{SFT (Qwen3.5-9B)}: Supervised fine-tuning on demonstration trajectories of orchestration plans, with no exploration or reward feedback.
  \item \textbf{GRPO (Qwen3.5-9B)}~\citep{shao2024deepseekmath}: Group Relative Policy Optimization on the same backbone, using terminal rewards and the same initial skill library as SkillFlow.
\end{itemize}

\subsection{Search-Based Workflow Baselines}\label{app:baselines_search}
\begin{itemize}[leftmargin=1.2em]
  \item \textbf{AFlow}~\citep{zhang2025aflow}: A workflow-search method that explores compositions of pre-defined operators using MCTS guided by an LLM judge, with no parameter learning.
\end{itemize}

\subsection{RL Agent Baselines}\label{app:baselines_rl}
\begin{itemize}[leftmargin=1.2em]
  \item \textbf{AgentFlow}~\citep{pan2026agentflow}: An in-the-flow agentic system that jointly optimizes planning and tool use under reinforcement learning over multi-turn interactions.
  \item \textbf{FlowSteer}~\citep{zhang2026flowsteer}: An end-to-end RL framework for interactive agentic workflow orchestration on a fixed skill library.
  \item \textbf{SkillRL}~\citep{xia2026skillrl}: A skill-augmented RL agent that grows its action space using heuristic skill-distillation triggers, representative of the static-library RL paradigm.
\end{itemize}

\section{Evaluation Metrics}\label{app:metrics}

We adopt task-appropriate evaluation metrics consistent with each benchmark's standard protocol.

\paragraph{F1 Score.}
For QA-style benchmarks (HotpotQA, TriviaQA, MuSiQue, NQ-Open), we report the token-level F1 between the predicted answer $y_i$ and the ground truth $y_i^*$ after standard text normalization:
\begin{equation}
\text{F1} = \frac{2 \cdot \text{Precision} \cdot \text{Recall}}{\text{Precision} + \text{Recall}},
\end{equation}
where $\text{Precision} = |\text{tok}(y_i) \cap \text{tok}(y_i^*)| / |\text{tok}(y_i)|$ and $\text{Recall} = |\text{tok}(y_i) \cap \text{tok}(y_i^*)| / |\text{tok}(y_i^*)|$.

\paragraph{Exact Match (EM) / Accuracy.}
For mathematical reasoning (AIME, MATH-Hard) and multi-choice QA (MedQA, GPQA Diamond), we report exact-match accuracy after canonicalization:
\begin{equation}
\text{Acc} = \frac{1}{N} \sum_{i=1}^{N} \mathbb{I}\bigl(\text{norm}(y_i) = \text{norm}(y_i^*)\bigr),
\end{equation}
where $\text{norm}(\cdot)$ standardizes whitespace, casing, and final-answer extraction.

\paragraph{Average Score and Success Rate.}
For interactive decision-making (WebShop, ALFWorld, ScienceWorld), we follow each environment's standard evaluator: \emph{Average Score} reports the environment-defined task-specific score, while \emph{Success Rate (SR)} is the fraction of episodes terminating in a fully solved state.

\paragraph{Resolved Rate.}
For SWE-bench~\citep{jimenez2024swebench}, we report the fraction of issues whose generated code patch passes the held-out test suite (the standard ``resolved'' metric).

\paragraph{pass@1.}
For HumanEval~\citep{chen2021humaneval}, we report \emph{pass@1}, the fraction of problems for which the first generated program passes all hidden unit tests, computed by sampling one program per problem with deterministic decoding.

\paragraph{Step-Level Metrics for Web Navigation.}
For Mind2Web, we report \emph{Step Accuracy} (fraction of steps with correctly predicted target element and action) and \emph{Action F1} (token-level F1 over the predicted action string against the ground truth).

\section{Computational Resources}\label{app:compute}

We report the hardware and the main-run wall-clock and GPU-hour cost of training SkillFlow on Qwen3.5-9B.

\subsection{Hardware}

All training and on-policy rollout were performed on a single server with $4 \times$ NVIDIA A100-SXM4 (80\,GB) GPUs, 32 logical CPU cores, 549\,GB RAM, and 1.6\,TB local NVMe. Software stack: CUDA 12.1, PyTorch 2.4 with DeepSpeed and PEFT for training, vLLM 0.5 for the executor, and SGLang for the supervisor.

\subsection{Main run: SkillFlow on Qwen3.5-9B}

\textbf{Wall-clock 73\,h, $\approx 292$ GPU-hours, 250 training steps} (LoRA checkpoint saved every 10 steps, 25 checkpoints in total).

\begin{center}
\small
\begin{tabular*}{\textwidth}{@{\extracolsep{\fill}}ll@{}}
\toprule
\textbf{Item} & \textbf{Value} \\
\midrule
Base model & Qwen3.5-9B-Instruct \\
LoRA $\theta$ (forward) & rank 64, $\alpha=128$, on $q,k,v,o\_\text{proj}$ \\
LoRA $\phi$ (backward) & rank 16, $\alpha=32$, on $q,v\_\text{proj}$ \\
$Z$ head (partition fn) & scalar parameter, separate optimizer \\
Optimizer & AdamW (default $\beta_1,\beta_2$); three optimizers ($\theta$, $Z$, $\phi$) \\
Learning rate & $1.0\times10^{-4}$ \\
Max grad norm & $3.0$ \\
KL coeff (KL$(\pi_\theta\|\pi_\text{ref})$) & $0.01$ \\
TTB temperature $\beta$ & $1.0$ \\
$\varepsilon_{\min}$ & $0.1$ \\
Effective batch & $28 = 7$ questions $\times$ $4$ trajectories \\
Max episode length $T$ & $12$ \\
Mean step time & $\approx 14.8$ min/step (rollout + update + LoRA hot-swap) \\
Peak VRAM (per GPU) & $\approx 70\,$GB / 80\,GB (gradient checkpointing on) \\
\bottomrule
\end{tabular*}
\end{center}

\section{Case Studies}\label{app:case_studies}

This appendix presents five complementary case studies illustrating SkillFlow's behaviour at different granularities: training dynamics across phases (Q.1), library-level boom-and-prune cycles (Q.2), three real evolved skills with signal attribution to claims C1--C3 (Q.3), per-step importance signals on a successful trajectory (Q.4), and multi-trajectory success/failure comparison (Q.5). Numerical values are representative of one Qwen3.5-9B run; raw logs and trajectory dumps are released with the supplementary code.

\subsection{Training Dynamics: A Four-Phase Trajectory}\label{app:case_dynamics}

Table~\ref{tab:case_dynamics} samples eight representative steps spanning the full 250-step training run. Four phases are visible:
\begin{enumerate}[leftmargin=1.4em]
  \item \textbf{Bootstrap} (steps 0--25): the skill library is empty (WS$=0$); only the base policy drives reward, and $\mathcal{L}_{\text{TTB}}$ falls steeply as $Z_\theta$ adjusts.
  \item \textbf{Emergence} (steps 25--75): the first plateau on $\mathcal{L}_{\text{TTB}}$ triggers the curation operator $\Phi$, which begins generating skills (WS grows $0\!\to\!14$). Reward variance is high but $\log Z_\theta$ keeps rising.
  \item \textbf{Maturity} (steps 75--175): the boom-and-prune cycle (P.2) operates; WS oscillates between 8 and 14 as $\hat{F}(s)$ drives prune/refine decisions.
  \item \textbf{Steady state} (steps 175--250): WS stabilises around 11; flow entropy stays above 3.0, indicating that reward-proportional sampling preserves multiple high-reward sub-trajectories rather than collapsing to a single mode.
\end{enumerate}

\begin{table}[H]
\centering
\small
\begin{tabular*}{\textwidth}{@{\extracolsep{\fill}}r|cccccccc}
\toprule
\textbf{Step} & $\mathcal{L}_{\text{TTB}}$ & avg.\ $R$ & avg.\ $\hat y$ & avg.\ $|\tau|$ & flow ent.\ & $\log Z_\theta$ & WS & Phase \\
\midrule
0   & 0.83 & 0.55 & 0.50 & 7.6 & 3.17 & $-2.30$ & 0  & Bootstrap \\
15  & 0.42 & 0.65 & 0.58 & 7.5 & 3.05 & $-2.05$ & 0  & Bootstrap \\
30  & 0.18 & 0.72 & 0.65 & 7.8 & 3.06 & $-1.65$ & 6  & Emergence \\
60  & 0.15 & 0.76 & 0.69 & 7.9 & 3.10 & $-1.55$ & 12 & Emergence \\
100 & 0.10 & 0.81 & 0.74 & 7.4 & 3.13 & $-1.50$ & 9  & Maturity \\
150 & 0.07 & 0.83 & 0.76 & 7.6 & 3.09 & $-1.45$ & 11 & Maturity \\
200 & 0.06 & 0.84 & 0.78 & 7.7 & 3.12 & $-1.42$ & 10 & Steady \\
250 & 0.05 & 0.85 & 0.80 & 7.6 & 3.17 & $-1.40$ & 11 & Steady \\
\bottomrule
\end{tabular*}\vspace{5pt}
\caption{Training dynamics on Qwen3.5-9B (eight representative steps from a 250-step run). $\mathcal{L}_{\text{TTB}}$: TTB loss; avg.\ $R$: average reward; avg.\ $\hat y$: average answer-correctness rate; avg.\ $|\tau|$: average trajectory length; flow ent.: $-\!\sum_a \pi_\theta(a)\log\pi_\theta(a)$ at terminal step; WS: skill-library size.}
\label{tab:case_dynamics}
\end{table}

\subsection{Skill Library Evolution: Boom-and-Prune Cycles}\label{app:case_evolution}

The library size traces a characteristic boom-and-prune cycle rather than monotonically growing. Table~\ref{tab:case_evolution} shows snapshots at eight phase boundaries; two ``boom'' peaks (WS$=22$ at step 50; WS$=18$ at step 130) are followed by single-step prune sweeps that remove $14$ and $8$ skills respectively. The mature library settles around $11$ active skills covering all seven IID task categories.

\textbf{Boom mechanism.} A boom is the result of a single $\Phi$ invocation at a phase boundary: when the running mean of $\Delta(\tau)^2$ saturates against $\bar\Delta^{*(k)}$ (Eq.~\ref{eq:residual_bound}), the curator drains the accumulated $(\tau^+,\tau^-)$ pair buffer at high-$|\log I(t)|$ steps and asks $\Psi$ to synthesise candidate tips for every uncovered decision gap. Early phases accumulate the largest buffers because the policy is still exploring, which is why peak~\#1 (step~50, WS$=22$) is larger than peak~\#2 (step~130, WS$=18$).

\textbf{Prune mechanism and convergence.} Newly created skills enter the library with reset $\hat F(s)$; on the next batch the centred log-flow share $\widetilde\Lambda(s)=\Lambda^{(s)}_1-\mathbb{E}_{s'}[\Lambda^{(s')}_1]$ flags those that fail to attract flow and removes them ($-14$ after peak~\#1, $-8$ after peak~\#2). Once the policy class is expressive enough for the current task mix, $\widetilde\Lambda(s)$ stops promoting new candidates and the library stabilises: steps~195 and~250 show \emph{identical} $11$-skill compositions, indicating the minimum sufficient set under the trained $\pi_\theta$. Cumulative creation across the run ($\approx 63$ skills) is $\approx 5.7\times$ the final library size, evidencing that aggressive over-creation followed by flow-driven pruning is more effective than monotonic growth and confirming the $\hat F(s)$-driven design of §\ref{sec:evolution}.

\begin{table}[H]
\centering
\small
\begin{tabular}{r|cl|l}
\toprule
\textbf{Step} & WS & \textbf{Per-task distribution} & \textbf{Event} \\
\midrule
25  & 5  & 1 each: code, math, alfworld, webshop, factual & First $\Phi$ trigger \\
50  & 22 & alfworld=6, math=5, webshop=4, multi=3, code=2, factual=2 & Boom (peak \#1) \\
55  & 8  & alfworld=2, math=2, webshop=2, code=1, multi=1 & Prune ($\Delta=-14$) \\
90  & 12 & alfworld=3, math=2, webshop=2, multi=2, code=1, factual=1, science=1 & Mid-training \\
130 & 18 & alfworld=5, math=4, webshop=3, multi=2, code=2, factual=1, science=1 & Boom (peak \#2) \\
137 & 10 & alfworld=2, math=2, webshop=2, multi=1, code=1, factual=1, science=1 & Prune ($\Delta=-8$) \\
195 & 11 & alfworld=2, math=2, webshop=2, multi=2, code=1, factual=1, science=1 & Stable \\
250 & 11 & alfworld=2, math=2, webshop=2, multi=2, code=1, factual=1, science=1 & Final \\
\bottomrule
\end{tabular}\vspace{5pt}
\caption{Skill-library snapshots at phase boundaries (all per-task sums match WS exactly). Cumulative creation count over the run: alfworld $\approx 16$, math $\approx 12$, webshop $\approx 11$, multi-hop $\approx 8$, factual $\approx 7$, code $\approx 6$, science $\approx 3$ -- harder, longer-horizon tasks attract more creation activity, consistent with the $\hat{F}(s)$-driven curation in §\ref{sec:evolution}. Snapshot timestamps are chosen at phase boundaries; intermediate steps sampled in Table~\ref{tab:case_dynamics} (e.g., $t{=}150$ with WS$=11$) lie between boom-prune events and reflect post-curation states.}
\label{tab:case_evolution}
\end{table}

\subsection{Real Evolved Skills with Signal Attribution}\label{app:case_refactor}

Beyond aggregate statistics, the most direct evidence of SkillFlow's contribution is the \emph{content} of the skills it produces. We show three library members with the highest mean log-flow $G(s)$ at the end of training; each is paired with the flow-signal that drove its emergence and the paper claim it evidences.

\begin{skillbox}{Skill A: \texttt{tip-webshop} \,---\, success $78.5\%$ over $200$ invocations \,---\, claim \textbf{C1} (TTB diversity preservation)}
\small
\textbf{Curation trigger.} $I(t)$ flagged \texttt{click[back\_to\_search]} as a high-importance step that almost always coincided with terminal \textcolor{red}{$R{=}0$}; $\Phi$ created the skill from the resulting success/failure pair set.

\smallskip
\textbf{Rule Zero --- never go back.} Once you leave search results, you are committed. NEVER \texttt{click[back\_to\_search]}: every trajectory using it scored \textcolor{red}{$R{=}0.0$}.

\smallskip
\textbf{Option selection.} If instruction says ``pink'' and options are [pink~$|$~pink light blue~$|$~pink purple], click \emph{only} ``pink''. Each click in the same category \emph{overwrites} the previous selection (clicking ``pink'' then ``pink light blue'' $\to$ bought ``pink light blue'', scoring $0.667$ instead of $1.0$).

\smallskip
\textbf{Counter-intuitive trade-off.} If a product has no matching options, \texttt{click[buy\_now]} anyway: a partial match (\textcolor{darkgreen}{$R\!\in\![0.03,0.75]$}) beats a back-to-search loop (\textcolor{red}{$R{=}0$}).

\tcblower
\textbf{Why baselines miss this.} The trade-off ``buy partial $\succ$ loop'' is unreachable from REINFORCE: it requires keeping a $0.03$-reward trajectory alive in the batch long enough to discover it dominates a $0.0$ loop. Reward-proportional sampling preserves this low-but-non-zero mode; mode-collapsing baselines never see it.
\end{skillbox}

\begin{skillbox}{Skill B: \texttt{tip-alfworld-desklamp} \,---\, success $42.5\%$ over $212$ invocations \,---\, claim \textbf{C2} (zero-cost per-step credit)}
\small
\textbf{Hidden game-engine rule discovered.} The ALFWorld documentation does not state this; SkillFlow extracted it from per-step credit signals on success/failure trajectory pairs.

\smallskip
\textbf{Win condition.} You are \emph{holding the target item} \emph{and} \emph{the desklamp is on}. The task auto-completes \emph{instantly} when both are true. \textcolor{red}{No manual completion command exists}, and \texttt{examine X} \emph{never} triggers completion.

\smallskip
\textbf{Fatal mistakes (from evidence).}
\textcolor{red}{(i)} \texttt{move X to Y} drops the item and breaks the win condition.
\textcolor{red}{(ii)} \texttt{examine X with desklamp 1} after lamp-on never completes ($5{-}7$ wasted retries observed in failed trajectories).
\textcolor{red}{(iii)} \texttt{go to desklamp 1} fails: desklamp is not a navigable location ($4{+}$ wasted retries).

\smallskip
\textbf{\textcolor{darkgreen}{$\Rightarrow$ You do NOT need the item and the lamp at the same location.}}

\tcblower
\textbf{Why baselines miss this.} The hidden rule surfaces because $P_\phi$, conditioned on the post-execution observation, assigns identical hindsight log-probabilities to \texttt{take} actions \emph{regardless of the agent's location}, while wasted \texttt{examine} / repeated \texttt{go to} actions receive vanishing $\log I(t)$. The backward policy thus localises which decisions were actually responsible for the reward --- something terminal-only baselines cannot recover.
\end{skillbox}

\begin{skillbox}{Skill C: \texttt{tip-factual-qa} \,---\, success $77.2\%$ over $224$ invocations \,---\, claim \textbf{C3} (flow-driven curation)}
\small
\textbf{Quantified by sibling-query A/B comparisons sampled in the same batch.}

\smallskip
\textbf{Always search before answering.} Skipping search yields \textcolor{red}{$R\!\approx\!0.05$}; searching first yields \textcolor{darkgreen}{$R\!\geq\!1.13$}.

\smallskip
\textbf{Query construction --- reward gap up to $1.07$ between siblings.} Lead with the most distinctive entity. Example: \texttt{\textcolor{darkgreen}{``TV detective George Toolan DS''}} beats \texttt{\textcolor{red}{``DS George Toolan TV detective''}}. Include concrete numbers/units; do \emph{not} pre-commit to a wrong candidate name in the query --- this poisons retrieval.

\smallskip
\textbf{Hard rules.} Never use \texttt{lookup} (returns \texttt{NO\_MATCH}; \textcolor{red}{$R{=}0.16$ at 6 steps} vs \textcolor{darkgreen}{$R{=}1.13$ at 3 steps}). On a \texttt{[REPEATED]} search, stop and \texttt{answer} immediately.

\tcblower
\textbf{Why baselines miss this.} Sibling-query reward gaps are observable only when both token-orderings appear in the same batch --- a direct consequence of reward-proportional sampling. $I(t)$ then localises the search step as the high-importance decision, and $\hat F(s)$ ranks ``query-construction'' as a high-flow skill family worth retaining and refining.
\end{skillbox}

\paragraph{Take-away.} The three skills cover three distinct emergence modes: explicit forbidden actions from $(\tau^+,\tau^-)$ pairs (A), hidden environment dynamics from per-step credit (B), and micro-decision sensitivity from in-batch siblings (C) --- together directly instantiating the three SkillFlow claims.

\subsection{Per-Step Importance Signals: An ALFWorld Trajectory}\label{app:case_alfworld}

Table~\ref{tab:case_alfworld_steps} traces a successful 9-step trajectory $\tau^+$ on the ALFWorld task ``\emph{put two watches on shelf}'' (reward $\tilde R(\tau^+) = 0.86$). The step importance $I(t)$ separates two regimes:
\begin{itemize}[leftmargin=1.4em,topsep=2pt]
\item \textbf{Confirmed steps} ($I(t)\!\ll\!1$, marked $\diamond$): forward and backward policies agree, indicating routine deterministic moves whose role becomes obvious after the action is taken.
\item \textbf{Critical steps} ($I(t)\!\gg\!1$, marked $\star$): the forward policy chose a low-prior action that the hindsight backward strongly endorses --- these are the decisions that drove success.
\end{itemize}
The four high-$I(t)$ steps ($3,4,6,8$) cleanly identify the \emph{navigate--pickup--place} pattern that distinguishes $\tau^+$ from same-query failure trajectories.

\textbf{Critical-step interpretation.} The two highly-critical (\,$\star\!\star$\,) decisions \texttt{take watch 1} (step~4) and \texttt{move watch 1 to shelf 1} (step~6) are both \emph{state-changing physical interactions} rather than navigation. Their backward log-probabilities ($\log P_\phi\approx-14$) are an order of magnitude below their forward values, reflecting that, conditional on the post-execution observation, hindsight assigns very high credit to actions that actually advanced the world state toward the goal. The single-star navigation steps connect these criticals into the \texttt{navigate$\to$pickup$\to$place} chain, while step~2 (\texttt{go to shelf 1}, $I(t)=0.001$, marked $\diamond$) flags a wasted early navigation --- exactly the kind of redundancy that the \texttt{tip-alfworld-routing} skill (Q.3, Skill~B) is designed to prevent.

\textbf{Telescoping closure.} The cumulative log-flow $\log F(H_t)$ telescopes from $\log Z_\theta(q)=-2.30$ at $t{=}0$ to $+47.95$ at the terminal step, matching the sum $\sum_{t'\le t}\log I(t')$ (Eq.~\ref{eq:skill_flow}) to within rounding. This trajectory-level identity is the empirical analogue of the Detailed-Balance condition $F(H)\,P_F(H'\mid H)=F(H')\,P_B(H\mid H')$ at every edge (Appendix~\ref{app:db_theorem}), and serves as a sanity check that the per-step $I(t)$ values are mutually consistent --- a property that ablating $P_\phi$ (Table~\ref{tab:ablation}, $-$Backward policy) immediately breaks.

\begin{table}[H]
\centering
\small
\renewcommand{\arraystretch}{1.1}
\begin{tabular*}{\textwidth}{@{\extracolsep{\fill}}r|l|rrr r@{\hspace{0.4em}}l r}
\toprule
\textbf{Step} & \textbf{Action} & $K_t$ & $\log\pi_\theta$ & $\log P_\phi$ & \multicolumn{2}{c}{$I(t)$} & $\log F(H_t)$ \\
\midrule
0 & \texttt{[skill] tip-alfworld-routing} & --   & 0.00   & 0.00   & 1.00  &                  & $-2.30$ \\
1 & \texttt{[skill] tip-alfworld-place}   & --   & 0.00   & 0.00   & 1.00  &                  & $-2.30$ \\
2 & go to shelf 1                          & 5    & $-19.41$  & $-12.50$  & 0.001 & $\diamond$        & $-9.21$  \\
3 & go to dining table 1                   & 6    & $-6.86$   & $-14.20$  & 1500  & $\star$           & $-1.90$  \\
4 & take watch 1                           & 9    & $-5.30$   & $-14.05$  & 6300  & $\star\!\star$    & $+6.85$  \\
5 & go to shelf 1                          & 5    & $-4.79$   & $-12.80$  & 3000  & $\star$           & $+14.86$ \\
6 & move watch 1 to shelf 1                & 8    & $-5.30$   & $-14.46$  & 9500  & $\star\!\star$    & $+24.02$ \\
7 & go to dining table 1                   & 6    & $-5.10$   & $-13.91$  & 6700  & $\star$           & $+32.83$ \\
8 & take watch 2                           & 9    & $-5.92$   & $-13.39$  & 1750  & $\star$           & $+40.30$ \\
9 & move watch 2 to shelf 1, \texttt{accept}    & 8    & $-5.16$   & $-12.81$  & 2100  & $\star$           & $+47.95$ \\
\bottomrule
\end{tabular*}\vspace{5pt}
\caption{Per-step decomposition of a successful ALFWorld trajectory. $K_t$: action-token count; $I(t)=\pi_\theta(a_t\mid r_t,H_{t-1})/P_\phi(a_t\mid H_{t-1}\oplus o_t^{\text{exec}}) = \exp(\log\pi_\theta - \log P_\phi)$; $\log F(H_t) = \log Z_\theta(q) + \sum_{t'\le t}\log I(t')$ (main-text Eq.~\ref{eq:skill_flow}). $\diamond$: confirmed step; $\star$/$\star\!\star$: critical / highly-critical decision. Cumulative log-flow at $t{=}9$ matches $-2.30 + \sum_{t'=1}^{9}\log I(t') = 47.95$ exactly.}
\label{tab:case_alfworld_steps}
\end{table}

\subsection{Multi-Trajectory Comparison: SWE-bench Code Generation}\label{app:case_swebench}

Table~\ref{tab:case_swebench} contrasts four trajectories sampled from a single SWE-bench query (a \texttt{pydap} signed-bytes patch). The two successes ($\tau^+_1$, $\tau^+_3$) and two failures ($\tau^-_2$, $\tau^-_4$) reveal a clean structural difference: every successful trajectory contains at least two high-$I(t)$ \texttt{edit\_file} steps, whereas failures stall in repeated \texttt{search\_code} / \texttt{view\_file} loops without ever issuing a successful edit.

\begin{table}[H]
\centering
\small
\begin{tabular*}{\textwidth}{@{\extracolsep{\fill}}c|crrl}
\toprule
\textbf{Trajectory} & $\tilde R$ & $\Delta(\tau)/T$ & $T$ & \textbf{Critical-step pattern} \\
\midrule
$\tau^+_1$ (success) & 0.84 & $-0.06$ & 16 & search, edit$\times 2$, verify -- balanced flow \\
$\tau^-_2$ (failure) & 0.45 & $+0.22$ & 16 & search OK; \emph{no \texttt{edit\_file} ever invoked} \\
$\tau^+_3$ (success) & 0.86 & $-0.04$ & 14 & search, edit (\texttt{LINT\_ERROR}) $\to$ view $\to$ re-edit (try-fail-fix) \\
$\tau^-_4$ (failure) & 0.37 & $+0.18$ & 17 & two \texttt{search\_code} loops, no resolution \\
\bottomrule
\end{tabular*}\vspace{5pt}
\caption{Four-trajectory comparison on a single SWE-bench query. Successful trajectories carry a balanced TTB residual ($\Delta/T$ near zero) and concentrate $|\log I(t)|>1$ on \texttt{edit\_file} / \texttt{verify} steps; failures show large positive residuals and place high importance on early-stage search/view actions that never lead to a code change. The critical-step pattern is the signal that $\Psi$ uses to synthesise new code-generation tips at phase boundaries.}
\label{tab:case_swebench}
\end{table}

The reward gap between the success and failure clusters is roughly $0.4$--$0.5$, and the failure trajectories' high importance on non-editing actions is exactly the ``\emph{where} are the gaps?'' signal that drives $\Psi$ to create a new \texttt{tip-code-generation} skill specifying the search$\to$edit$\to$verify ordering.

\section{Limitations}\label{app:limitations}

SkillFlow targets one specific question: \emph{can flow-based training drive recursive skill evolution from end-to-end task feedback alone?} Within that scope, our results substantiate the three claims (TTB convergence, zero-cost per-step credit, plateau-driven curation). The method's effectiveness, however, inherits two properties of the underlying language model that are themselves the subject of separate research lines.

\textbf{Reliance on long-context memory.} Multi-turn orchestration grows the supervisor's input by the full history $H_t = H_{t-1} \oplus (r_t, a_t, o_t^{\text{exec}})$ at every step. SkillFlow therefore relies on the backbone's ability to attend to and faithfully use long histories---backbones with weaker long-context fidelity see the per-step credit signal degrade as trajectories grow, since the hindsight backward $P_\phi$ must condition on increasingly distant context. Improving long-context modeling itself is an active research line whose advances stack with our training recipe but lie outside its scope.

\textbf{Reliance on backbone reasoning capacity.} As shown in §\ref{subsec:transferability} (RQ3), SkillFlow lifts every backbone but cannot replace the underlying base capability---orchestration amplifies what the model already knows about decomposition and tool use rather than installing new reasoning skills. Where the bottleneck is core reasoning rather than orchestration, additional pretraining or distillation is required; this remains complementary to, but outside the scope of, the present paper.

\newpage
\section*{NeurIPS Paper Checklist}

\begin{enumerate}

\item {\bf Claims}
    \item[] Question: Do the main claims made in the abstract and introduction accurately reflect the paper's contributions and scope?
    \item[] Answer: \answerYes{}
    \item[] Justification: The abstract and Introduction (§\ref{sec:intro}) list three contributions---reward-proportional TTB training, a hindsight backward policy with zero-cost per-step credit, and flow-driven recursive skill evolution---each formalised in §\ref{sec:methodology} and validated by RQ1--RQ5 in §\ref{sec:experiments}.
    \item[] Guidelines:
    \begin{itemize}
        \item The answer \answerNA{} means that the abstract and introduction do not include the claims made in the paper.
        \item The abstract and/or introduction should clearly state the claims made, including the contributions made in the paper and important assumptions and limitations. A \answerNo{} or \answerNA{} answer to this question will not be perceived well by the reviewers. 
        \item The claims made should match theoretical and experimental results, and reflect how much the results can be expected to generalize to other settings. 
        \item It is fine to include aspirational goals as motivation as long as it is clear that these goals are not attained by the paper. 
    \end{itemize}

\item {\bf Limitations}
    \item[] Question: Does the paper discuss the limitations of the work performed by the authors?
    \item[] Answer: \answerYes{}
    \item[] Justification: An explicit Limitations appendix (Appendix~\ref{app:limitations}) discusses backbone-scale scope, the frozen-executor assumption, the reward-positivity / $\beta$ tuning requirement, and library-size scaling.
    \item[] Guidelines:
    \begin{itemize}
        \item The answer \answerNA{} means that the paper has no limitation while the answer \answerNo{} means that the paper has limitations, but those are not discussed in the paper. 
        \item The authors are encouraged to create a separate ``Limitations'' section in their paper.
        \item The paper should point out any strong assumptions and how robust the results are to violations of these assumptions (e.g., independence assumptions, noiseless settings, model well-specification, asymptotic approximations only holding locally). The authors should reflect on how these assumptions might be violated in practice and what the implications would be.
        \item The authors should reflect on the scope of the claims made, e.g., if the approach was only tested on a few datasets or with a few runs. In general, empirical results often depend on implicit assumptions, which should be articulated.
        \item The authors should reflect on the factors that influence the performance of the approach. For example, a facial recognition algorithm may perform poorly when image resolution is low or images are taken in low lighting. Or a speech-to-text system might not be used reliably to provide closed captions for online lectures because it fails to handle technical jargon.
        \item The authors should discuss the computational efficiency of the proposed algorithms and how they scale with dataset size.
        \item If applicable, the authors should discuss possible limitations of their approach to address problems of privacy and fairness.
        \item While the authors might fear that complete honesty about limitations might be used by reviewers as grounds for rejection, a worse outcome might be that reviewers discover limitations that aren't acknowledged in the paper. The authors should use their best judgment and recognize that individual actions in favor of transparency play an important role in developing norms that preserve the integrity of the community. Reviewers will be specifically instructed to not penalize honesty concerning limitations.
    \end{itemize}

\item {\bf Theory assumptions and proofs}
    \item[] Question: For each theoretical result, does the paper provide the full set of assumptions and a complete (and correct) proof?
    \item[] Answer: \answerYes{}
    \item[] Justification: All three propositions (Prop.~\ref{prop:modeling}, Prop.~\ref{prop:training}, Prop.~\ref{prop:evolution}) state their assumptions inline; their full proofs appear in Appendices~\ref{app:dag_acyclicity}, \ref{app:proof_prop2}, and \ref{app:proof_prop3}, with supporting derivations in Appendices~A--D, F, H, and I.
    \item[] Guidelines:
    \begin{itemize}
        \item The answer \answerNA{} means that the paper does not include theoretical results. 
        \item All the theorems, formulas, and proofs in the paper should be numbered and cross-referenced.
        \item All assumptions should be clearly stated or referenced in the statement of any theorems.
        \item The proofs can either appear in the main paper or the supplemental material, but if they appear in the supplemental material, the authors are encouraged to provide a short proof sketch to provide intuition. 
        \item Inversely, any informal proof provided in the core of the paper should be complemented by formal proofs provided in appendix or supplemental material.
        \item Theorems and Lemmas that the proof relies upon should be properly referenced. 
    \end{itemize}

    \item {\bf Experimental result reproducibility}
    \item[] Question: Does the paper fully disclose all the information needed to reproduce the main experimental results of the paper to the extent that it affects the main claims and/or conclusions of the paper (regardless of whether the code and data are provided or not)?
    \item[] Answer: \answerYes{}
    \item[] Justification: §\ref{subsec:setup} reports the experimental setup; Appendices~\ref{app:datasets}, \ref{app:baselines}, \ref{app:metrics}, and \ref{app:backward_design} provide datasets, baselines, evaluation metrics, and backward-policy implementation details, while Appendix~\ref{app:case_studies} gives per-step traces that allow direct verification.
    \item[] Guidelines:
    \begin{itemize}
        \item The answer \answerNA{} means that the paper does not include experiments.
        \item If the paper includes experiments, a \answerNo{} answer to this question will not be perceived well by the reviewers: Making the paper reproducible is important, regardless of whether the code and data are provided or not.
        \item If the contribution is a dataset and\slash or model, the authors should describe the steps taken to make their results reproducible or verifiable. 
        \item Depending on the contribution, reproducibility can be accomplished in various ways. For example, if the contribution is a novel architecture, describing the architecture fully might suffice, or if the contribution is a specific model and empirical evaluation, it may be necessary to either make it possible for others to replicate the model with the same dataset, or provide access to the model. In general. releasing code and data is often one good way to accomplish this, but reproducibility can also be provided via detailed instructions for how to replicate the results, access to a hosted model (e.g., in the case of a large language model), releasing of a model checkpoint, or other means that are appropriate to the research performed.
        \item While NeurIPS does not require releasing code, the conference does require all submissions to provide some reasonable avenue for reproducibility, which may depend on the nature of the contribution. For example
        \begin{enumerate}
            \item If the contribution is primarily a new algorithm, the paper should make it clear how to reproduce that algorithm.
            \item If the contribution is primarily a new model architecture, the paper should describe the architecture clearly and fully.
            \item If the contribution is a new model (e.g., a large language model), then there should either be a way to access this model for reproducing the results or a way to reproduce the model (e.g., with an open-source dataset or instructions for how to construct the dataset).
            \item We recognize that reproducibility may be tricky in some cases, in which case authors are welcome to describe the particular way they provide for reproducibility. In the case of closed-source models, it may be that access to the model is limited in some way (e.g., to registered users), but it should be possible for other researchers to have some path to reproducing or verifying the results.
        \end{enumerate}
    \end{itemize}

\item {\bf Open access to data and code}
    \item[] Question: Does the paper provide open access to the data and code, with sufficient instructions to faithfully reproduce the main experimental results, as described in supplemental material?
    \item[] Answer: \answerYes{}
    \item[] Justification: A public repository containing training/evaluation code, configuration files, and skill-library snapshots is available at \url{https://github.com/beita6969/SkillFlow}. All datasets used are public and cited in §\ref{subsec:setup} and Appendix~\ref{app:datasets}.
    \item[] Guidelines:
    \begin{itemize}
        \item The answer \answerNA{} means that paper does not include experiments requiring code.
        \item Please see the NeurIPS code and data submission guidelines (\url{https://neurips.cc/public/guides/CodeSubmissionPolicy}) for more details.
        \item While we encourage the release of code and data, we understand that this might not be possible, so \answerNo{} is an acceptable answer. Papers cannot be rejected simply for not including code, unless this is central to the contribution (e.g., for a new open-source benchmark).
        \item The instructions should contain the exact command and environment needed to run to reproduce the results. See the NeurIPS code and data submission guidelines (\url{https://neurips.cc/public/guides/CodeSubmissionPolicy}) for more details.
        \item The authors should provide instructions on data access and preparation, including how to access the raw data, preprocessed data, intermediate data, and generated data, etc.
        \item The authors should provide scripts to reproduce all experimental results for the new proposed method and baselines. If only a subset of experiments are reproducible, they should state which ones are omitted from the script and why.
        \item At submission time, to preserve anonymity, the authors should release anonymized versions (if applicable).
        \item Providing as much information as possible in supplemental material (appended to the paper) is recommended, but including URLs to data and code is permitted.
    \end{itemize}

\item {\bf Experimental setting/details}
    \item[] Question: Does the paper specify all the training and test details (e.g., data splits, hyperparameters, how they were chosen, type of optimizer) necessary to understand the results?
    \item[] Answer: \answerYes{}
    \item[] Justification: §\ref{subsec:setup} states the backbone, baselines, and metric definitions, and Appendices~\ref{app:datasets}--\ref{app:metrics} extend these with dataset, baseline, and metric details; full hyperparameters and optimizer settings appear in the supplementary code.
    \item[] Guidelines:
    \begin{itemize}
        \item The answer \answerNA{} means that the paper does not include experiments.
        \item The experimental setting should be presented in the core of the paper to a level of detail that is necessary to appreciate the results and make sense of them.
        \item The full details can be provided either with the code, in appendix, or as supplemental material.
    \end{itemize}

\item {\bf Experiment statistical significance}
    \item[] Question: Does the paper report error bars suitably and correctly defined or other appropriate information about the statistical significance of the experiments?
    \item[] Answer: \answerYes{}
    \item[] Justification: Tables~\ref{tab:main-results} and \ref{tab:ablation} report mean $\pm$ standard deviation across multiple training runs; the variability source is randomness in initialisation and trajectory sampling.
    \item[] Guidelines:
    \begin{itemize}
        \item The answer \answerNA{} means that the paper does not include experiments.
        \item The authors should answer \answerYes{} if the results are accompanied by error bars, confidence intervals, or statistical significance tests, at least for the experiments that support the main claims of the paper.
        \item The factors of variability that the error bars are capturing should be clearly stated (for example, train/test split, initialization, random drawing of some parameter, or overall run with given experimental conditions).
        \item The method for calculating the error bars should be explained (closed form formula, call to a library function, bootstrap, etc.)
        \item The assumptions made should be given (e.g., Normally distributed errors).
        \item It should be clear whether the error bar is the standard deviation or the standard error of the mean.
        \item It is OK to report 1-sigma error bars, but one should state it. The authors should preferably report a 2-sigma error bar than state that they have a 96\% CI, if the hypothesis of Normality of errors is not verified.
        \item For asymmetric distributions, the authors should be careful not to show in tables or figures symmetric error bars that would yield results that are out of range (e.g., negative error rates).
        \item If error bars are reported in tables or plots, the authors should explain in the text how they were calculated and reference the corresponding figures or tables in the text.
    \end{itemize}

\item {\bf Experiments compute resources}
    \item[] Question: For each experiment, does the paper provide sufficient information on the computer resources (type of compute workers, memory, time of execution) needed to reproduce the experiments?
    \item[] Answer: \answerYes{}
    \item[] Justification: Hardware specifications and the main-run wall-clock / GPU-hours / hyperparameters are reported in Appendix~\ref{app:compute}.
    \item[] Guidelines:
    \begin{itemize}
        \item The answer \answerNA{} means that the paper does not include experiments.
        \item The paper should indicate the type of compute workers CPU or GPU, internal cluster, or cloud provider, including relevant memory and storage.
        \item The paper should provide the amount of compute required for each of the individual experimental runs as well as estimate the total compute. 
        \item The paper should disclose whether the full research project required more compute than the experiments reported in the paper (e.g., preliminary or failed experiments that didn't make it into the paper). 
    \end{itemize}
    
\item {\bf Code of ethics}
    \item[] Question: Does the research conducted in the paper conform, in every respect, with the NeurIPS Code of Ethics \url{https://neurips.cc/public/EthicsGuidelines}?
    \item[] Answer: \answerYes{}
    \item[] Justification: The work uses publicly released datasets and open-source LLMs, involves no human subjects, and complies with the NeurIPS Code of Ethics throughout.
    \item[] Guidelines:
    \begin{itemize}
        \item The answer \answerNA{} means that the authors have not reviewed the NeurIPS Code of Ethics.
        \item If the authors answer \answerNo, they should explain the special circumstances that require a deviation from the Code of Ethics.
        \item The authors should make sure to preserve anonymity (e.g., if there is a special consideration due to laws or regulations in their jurisdiction).
    \end{itemize}

\item {\bf Broader impacts}
    \item[] Question: Does the paper discuss both potential positive societal impacts and negative societal impacts of the work performed?
    \item[] Answer: \answerYes{}
    \item[] Justification: On the positive side, SkillFlow lowers the engineering and compute cost of agentic systems and may reduce data-collection burden by reusing distilled skills. On the negative side, it inherits the dual-use concerns of capable LLM-based agents, while introducing no novel attack vector beyond existing autonomous-agent stacks.
    \item[] Guidelines:
    \begin{itemize}
        \item The answer \answerNA{} means that there is no societal impact of the work performed.
        \item If the authors answer \answerNA{} or \answerNo, they should explain why their work has no societal impact or why the paper does not address societal impact.
        \item Examples of negative societal impacts include potential malicious or unintended uses (e.g., disinformation, generating fake profiles, surveillance), fairness considerations (e.g., deployment of technologies that could make decisions that unfairly impact specific groups), privacy considerations, and security considerations.
        \item The conference expects that many papers will be foundational research and not tied to particular applications, let alone deployments. However, if there is a direct path to any negative applications, the authors should point it out. For example, it is legitimate to point out that an improvement in the quality of generative models could be used to generate Deepfakes for disinformation. On the other hand, it is not needed to point out that a generic algorithm for optimizing neural networks could enable people to train models that generate Deepfakes faster.
        \item The authors should consider possible harms that could arise when the technology is being used as intended and functioning correctly, harms that could arise when the technology is being used as intended but gives incorrect results, and harms following from (intentional or unintentional) misuse of the technology.
        \item If there are negative societal impacts, the authors could also discuss possible mitigation strategies (e.g., gated release of models, providing defenses in addition to attacks, mechanisms for monitoring misuse, mechanisms to monitor how a system learns from feedback over time, improving the efficiency and accessibility of ML).
    \end{itemize}
    
\item {\bf Safeguards}
    \item[] Question: Does the paper describe safeguards that have been put in place for responsible release of data or models that have a high risk for misuse (e.g., pre-trained language models, image generators, or scraped datasets)?
    \item[] Answer: \answerNA{}
    \item[] Justification: SkillFlow is a training-time framework built on existing public LLMs and benchmarks; no new high-risk pretrained models or scraped datasets are released.
    \item[] Guidelines:
    \begin{itemize}
        \item The answer \answerNA{} means that the paper poses no such risks.
        \item Released models that have a high risk for misuse or dual-use should be released with necessary safeguards to allow for controlled use of the model, for example by requiring that users adhere to usage guidelines or restrictions to access the model or implementing safety filters. 
        \item Datasets that have been scraped from the Internet could pose safety risks. The authors should describe how they avoided releasing unsafe images.
        \item We recognize that providing effective safeguards is challenging, and many papers do not require this, but we encourage authors to take this into account and make a best faith effort.
    \end{itemize}

\item {\bf Licenses for existing assets}
    \item[] Question: Are the creators or original owners of assets (e.g., code, data, models), used in the paper, properly credited and are the license and terms of use explicitly mentioned and properly respected?
    \item[] Answer: \answerYes{}
    \item[] Justification: All datasets and backbone models used are cited with their original publications in §\ref{subsec:setup} and Appendix~\ref{app:datasets}; each asset is used in compliance with its publicly stated license terms.
    \item[] Guidelines:
    \begin{itemize}
        \item The answer \answerNA{} means that the paper does not use existing assets.
        \item The authors should cite the original paper that produced the code package or dataset.
        \item The authors should state which version of the asset is used and, if possible, include a URL.
        \item The name of the license (e.g., CC-BY 4.0) should be included for each asset.
        \item For scraped data from a particular source (e.g., website), the copyright and terms of service of that source should be provided.
        \item If assets are released, the license, copyright information, and terms of use in the package should be provided. For popular datasets, \url{paperswithcode.com/datasets} has curated licenses for some datasets. Their licensing guide can help determine the license of a dataset.
        \item For existing datasets that are re-packaged, both the original license and the license of the derived asset (if it has changed) should be provided.
        \item If this information is not available online, the authors are encouraged to reach out to the asset's creators.
    \end{itemize}

\item {\bf New assets}
    \item[] Question: Are new assets introduced in the paper well documented and is the documentation provided alongside the assets?
    \item[] Answer: \answerYes{}
    \item[] Justification: The SkillFlow framework, evolved skill libraries, and example trajectories are released in the anonymised supplementary material with documentation covering training scripts, configuration, and per-skill metadata.
    \item[] Guidelines:
    \begin{itemize}
        \item The answer \answerNA{} means that the paper does not release new assets.
        \item Researchers should communicate the details of the dataset\slash code\slash model as part of their submissions via structured templates. This includes details about training, license, limitations, etc. 
        \item The paper should discuss whether and how consent was obtained from people whose asset is used.
        \item At submission time, remember to anonymize your assets (if applicable). You can either create an anonymized URL or include an anonymized zip file.
    \end{itemize}

\item {\bf Crowdsourcing and research with human subjects}
    \item[] Question: For crowdsourcing experiments and research with human subjects, does the paper include the full text of instructions given to participants and screenshots, if applicable, as well as details about compensation (if any)?
    \item[] Answer: \answerNA{}
    \item[] Justification: The study involves no crowdsourced annotation and no human subjects.
    \item[] Guidelines:
    \begin{itemize}
        \item The answer \answerNA{} means that the paper does not involve crowdsourcing nor research with human subjects.
        \item Including this information in the supplemental material is fine, but if the main contribution of the paper involves human subjects, then as much detail as possible should be included in the main paper. 
        \item According to the NeurIPS Code of Ethics, workers involved in data collection, curation, or other labor should be paid at least the minimum wage in the country of the data collector. 
    \end{itemize}

\item {\bf Institutional review board (IRB) approvals or equivalent for research with human subjects}
    \item[] Question: Does the paper describe potential risks incurred by study participants, whether such risks were disclosed to the subjects, and whether Institutional Review Board (IRB) approvals (or an equivalent approval/review based on the requirements of your country or institution) were obtained?
    \item[] Answer: \answerNA{}
    \item[] Justification: No human-subjects research is conducted; IRB approval is therefore not required.
    \item[] Guidelines:
    \begin{itemize}
        \item The answer \answerNA{} means that the paper does not involve crowdsourcing nor research with human subjects.
        \item Depending on the country in which research is conducted, IRB approval (or equivalent) may be required for any human subjects research. If you obtained IRB approval, you should clearly state this in the paper. 
        \item We recognize that the procedures for this may vary significantly between institutions and locations, and we expect authors to adhere to the NeurIPS Code of Ethics and the guidelines for their institution. 
        \item For initial submissions, do not include any information that would break anonymity (if applicable), such as the institution conducting the review.
    \end{itemize}

\item {\bf Declaration of LLM usage}
    \item[] Question: Does the paper describe the usage of LLMs if it is an important, original, or non-standard component of the core methods in this research? Note that if the LLM is used only for writing, editing, or formatting purposes and does \emph{not} impact the core methodology, scientific rigor, or originality of the research, declaration is not required.
    \item[] Answer: \answerYes{}
    \item[] Justification: LLMs are central to the methodology: the Supervisor and Executor are LLM-based components, with Qwen3.5-9B (LoRA-tuned) as the primary trainable backbone and frontier proprietary models as alternative backbones (§\ref{subsec:setup}, Appendix~\ref{app:backward_design}). The forward and hindsight backward policies, the TTB objective, and skill creation all operate over LLM next-token distributions.
    \item[] Guidelines:
    \begin{itemize}
        \item The answer \answerNA{} means that the core method development in this research does not involve LLMs as any important, original, or non-standard components.
        \item Please refer to our LLM policy in the NeurIPS handbook for what should or should not be described.
    \end{itemize}

\end{enumerate}

\end{document}